\documentclass{article}
\usepackage{preprint,times}
\preptfinalcopy
\usepackage{amsmath,amssymb,amsthm,graphicx,booktabs,microtype}
\usepackage{newtxtext,newtxmath}
\usepackage{hyperref}
\hypersetup{hidelinks}
\usepackage{url}

\newtheorem{theorem}{Theorem}
\newtheorem{proposition}{Proposition}
\newtheorem{corollary}{Corollary}
\newtheorem{lemma}{Lemma}
\newtheorem*{restate}{Proposition~\ref{prop:transport} (restated)}
\newtheorem*{restateinv}{Proposition~\ref{prop:invisible} (restated)}

\title{How Much Can Reliability Drift Under a Fixed Confidence Distribution?}
\author{Wenhao Liang, Lin Yue, Wei Emma Zhang, Mingyu Guo, Olaf Maennel, Weitong Chen \\[3pt]
Adelaide University}

\newcommand{\SpearmanMedian}{0.82}
\newcommand{\SpearmanRangeLo}{0.08}
\newcommand{\SpearmanRangeHi}{0.94}
\newcommand{\PearsonAll}{0.91}
\newcommand{\SpearmanAll}{0.88}
\newcommand{\PartialEceAll}{0.00}
\newcommand{\PartialGAll}{0.86}
\newcommand{\PartialEcePrim}{-0.78}
\newcommand{\PartialGPrim}{0.95}

\newcommand{\NatShareShiftIndep}{0.027}
\newcommand{\NatShareShiftDisjoint}{0.031}

\newcommand{\BrN}{12}
\newcommand{\BrFamilies}{17}
\newcommand{\BrRankMedian}{0.80}
\newcommand{\BrRankLo}{0.14}
\newcommand{\BrRankHi}{0.94}
\newcommand{\CombRankMedian}{0.81}

\newcommand{\BrEOneSpMax}{0.924}

\newcommand{\BrEOneBelow}{96.2}
\newcommand{\BrEOneMedRatio}{0.260}
\newcommand{\BrEOneQone}{9.2}
\newcommand{\BrEOneQfour}{0.74}
\newcommand{\BrCertBothRaw}{9}
\newcommand{\BrCertOneRaw}{10}
\newcommand{\BrCertBothTemp}{3}
\newcommand{\BrConfEceGRraw}{+0.61}
\newcommand{\BrConfEceFaucRaw}{+0.81}
\newcommand{\BrConfGRFaucRaw}{+0.92}
\newcommand{\CombConfEceGRraw}{+0.75}
\newcommand{\CombConfEceFaucRaw}{+0.85}
\newcommand{\CombConfGRFaucRaw}{+0.95}
\newcommand{\LofoEceFaucLo}{+0.79}
\newcommand{\LofoEceFaucHi}{+0.90}
\newcommand{\LofoGRFaucLo}{+0.93}
\newcommand{\LofoGRFaucHi}{+0.98}
\newcommand{\LofoRankLo}{0.81}
\newcommand{\LofoRankHi}{0.82}
\newcommand{\LofoEOneLo}{0.899}
\newcommand{\LofoEOneHi}{0.903}
\newcommand{\OpfEceFauc}{+0.82}
\newcommand{\OpfGRFauc}{+0.94}
\newcommand{\OpfRank}{0.81}
\newcommand{\OpfEOne}{0.899}
\newcommand{\BrMonoMedian}{76.2}
\newcommand{\BrMonoNLL}{10}
\newcommand{\BrMonoMaxdF}{0}

\newcommand{\SpEceFaucRawPrim}{+1.00}
\newcommand{\SpEceFaucRawAll}{+0.85}
\newcommand{\SpEceFaucTsPrim}{$-$0.83}
\newcommand{\SpEceFaucTsAll}{$-$0.36}
\newcommand{\SpGrFaucTsPrim}{+0.60}
\newcommand{\SpGrFaucTsAll}{+0.95}
\newcommand{\LofoTsLo}{$-$0.50}
\newcommand{\LofoTsHi}{$-$0.28}
\newcommand{\OpfTsEceFauc}{$-$0.35}
\newcommand{\MatchedTs}{5}
\newcommand{\MatchedRaw}{2}
\newcommand{\MatchedMaxRatio}{1.71}
\newcommand{\RsPooledPrim}{0.9010}
\newcommand{\RsPooledBr}{0.9005}
\newcommand{\RsPooledAll}{0.9007}

\newcommand{\RsFixedLo}{0.82}
\newcommand{\RsFixedHi}{0.84}
\newcommand{\RsNullFixedMean}{0.75}
\newcommand{\RsNullFixedMax}{0.77}

\newcommand{\RsNullPooledMean}{0.877}
\newcommand{\RsNullPooledMax}{0.881}

\newcommand{\RsUnitObsPrim}{0.78}

\newcommand{\RegTailThree}{45\%}
\newcommand{\VexcessThree}{44\%}
\newcommand{\VcapExcessThree}{10\%}

\newcommand{\MatchedVarThree}{19\%}

\newcommand{\FaucVcapRank}{0.996}
\newcommand{\FaucVRank}{0.976}
\newcommand{\FaucVoverF}{1.17}
\newcommand{\FaucVcapOverF}{1.09}
\newcommand{\RankTwoStatePrim}{0.59}
\newcommand{\RankTwoStateBr}{0.63}
\newcommand{\RankTwoStateLo}{0.08}
\newcommand{\RankTwoStateHi}{0.93}

\newcommand{\RankStateHalf}{0.84}
\newcommand{\RankStateTwo}{0.88}

\newcommand{\CtrlObsTwoState}{0.62}
\newcommand{\CtrlNullMean}{$-0.01$}
\newcommand{\CtrlNullMax}{0.05}

\newcommand{\LcRetainedMass}{77.9\%}
\newcommand{\LcRetainedMin}{55.8\%}
\newcommand{\LcRetainedMax}{98.6\%}
\newcommand{\LcRatioAll}{41.8\%}
\newcommand{\LcRatioRet}{43.9\%}
\newcommand{\LcCondAll}{400}
\newcommand{\LcCondRet}{401}

\newcommand{\LcSevOneRet}{52.7\%}

\newcommand{\LcSevFiveRet}{37.0\%}

\newcommand{\MagSlackB}{0.101}
\newcommand{\MagFullBsmall}{0.022}
\newcommand{\MagCapBsmall}{0.148}
\newcommand{\MagFullBlarge}{0.006}
\newcommand{\MagCapBlarge}{0.102}
\newcommand{\MagMinLabelB}{2}
\newcommand{\MagWeakRelFull}{56\%}

\newcommand{\MagWeakSelBias}{0.027}
\newcommand{\MagCondTotal}{600}
\newcommand{\MagMaxWorse}{0.0002}
\newcommand{\MagCondFull}{78}
\newcommand{\MagCondTie}{522}

\newcommand{\MagMaxFullOverCap}{6\times10^{-17}}
\newcommand{\MagIIDDropMaxN}{2000}
\newcommand{\ResBratioMedOne}{1.5}
\newcommand{\ResAratioMedOne}{9}
\newcommand{\ResBratioMedThree}{1.9}
\newcommand{\ResAratioMedThree}{6}
\newcommand{\ResBratioMaxThree}{6.2}
\newcommand{\ResBestKsmall}{8}
\newcommand{\ResBestKlarge}{16}
\newcommand{\RealKfourLo}{73\%}
\newcommand{\RealKfourHi}{94\%}
\newcommand{\RealKtwoLo}{48\%}
\newcommand{\RealKtwoHi}{75\%}
\newcommand{\RealDebShareRaw}{35\%}
\newcommand{\RealDebShareTemp}{14\%}
\newcommand{\EnvRatioMed}{82}
\newcommand{\EnvRatioQten}{20}
\newcommand{\EnvRecords}{3{,}159}
\newcommand{\RealReproMax}{3\times10^{-17}}
\newcommand{\AudChanged}{36}
\newcommand{\AudTotal}{6{,}480{,}000}
\newcommand{\AudMaxEntry}{0.018}
\newcommand{\AudMaxSummary}{3.2\times10^{-5}}

\newcommand{\RnineMainTable}{%
\begin{tabular}{lcccccc}
\toprule
& & & \multicolumn{2}{c}{$n=250$} & \multicolumn{2}{c}{$n=5000$} \\
\cmidrule(lr){4-5}\cmidrule(lr){6-7}
law of $\eta$ & $\mathcal F^+(3)$ & cap slack & full & capped & full & capped \\
\midrule
four-point (B) & 0.299 & 0.101 & 0.022 & 0.148 & 0.006 & 0.102 \\
mixture $\lambda=\tfrac12$ & 0.362 & 0.038 & 0.029 & 0.052 & 0.006 & 0.038 \\
three-point (A) & 0.400 & 0.000 & 0.030 & 0.030 & 0.006 & 0.006 \\
B, variance $/16$ & 0.075 & 0.025 & 0.042 & 0.052 & 0.008 & 0.027 \\
\bottomrule
\end{tabular}}
\newcommand{\RnineEoneTable}{%
\begin{tabular}{llccccccc}
\toprule
& & & \multicolumn{2}{c}{$n=250$} & \multicolumn{2}{c}{$n=1000$} & \multicolumn{2}{c}{$n=5000$} \\
\cmidrule(lr){4-5}\cmidrule(lr){6-7}\cmidrule(lr){8-9}
law, direction & $\rho$ & $\mathcal F$ & MAE f/c & $\Delta$ & MAE f/c & $\Delta$ & MAE f/c & $\Delta$ \\
\midrule
B & 0.3 & 0.155 & 0.011/0.011 & $+0.000$ & 0.006/0.006 & $+0.000$ & 0.003/0.003 & $+0.000$ \\
 & 1 & 0.283 & 0.020/0.020 & $+0.000$ & 0.010/0.010 & $+0.000$ & 0.005/0.005 & $+0.000$ \\
 & 3 & 0.299 & 0.022/0.148 & $-0.126$ & 0.014/0.110 & $-0.096$ & 0.006/0.102 & $-0.095$ \\
\addlinespace[1pt]
$\lambda=\tfrac14$ & 0.3 & 0.155 & 0.011/0.011 & $+0.000$ & 0.005/0.005 & $+0.000$ & 0.002/0.002 & $+0.000$ \\
 & 1 & 0.279 & 0.020/0.021 & $+0.000$ & 0.010/0.011 & $-0.001$ & 0.004/0.005 & $-0.001$ \\
 & 3 & 0.337 & 0.028/0.075 & $-0.047$ & 0.014/0.063 & $-0.049$ & 0.006/0.062 & $-0.056$ \\
\addlinespace[1pt]
$\lambda=\tfrac12$ & 0.3 & 0.155 & 0.010/0.010 & $+0.000$ & 0.005/0.005 & $+0.000$ & 0.002/0.002 & $+0.000$ \\
 & 1 & 0.275 & 0.019/0.020 & $-0.002$ & 0.009/0.012 & $-0.003$ & 0.004/0.008 & $-0.004$ \\
 & 3 & 0.362 & 0.029/0.052 & $-0.024$ & 0.015/0.040 & $-0.025$ & 0.006/0.038 & $-0.032$ \\
\addlinespace[1pt]
$\lambda=\tfrac34$ & 0.3 & 0.155 & 0.009/0.009 & $+0.000$ & 0.005/0.005 & $+0.000$ & 0.002/0.002 & $+0.000$ \\
 & 1 & 0.271 & 0.017/0.021 & $-0.004$ & 0.009/0.014 & $-0.005$ & 0.004/0.012 & $-0.009$ \\
 & 3 & 0.382 & 0.029/0.037 & $-0.008$ & 0.015/0.023 & $-0.008$ & 0.007/0.018 & $-0.012$ \\
\addlinespace[1pt]
A & 0.3 & 0.155 & 0.008/0.008 & $+0.000$ & 0.004/0.004 & $+0.000$ & 0.002/0.002 & $+0.000$ \\
 & 1 & 0.267 & 0.016/0.021 & $-0.006$ & 0.008/0.017 & $-0.009$ & 0.003/0.016 & $-0.013$ \\
 & 3 & 0.400 & 0.030/0.030 & $+0.000$ & 0.015/0.015 & $+0.000$ & 0.006/0.006 & $+0.000$ \\
\addlinespace[1pt]
B weak & 0.3 & 0.039 & 0.013/0.014 & $+0.000$ & 0.007/0.007 & $+0.000$ & 0.003/0.003 & $+0.000$ \\
 & 1 & 0.071 & 0.024/0.025 & $-0.001$ & 0.012/0.012 & $+0.000$ & 0.006/0.006 & $+0.000$ \\
 & 3 & 0.075 & 0.042/0.052 & $-0.010$ & 0.017/0.035 & $-0.018$ & 0.008/0.027 & $-0.019$ \\
\addlinespace[1pt]
A weak & 0.3 & 0.039 & 0.013/0.013 & $+0.000$ & 0.007/0.007 & $+0.000$ & 0.003/0.003 & $+0.000$ \\
 & 1 & 0.067 & 0.023/0.025 & $-0.002$ & 0.012/0.013 & $-0.001$ & 0.006/0.007 & $-0.001$ \\
 & 3 & 0.100 & 0.040/0.040 & $+0.000$ & 0.021/0.021 & $+0.000$ & 0.010/0.010 & $+0.000$ \\
\addlinespace[1pt]
constant & 0.3 & 0.000 & 0.028/0.028 & $+0.000$ & 0.014/0.014 & $+0.000$ & 0.006/0.006 & $+0.000$ \\
 & 1 & 0.000 & 0.047/0.050 & $-0.002$ & 0.024/0.025 & $-0.001$ & 0.011/0.011 & $-0.001$ \\
 & 3 & 0.000 & 0.065/0.065 & $+0.000$ & 0.033/0.033 & $+0.000$ & 0.014/0.014 & $+0.000$ \\
\addlinespace[1pt]
two-point & 0.3 & 0.110 & 0.012/0.012 & $+0.000$ & 0.006/0.006 & $+0.000$ & 0.003/0.003 & $+0.000$ \\
 & 1 & 0.200 & 0.022/0.022 & $+0.000$ & 0.011/0.011 & $+0.000$ & 0.005/0.005 & $+0.000$ \\
 & 3 & 0.200 & 0.040/0.040 & $+0.000$ & 0.020/0.020 & $+0.000$ & 0.009/0.009 & $+0.000$ \\
\addlinespace[1pt]
asym., up & 0.3 & 0.065 & 0.011/0.018 & $-0.008$ & 0.005/0.016 & $-0.011$ & 0.002/0.015 & $-0.013$ \\
 & 1 & 0.099 & 0.018/0.040 & $-0.021$ & 0.009/0.040 & $-0.031$ & 0.004/0.043 & $-0.039$ \\
 & 3 & 0.134 & 0.028/0.034 & $-0.006$ & 0.015/0.020 & $-0.005$ & 0.006/0.012 & $-0.006$ \\
asym., down & 0.3 & 0.080 & 0.011/0.011 & $+0.000$ & 0.005/0.005 & $+0.000$ & 0.002/0.002 & $+0.000$ \\
 & 1 & 0.147 & 0.020/0.020 & $+0.000$ & 0.010/0.010 & $+0.000$ & 0.004/0.004 & $+0.000$ \\
 & 3 & 0.247 & 0.035/0.036 & $-0.001$ & 0.017/0.018 & $-0.001$ & 0.008/0.009 & $-0.002$ \\
\addlinespace[1pt]
\bottomrule
\end{tabular}}
\newcommand{\RnineArmsTable}{%
\begin{tabular}{llcccc}
\toprule
condition & sampling & bias f/c, $n=250$ & MAE f/c, $n=250$ & $n=1000$ & $n=5000$ \\
\midrule
B, up, $\rho=3$ & proportional & $+0.010$/$+0.141$ & 0.022/0.148 & 0.014/0.110 & 0.006/0.102 \\
 & equal & $+0.000$/$+0.096$ & 0.025/0.097 & 0.013/0.101 & 0.006/0.101 \\
 & iid, drop empty & $+0.006$/$+0.120$ & 0.023/0.130 & 0.014/0.106 & 0.006/0.102 \\
 & iid, drop $<30$ & $-0.017$/$-0.017$ & 0.025/0.025 & 0.020/0.020 & 0.006/0.102 \\
\addlinespace[1pt]
asym., down, $\rho=1$ & proportional & $+0.003$/$+0.004$ & 0.020/0.020 & 0.010/0.010 & 0.004/0.004 \\
 & equal & $+0.004$/$+0.004$ & 0.012/0.013 & 0.006/0.006 & 0.003/0.003 \\
 & iid, drop empty & $+0.004$/$+0.004$ & 0.021/0.020 & 0.010/0.010 & 0.005/0.005 \\
 & iid, drop $<30$ & $-0.076$/$-0.071$ & 0.076/0.071 & 0.010/0.010 & 0.004/0.004 \\
\addlinespace[1pt]
\bottomrule
\end{tabular}}
\newcommand{\RnineEtwoPopTable}{%
\begin{tabular}{llccccccccc}
\toprule
& & & & \multicolumn{3}{c}{$\rho=1$} & \multicolumn{3}{c}{$\rho=3$} \\
\cmidrule(lr){5-7}\cmidrule(lr){8-10}
law, ordering & $K$ & $G_R$ & $H_R$ & gap & B$-\mathcal F_R$ & A$-\mathcal F_R$ & gap & B$-\mathcal F_R$ & A$-\mathcal F_R$ \\
\midrule
smooth, feature & 1 & 0.0000 & 0.0469 & 0.212 & 0.217 & 0.217 & 0.296 & 0.375 & 0.375 \\
 & 2 & 0.0385 & 0.0085 & 0.016 & 0.020 & 0.092 & 0.088 & 0.126 & 0.159 \\
 & 4 & 0.0417 & 0.0053 & 0.010 & 0.014 & 0.072 & 0.030 & 0.057 & 0.126 \\
 & 8 & 0.0440 & 0.0029 & 0.006 & 0.008 & 0.054 & 0.016 & 0.036 & 0.094 \\
 & 16 & 0.0458 & 0.0012 & 0.003 & 0.003 & 0.034 & 0.003 & 0.013 & 0.059 \\
\addlinespace[1pt]
rare, feature & 1 & 0.0000 & 0.0342 & 0.178 & 0.185 & 0.185 & 0.264 & 0.320 & 0.320 \\
 & 2 & 0.0231 & 0.0111 & 0.026 & 0.033 & 0.106 & 0.050 & 0.107 & 0.183 \\
 & 4 & 0.0292 & 0.0050 & 0.016 & 0.017 & 0.071 & 0.031 & 0.055 & 0.122 \\
 & 8 & 0.0336 & 0.0006 & 0.002 & 0.002 & 0.025 & 0.007 & 0.013 & 0.044 \\
 & 16 & 0.0340 & 0.0002 & 0.000 & 0.001 & 0.015 & 0.001 & 0.004 & 0.027 \\
\addlinespace[1pt]
wave, feature & 1 & 0.0000 & 0.0504 & 0.213 & 0.224 & 0.224 & 0.273 & 0.389 & 0.389 \\
 & 2 & 0.0016 & 0.0488 & 0.173 & 0.184 & 0.221 & 0.230 & 0.346 & 0.383 \\
 & 4 & 0.0422 & 0.0081 & 0.015 & 0.026 & 0.090 & 0.055 & 0.108 & 0.156 \\
 & 8 & 0.0425 & 0.0078 & 0.013 & 0.024 & 0.089 & 0.047 & 0.099 & 0.153 \\
 & 16 & 0.0472 & 0.0032 & 0.003 & 0.009 & 0.056 & 0.012 & 0.043 & 0.097 \\
\addlinespace[1pt]
smooth, by $\eta$ (oracle) & 1 & 0.0000 & 0.0469 & 0.212 & 0.217 & 0.217 & 0.296 & 0.375 & 0.375 \\
 & 2 & 0.0385 & 0.0085 & 0.016 & 0.020 & 0.092 & 0.088 & 0.126 & 0.159 \\
 & 4 & 0.0433 & 0.0036 & 0.006 & 0.010 & 0.060 & 0.017 & 0.038 & 0.104 \\
 & 8 & 0.0460 & 0.0009 & 0.001 & 0.002 & 0.030 & 0.002 & 0.011 & 0.052 \\
 & 16 & 0.0467 & 0.0003 & 0.000 & 0.001 & 0.017 & 0.001 & 0.004 & 0.029 \\
\addlinespace[1pt]
rare, by $\eta$ (oracle) & 1 & 0.0000 & 0.0342 & 0.178 & 0.185 & 0.185 & 0.264 & 0.320 & 0.320 \\
 & 2 & 0.0218 & 0.0125 & 0.041 & 0.048 & 0.112 & 0.127 & 0.163 & 0.193 \\
 & 4 & 0.0322 & 0.0021 & 0.006 & 0.007 & 0.046 & 0.006 & 0.019 & 0.079 \\
 & 8 & 0.0341 & 0.0002 & 0.000 & 0.001 & 0.013 & 0.001 & 0.003 & 0.023 \\
 & 16 & 0.0341 & 0.0001 & 0.000 & 0.000 & 0.010 & 0.000 & 0.002 & 0.017 \\
\addlinespace[1pt]
wave, by $\eta$ (oracle) & 1 & 0.0000 & 0.0504 & 0.213 & 0.224 & 0.224 & 0.273 & 0.389 & 0.389 \\
 & 2 & 0.0416 & 0.0088 & 0.021 & 0.032 & 0.094 & 0.081 & 0.137 & 0.163 \\
 & 4 & 0.0475 & 0.0029 & 0.004 & 0.008 & 0.054 & 0.029 & 0.057 & 0.093 \\
 & 8 & 0.0494 & 0.0010 & 0.001 & 0.003 & 0.031 & 0.006 & 0.017 & 0.054 \\
 & 16 & 0.0500 & 0.0004 & 0.000 & 0.001 & 0.019 & 0.001 & 0.006 & 0.033 \\
\addlinespace[1pt]
\bottomrule
\end{tabular}}
\newcommand{\RnineEtwoFinTable}{%
\begin{tabular}{llccccc}
\toprule
law & $n$ & $K=1$ & $K=2$ & $K=4$ & $K=8$ & $K=16$ \\
\midrule
smooth & 250 & 0.000/0.296 & 0.030/0.094 & 0.039/0.048 & 0.040/0.037 & 0.045/0.043 \\
 & 1000 & 0.000/0.296 & 0.015/0.089 & 0.020/0.036 & 0.020/0.023 & 0.020/0.019 \\
 & 5000 & 0.000/0.296 & 0.007/0.088 & 0.009/0.031 & 0.009/0.018 & 0.008/0.008 \\
\addlinespace[1pt]
rare & 250 & 0.000/0.264 & 0.038/0.063 & 0.040/0.046 & 0.044/0.041 & 0.044/0.044 \\
 & 1000 & 0.000/0.264 & 0.019/0.053 & 0.021/0.037 & 0.021/0.020 & 0.021/0.021 \\
 & 5000 & 0.000/0.264 & 0.008/0.051 & 0.010/0.033 & 0.010/0.012 & 0.010/0.010 \\
\addlinespace[1pt]
wave & 250 & 0.000/0.273 & 0.026/0.231 & 0.037/0.059 & 0.046/0.041 & 0.051/0.043 \\
 & 1000 & 0.000/0.273 & 0.015/0.231 & 0.020/0.056 & 0.022/0.043 & 0.022/0.019 \\
 & 5000 & 0.000/0.273 & 0.007/0.230 & 0.010/0.056 & 0.010/0.047 & 0.010/0.013 \\
\addlinespace[1pt]
\bottomrule
\end{tabular}}
\newcommand{\RnineEthreeTable}{%
\begin{tabular}{lcccc}
\toprule
state & FAUC$_4$/FAUC$_8$ & FAUC$_2$/FAUC$_8$ & $G_{R}^{\mathrm{deb}}$: $K{=}4$/$K{=}8$ & $\Delta G^{\mathrm{deb}}/\Delta G^{\mathrm{plug}}$, $4\to8$ \\
\midrule
as released & 0.89 [0.76, 0.94] & 0.65 [0.56, 0.75] & 0.89 [0.85, 1.05] & 0.35 [-0.05, 0.79] \\
temperature-scaled & 0.80 [0.73, 0.89] & 0.62 [0.48, 0.72] & 0.89 [0.83, 1.08] & 0.14 [-0.04, 0.61] \\
\bottomrule
\end{tabular}}

\begin{document}
\maketitle

\begin{abstract}
A classifier's conditional accuracy can change while its confidence distribution
stays exactly the same. We study the worst-case movement of the reliability
relation under covariate shifts that preserve the distribution of the confidence
score, constraining the reweighting within each confidence level by a $\chi^2$
budget; the resulting worst case, as a function of the budget, is a fragility
profile. On an interval of budgets that can be computed from the source
distribution, the profile equals exactly the square root of the budget times the
within-level variance of the correctness propensity---the grouping-loss term of
calibration--refinement decompositions. Beyond this interval the profile is
governed by the tails of the propensity law, and the entire upward profile
determines the centred within-level law; consequently, calibration residual and
grouping variance do not determine fragility in general, though they do when
labels and predictions are deterministic. Since the propensity is not observed,
we restrict reweightings to a learned finite readout within confidence bins,
bound the part the restriction misses by the grouping variance remaining inside
readout cells, estimate the restricted profile with role-separated labels, and
provide a separate split-sample lower confidence bound. On ImageNet this bound is
positive in both splits for four of six primary classifiers and nine of twelve
additional ones as released, and for three of eighteen after temperature scaling.
Held-out drift under optimised reweightings fitted without evaluation labels
tracks the estimated profile; an exploratory label-permutation diagnostic yields
near-zero agreement for this statistic while largely reproducing the correlation
observed for unsigned random reweightings. Under natural shifts, the within-bin
composition change, which the source profile bounds, is smaller than the
within-cell change, which it does not control.
\end{abstract}

\section{Introduction}
\label{sec:intro}

A classifier's confidence describes its prediction; the reliability relation
$s\mapsto\Pr(C=1\mid S=s)$ describes how often predictions at that confidence are
correct. The latter depends on the data distribution, and calibration error
compares the two on the distribution at hand. Under shift the relation can move, and a common first check is whether the
confidence distribution has changed. This paper asks how much can happen when it
has not.

Within a confidence level, inputs share a score but not a propensity to be
correct. A shift that favours the harder inputs at each level, without changing
the mass of any level, leaves the score distribution unchanged and lowers
reliability; a rule predicting only when $S\ge t$ keeps its coverage while its error
rate rises (Proposition~\ref{prop:selective}, Figure~\ref{fig:concept}). Confidence histograms, coverage curves and accuracy
estimators computed from target confidences \citep{garg2022leveraging} cannot
register such a shift, although a monitor of the inputs might. We quantify the
exposure by the \emph{fragility profile}: the largest movement of the reliability
relation under covariate shifts that preserve the score distribution, as a
function of a $\chi^2$ budget on the reweighting within each confidence level.

The profile has an exact local law. With $\eta(X)=\Pr(C=1\mid X)$ and
$G_P(s)=\operatorname{Var}(\eta\mid S=s)$, the grouping heterogeneity of
calibration--refinement decompositions
\citep{degroot1983comparison,kull2015novel,perezlebel2023grouping}, the profile
equals $\sqrt{\rho\,G_P(s)}$ on an interval $[0,\rho_+]$ computable from the source
(Corollary~\ref{cor:local}), the non-negativity condition of
\citet{duchi2019variance} read within a level set; a static term of a calibration
decomposition is thus the rate at which reliability can drift. Beyond $\rho_+$ the
variance no longer suffices: via the
$\chi^2$ robust dual and the uniqueness of a law given its stop-loss transform
\citep{duchi2021learning,gushchin2018integrated}, the upward profile identifies the
whole centred within-level law of $\eta$ (Theorem~\ref{thm:ident}), and two
classifiers can agree in calibration residual and grouping variance yet differ in
fragility (Proposition~\ref{prop:witness}). This hierarchy collapses when labels
and predictions are deterministic functions of the input, where the reliability
relation fixes the profile (Section~\ref{sec:setup}).

The propensity is not observed, so we change the target rather than estimate it:
restricting reweightings to a learned finite readout inside confidence bins gives
a population restricted profile, which for fixed bins lies below the binned oracle
profile by at most the square root of the budget times the grouping variance left
inside readout cells (Propositions~\ref{prop:hierarchy} and~\ref{prop:approx}). Estimating and bounding it are
separate problems. Role-separated labels remove a selection bias from the plug-in
estimate, which nevertheless stays upward biased in expectation and has no
guarantee; a split-sample construction gives a finite-sample lower confidence bound
(Proposition~\ref{prop:fscert}). We evaluate these objects on six primary and
twelve additional ImageNet classifiers, check held-out tracking with exploratory
within-bin label-permutation diagnostics, compare the profile with its variance on the
same cells and, with known propensities, in finite samples, and decompose natural drift into a composition term that the source
profile bounds and a within-cell term it does not.

\paragraph{Contributions.}
\begin{itemize}\itemsep1pt \parskip0pt
\item A formulation of reliability drift under score-preserving covariate shift,
  with an exact local law and a proof that the upward profile identifies the
  centred within-level propensity law (Corollary~\ref{cor:local},
  Theorem~\ref{thm:ident}), and its scope: general-law witnesses and the
  deterministic-label collapse.
\item An observable restriction with a population ordering and approximation bound
  at fixed bins, a role-separated plug-in estimator with its remaining bias stated,
  and a split-sample lower confidence bound with per-analysis coverage
  (Propositions~\ref{prop:hierarchy}, \ref{prop:approx} and~\ref{prop:fscert}).
\item An empirical study on eighteen ImageNet classifiers with exploratory permutation diagnostics,
  a same-cell comparison of profile and variance with a known-propensity check of its
  finite-sample magnitude, and a natural-shift decomposition
  (Sections~\ref{sec:cannotsee} and~\ref{sec:natural}).
\end{itemize}

\begin{figure}[tb]
\centering
\includegraphics[width=\textwidth]{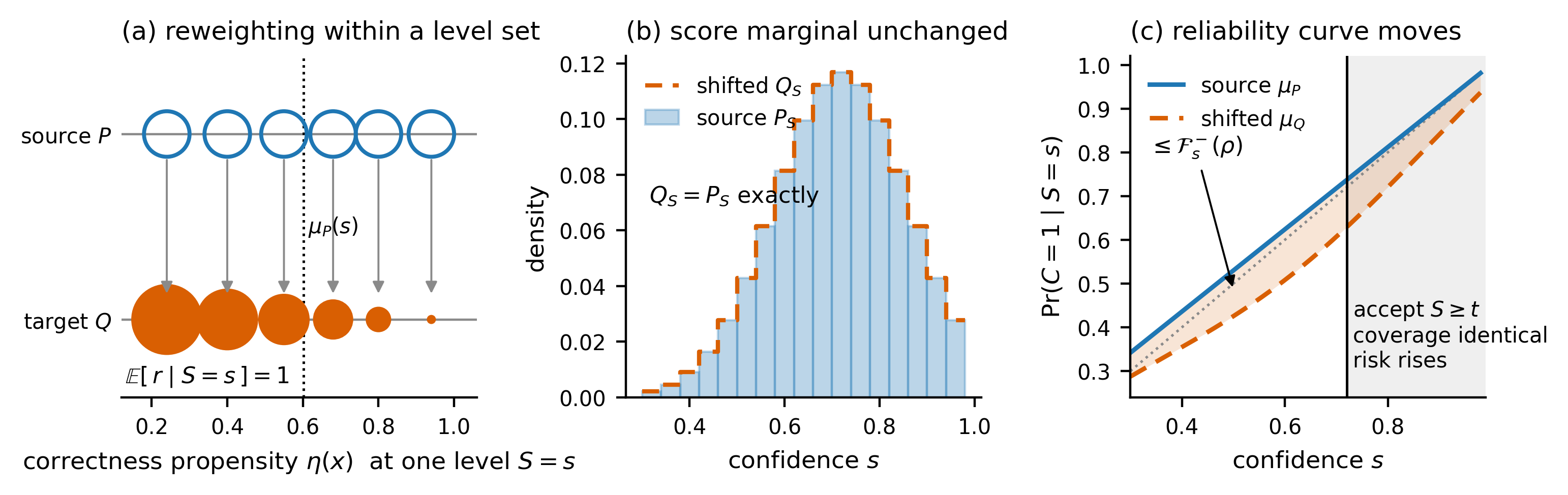}
\caption{\textbf{Reliability can move while the confidence distribution does not.}
\textbf{(a)} Within a confidence level set, mass moves between inputs that share a
score but differ in their propensity to be correct; $\mathbb E[r\mid S=s]=1$ keeps
the level's mass fixed. \textbf{(b)} The score distribution, and so any confidence
histogram, is unchanged. \textbf{(c)} The reliability curve moves by at most the
fragility profile; a threshold policy keeps its coverage while its error rate rises.
Movement is possible exactly where $G_P>0$ (Proposition~\ref{prop:invisible}). The
schematic is the population construction; the finite-sample one of
Section~\ref{sec:certificate} preserves the binned marginal.}
\label{fig:concept}
\end{figure}

\section{Setup}
\label{sec:setup}

Fix a trained classifier; nothing below modifies it. For an input $X$ with label
$Y$ and prediction $\hat Y(X)$, write $C=\mathbf 1\{\hat Y(X)=Y\}$ for correctness
and $S=s(X)\in[0,1]$ for the top-label probability. Let $\eta(X)=\Pr(C=1\mid X)$ be
the correctness propensity and $\mu_P(s)=\mathbb E_P[\eta\mid S=s]$ the
reliability relation, whose distance from the diagonal is what calibration error
measures. Inputs that share a confidence value need not share a propensity, and
their disagreement is
\begin{equation}
  G_P(s) \;=\; \operatorname{Var}_P\!\big(\eta(X)\mid S=s\big),
  \qquad Z \;=\; \eta - \mu_P(s),
  \label{eq:G}
\end{equation}
so $\mathbb E_P[Z\mid S=s]=0$ and $G_P(s)=\mathbb E_P[Z^2\mid S=s]$. This is the
integrand of the grouping loss \citep{degroot1983comparison,kull2015novel,
perezlebel2023grouping}, which measures heterogeneity a score cannot express; here
it will measure how far reliability can move.

\paragraph{Deterministic labels.}
If the label and the prediction are both deterministic functions of $X$, then
$\eta\in\{0,1\}$, $G_P(s)=\mu(1-\mu)$ with $\mu=\mu_P(s)$, and the profile defined
below is
\begin{equation}
  \mathcal F^{+}_s(\rho)=\min\big\{\sqrt{\rho\,\mu(1-\mu)},\,1-\mu\big\},\qquad
  \mathcal F^{-}_s(\rho)=\min\big\{\sqrt{\rho\,\mu(1-\mu)},\,\mu\big\},
  \label{eq:deterministic}
\end{equation}
also at $\mu\in\{0,1\}$, where both sides vanish (Appendix~\ref{app:claims}). In
this family the reliability relation determines the profile; the non-determination
statements of Section~\ref{sec:theory} concern general laws. One label
per input does not reveal whether labelling is deterministic, and we assume
neither; the empirical quantities of Section~\ref{sec:certificate} average
correctness over groups and need not be $\{0,1\}$-valued even then.

\paragraph{Score-preserving shifts.}
Let $Q\ll P$ with density $r=\mathrm dQ/\mathrm dP$ a function of $X$, so the shift
reweights inputs and leaves $P(Y\mid X)$ unchanged. We restrict to reweightings
\emph{within} confidence level sets and size them by a conditional Pearson $\chi^2$
budget:
\begin{equation}
  \mathbb E_P[\,r\mid S=s\,]=1 \ \text{ for $P_S$-a.e. } s,
  \qquad
  \mathbb E_P\!\big[(r-1)^2\mid S=s\big]\le\rho .
  \label{eq:set}
\end{equation}
The first condition holds if and only if $Q_S=P_S$ (Lemma~\ref{lem:preserve}), so
it defines score preservation rather than adding an assumption. Write
$\mathcal A_s(\rho)$ for the set of $r\ge0$ satisfying~\eqref{eq:set}. For scale,
doubling a subgroup's share of a level set from $0.2$ to $0.4$ costs
$\rho=0.2\cdot1^2+0.8\cdot0.25^2=0.25$. Smooth $f$-divergences behave alike for
small $\rho$ (Appendix~\ref{app:fdiv}).

\paragraph{Calibration fragility.}
The largest upward movement of the reliability relation at level $s$ is
\begin{equation}
  \mathcal F^{+}_s(\rho) \;=\; \sup_{r\in\mathcal A_s(\rho)}
  \mathbb E_P\big[(r-1)\,Z \mid S=s\big],
  \label{eq:fragility}
\end{equation}
and $\mathcal F^{-}_s$ replaces $Z$ by $-Z$;
$\rho\mapsto(\mathcal F^{+}_s,\mathcal F^{-}_s)$ is the \emph{fragility profile}, a
function of the source law alone.

\section{The fragility profile}
\label{sec:theory}

The profile, a supremum over reweightings, depends only on the conditional law of $Z$.

\begin{proposition}[Reliability transport]
\label{prop:transport}
For any $Q$ satisfying~\eqref{eq:set},
$\;\mu_Q(s)-\mu_P(s) = \mathbb E_P[(r-1)Z \mid S=s]$ for $P_S$-a.e.\ $s$.
\end{proposition}
\noindent
The proof uses that $r$ is a function of $X$; a density depending on $Y$ would
move $\eta$ itself (Appendix~\ref{app:transport}). Reliability drift is thus a
covariance between the reweighting and the centred propensity, and
$-\mathcal F^{-}_s(\rho)\le\mu_Q(s)-\mu_P(s)\le\mathcal F^{+}_s(\rho)$ for every
admissible $Q$, with both bounds attained.

\subsection{An exact local law, and what lies beyond it}

Under $P(\cdot\mid S=s)$, so that $\mathbb EZ=0$ and $\mathbb EZ^2=G_P(s)$, let
$z_- = \operatorname*{ess\,inf} Z$ and $z_+ = \operatorname*{ess\,sup} Z$, and for
$\tau\in\mathbb R$ define the stop-loss transform, its second moment and their
ratio
\begin{equation}
  M_1(\tau) = \mathbb E[(Z-\tau)_+],\quad
  M_2(\tau) = \mathbb E[(Z-\tau)_+^2],\quad
  \kappa(\tau) = \frac{M_2(\tau)}{M_1(\tau)^2}-1 .
  \label{eq:stoploss}
\end{equation}

\begin{theorem}[Full fragility profile]
\label{thm:profile}
Assume $G_P(s)>0$, let $p_+ = \Pr(Z=z_+)$ and adopt the convention
$1/0:=+\infty$. Then $\mathcal F^{+}_s$ is concave and non-decreasing, and
\begin{enumerate}\itemsep0pt
  \item \emph{variance-controlled}, for $0\le\rho\le\rho_+ := G_P(s)/z_-^2$:
    $\ \mathcal F^{+}_s(\rho)=\sqrt{\rho\,G_P(s)}$, attained by the affine
    reweighting $r^\ast = 1+\sqrt{\rho/G_P(s)}\,Z$;
  \item \emph{tail-controlled}, for $\rho_+<\rho<\rho_{\mathrm{sat}} := 1/p_+-1$:
    $\ \mathcal F^{+}_s(\rho)=\tau_\rho + M_2(\tau_\rho)/M_1(\tau_\rho)$ where
    $\kappa(\tau_\rho)=\rho$, attained by the truncated-affine
    $r^\ast = (Z-\tau_\rho)_+/M_1(\tau_\rho)$;
  \item \emph{saturated}, for $\rho\ge\rho_{\mathrm{sat}}$:
    $\ \mathcal F^{+}_s(\rho)=z_+$.
\end{enumerate}
The downward profile is obtained by replacing $Z$ with $-Z$, giving
$\rho_-=G_P(s)/z_+^2$ and $\rho_{\mathrm{sat}}^-=1/\Pr(Z=z_-)-1$.
Proof in Appendix~\ref{app:profile}.
\end{theorem}

The closed form is the $\chi^2$ case of the Cressie--Read robust dual
\citep[Lemma~1]{duchi2021learning}, a higher-moment risk measure
\citep{krokhmal2007higher}; we state it for its regimes. The first ends exactly when
the affine optimiser would put negative weight on the hardest inputs, the condition
of \citet[inequality~(9)]{duchi2019variance} for their objective to equal its
mean-plus-standard-deviation form.

\begin{corollary}[Local law]
\label{cor:local}
$\mathcal F^{\pm}_s(\rho)=\sqrt{\rho\,G_P(s)}$ holds \emph{with equality on the
closed interval} $[0,\rho_\pm]$, not merely asymptotically; hence
$\lim_{\rho\downarrow 0}\mathcal F^{\pm}_s(\rho)^2/\rho = G_P(s)$.
\end{corollary}

\noindent
The grouping heterogeneity of a static decomposition is thus the exact
coefficient of reliability drift at small budgets, on an interval computable from
the source. Figure~\ref{fig:profile} (left) shows the three regimes for one law.
Either of the last two can be empty: atomless $Z$ never saturates, and a
two-valued $Z$, such as~\eqref{eq:deterministic}, passes directly from the
square-root law to saturation. Two consequences need no tail analysis. First,
$\mathcal F^{\pm}_s>0$ exactly where $G_P(s)>0$:

\begin{proposition}[Score-marginal-invisible drift]
\label{prop:invisible}
Let $A$ be a measurable set of confidence values with $P_S(A)>0$. If $G_P>0$
holds $P_S$-a.e.\ on $A$, then for every $\rho>0$ there is a $Q$
satisfying~\eqref{eq:set} with $Q_S=P_S$ exactly and $\mu_Q>\mu_P$ $P_S$-a.e.\ on
$A$. If instead $G_P=0$ holds $P_S$-a.e.\ on $A$, then every such $Q$ has
$\mu_Q=\mu_P$ $P_S$-a.e.\ on $A$.
\end{proposition}

\noindent
(Sets of positive measure are needed because $\mu_P$ and $G_P$ are defined only
$P_S$-a.e.; Appendix~\ref{app:invisible}.) Second, a threshold policy inherits the
drift at fixed coverage. For \emph{predict iff $S\ge t$}
\citep{elyaniv2010foundations,geifman2017selective}, with selective risk
$\mathcal R(t)=\Pr(C=0\mid S\ge t)$:

\begin{proposition}[Selective-risk transport]
\label{prop:selective}
For $t$ with $P(S\ge t)>0$, under a score-preserving shift coverage is exactly preserved,
$Q(S\ge t)=P(S\ge t)$, and
\begin{equation}
  \mathcal R_Q(t)-\mathcal R_P(t)= -\,\frac{\mathbb E_{P_S}\!\big[\delta(S)\mathbf 1\{S\ge t\}\big]}
                        {\Pr(S\ge t)},
  \qquad \delta(s)=\mu_Q(s)-\mu_P(s),
\end{equation}
so the profile bounds the movement in each direction by the tail-averaged
$\mathcal F^{\mp}$ above $t$ (Appendix~\ref{app:selective}).
\end{proposition}

\subsection{The whole profile identifies the centred law}

The variance fixes the profile only near the origin; the whole curve fixes the
centred law.

\begin{theorem}[Profile identifiability]
\label{thm:ident}
Let $Z$ be bounded with $\mathbb E Z=0$ and $G_P(s)>0$. The map
$\mathrm{Law}(Z)\mapsto\mathcal F^{+}_s$ is injective. Explicitly,
$\mathcal F^+_s$ is differentiable on $(0,\rho_{\mathrm{sat}})$ with
\begin{equation}
  \tfrac{\mathrm d}{\mathrm d\rho}\mathcal F^{+}_s(\rho)
   = \tfrac12 M_1(\tau_\rho),
  \qquad
  \tau_\rho = \mathcal F^{+}_s(\rho)-2(\rho+1)\,
              \tfrac{\mathrm d}{\mathrm d\rho}\mathcal F^{+}_s(\rho),
  \label{eq:recovery}
\end{equation}
and $\Pr(Z>\tau)$ is the negative right derivative of $M_1$; the saturation
branch supplies
$z_+=\lim_{\rho\to\infty}\mathcal F^+_s(\rho)$ and
$p_+=1/(1+\rho_{\mathrm{sat}})$. Proof in Appendix~\ref{app:ident}.
\end{theorem}

As $\rho\uparrow\rho_{\mathrm{sat}}$, $\tau_\rho$ sweeps the support of $Z$ from below, so the curve and its slope trace out $M_1$, which determines the law
\citep{gushchin2018integrated}. The ingredients are classical; the theorem concerns
the exact curve: we prove no stability of the inverse and never invert an
estimated curve (Appendix~\ref{app:ident}). At a level, the reliability relation
supplies the mean, grouping heterogeneity the variance, and the profile the centred
law of $\eta$ (not of the label); over general laws neither the mean nor the profile
determines the other, and a scalar calibration error, an average over $S$, cannot
replace the mean.

\begin{proposition}[Calibration residual does not determine fragility in general]
\label{prop:witness}
Write $e_P(s)=\mu_P(s)-s$. Over general propensity laws neither $e_P$ nor the
profile determines the other: $\eta\equiv s$ makes both vanish, and the two-point
laws below give every other combination. They are not unrelated, since
$\eta\in[0,1]$ forces $G_P(s)\le\mu_P(s)(1-\mu_P(s))$.
(i) $\eta\in\{0.2,1.0\}$ w.p.\ $\tfrac12$ at $s=0.6$ gives $e_P=0$ and
$G_P=0.16$, perfectly calibrated yet fragile;
(ii) $\eta\equiv 0.9$ at $s=0.6$ gives $e_P=0.3$, $G_P=0$, badly calibrated and
not fragile;
(iii) $\eta\equiv 0.6$ versus $\eta\in\{0.2,1.0\}$ at $s=0.5$ share $e_P=0.1$
with $G_P\in\{0,0.16\}$;
(iv) two laws with $\mu_P=0.6$ and $G_P=0.08$ but $z_+\in\{0.3,0.4\}$ share
$e_P$ \emph{and} $G_P$, hence agree exactly for $\rho\le1.125$, yet separate by
$0.100$ at $\rho\ge2$.
\end{proposition}

\noindent
Case (iv) matches residual and variance and still leaves the profile undetermined
(Appendix~\ref{app:witness}); with deterministic labels $e_P(s)$ fixes it
through~\eqref{eq:deterministic}.

\begin{figure}[tb]
\centering
\includegraphics[width=0.49\textwidth]{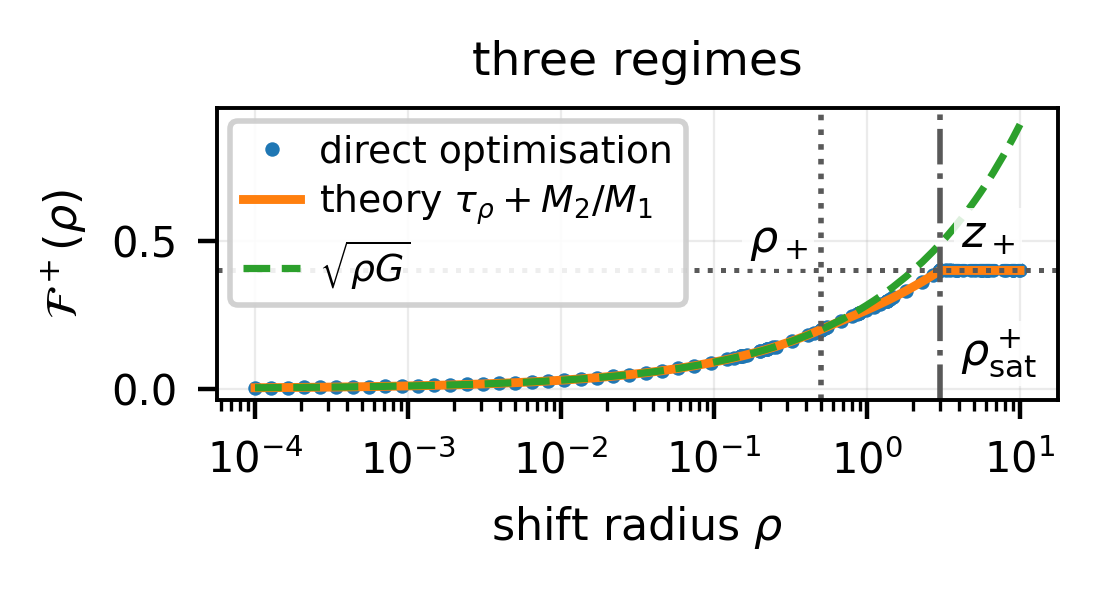}
\hfill
\includegraphics[width=0.49\textwidth]{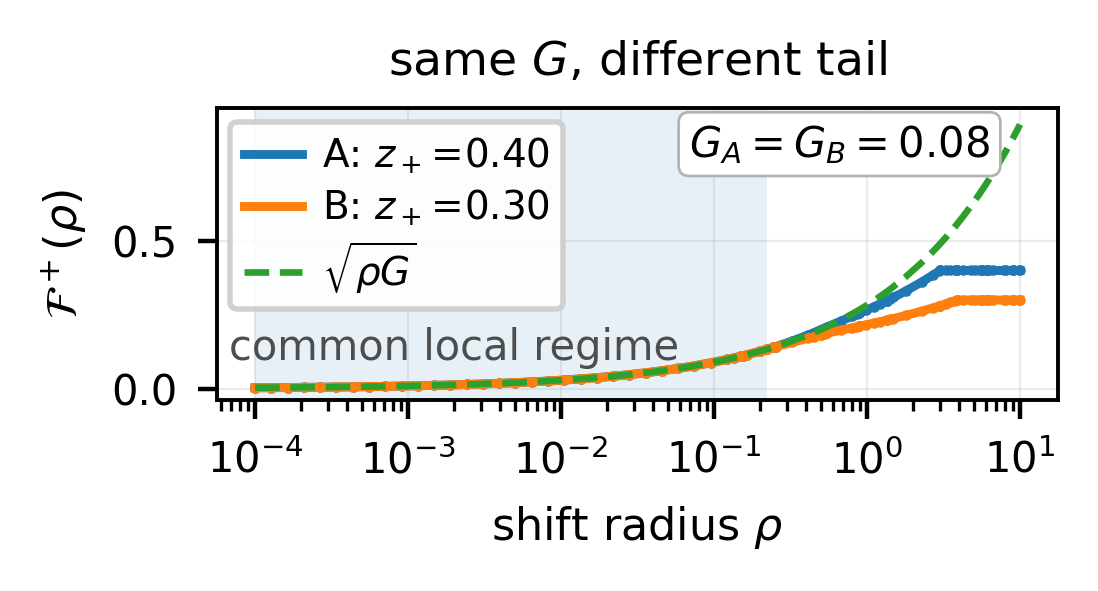}
\caption{\textbf{The profile beyond its variance-controlled regime.} Left: the
three regimes of Theorem~\ref{thm:profile} for $\eta\in\{0.2,0.6,1.0\}$ with
probabilities $(\tfrac14,\tfrac12,\tfrac14)$ ($\rho_+=0.5$, $\rho_{\mathrm{sat}}=3$).
Right: that law and a four-point law on $\{0,0.4,0.7,0.9\}$ with the same mean and
variance ($0.6$, $0.08$) agree for $\rho\le0.22$ and separate by $0.119$ at $\rho=3$.
The two-point witnesses of Proposition~\ref{prop:witness}(iv) have no tail-controlled
regime and separate by $0.100$ once both saturate (Appendix~\ref{app:witness}).}
\label{fig:profile}
\end{figure}

\section{Restricted fragility: estimation and a lower confidence bound}
\label{sec:certificate}

Section~\ref{sec:theory} is defined through the unobserved $\eta(X)$, and
estimating $G_P(s)$ itself is hard \citep{perezlebel2023grouping}. We therefore
change the object: reweightings are restricted to what a finite readout of the
input can express, and the score is replaced by a fixed partition into bins.

\paragraph{The population restricted profile.}
Let $R=\pi(X)$ take finitely many values and let $B$ be a fixed partition of the
confidence range. With $p_b=P_B(b)$ and $p_{bk}=P(R=k\mid B=b)$, define the cell
rates and the restricted heterogeneity
\begin{equation}
  \theta_P(b,k)=\Pr(C=1\mid B=b,R=k),\qquad
  G_R(b)=\operatorname{Var}_P\!\big(\theta_P(B,R)\mid B=b\big),
\end{equation}
and let $\mathcal F^{\pm}_{B,R}(b,\rho)$ be the profile~\eqref{eq:fragility} with
$B=b$ for $S=s$ and $(B,R)$-measurable reweightings. Write $\mu_P(b)$, $G_P(b)$ and
$\mathcal F^{\pm}_B(b,\rho)$ for the binned oracle mean, variance and profile of
$\eta$ given $B=b$.

\begin{proposition}[Readout hierarchy]
\label{prop:hierarchy}
If $\sigma(R_1)\subseteq\sigma(R_2)$, then $G_{R_1}(b)\le G_{R_2}(b)\le G_P(b)$ and
$\mathcal F^{\pm}_{B,R_1}\le\mathcal F^{\pm}_{B,R_2}\le\mathcal F^{\pm}_{B}$ pointwise
in $(b,\rho)$.
\end{proposition}

\noindent
The proof is conditional Jensen and nested feasible sets
(Appendix~\ref{app:hierarchy}). For fixed bins, the restricted profile is thus a
population lower bound on the binned oracle profile. It orders nested readouts only
--- a learned readout and quantile groups of the confidence are not ordered --- and
the binned oracle is not the level-set object:
$\operatorname{Var}(\eta\mid B)=\mathbb E[\operatorname{Var}(\eta\mid S)\mid B]
+\operatorname{Var}(\mu_P(S)\mid B)$, so a bin can be heterogeneous because the
reliability relation varies across it. Our reweightings preserve the binned
marginal exactly ($\max_b|Q_B(b)-P_B(b)|<10^{-14}$), not the continuous one, and
threshold rules keep their coverage exactly only at bin boundaries
(Appendix~\ref{app:diagnostics}).

\begin{proposition}[Readout approximation]
\label{prop:approx}
Let $H_R(b)=\mathbb E_P[(\eta-\theta_P(B,R))^2\mid B=b]=G_P(b)-G_R(b)$. For $\rho\ge0$,
$\mathcal F^{\pm}_{B}(b,\rho)\le\sup_{0\le u\le\rho}\{\mathcal F^{\pm}_{B,R}(b,u)
+\sqrt{(\rho-u)H_R(b)}\}\le\mathcal F^{\pm}_{B,R}(b,\rho)+\sqrt{\rho H_R(b)}$; the
middle bound is attained on the variance-controlled interval of the binned oracle.
If $\sigma(R_1)\subseteq\sigma(R_2)$, then
$0\le\mathcal F^{\pm}_{B,R_2}-\mathcal F^{\pm}_{B,R_1}\le\sqrt{\rho\,[G_{R_2}(b)-G_{R_1}(b)]}$.
\end{proposition}

\noindent
The proof averages a reweighting over readout cells and applies Cauchy--Schwarz to
the remainder (Appendix~\ref{app:approx}), as \citet[App.~B]{duchi2023distributionally}
do with a bounded residual. $H_R$ is the residual grouping loss of
\citet[Thm.~4.1]{perezlebel2023grouping}, which single labels do not identify: the
bound limits what a finer readout could add, and is neither an estimate nor a
confidence statement.

\paragraph{The plug-in estimate.}
Within each of five outer folds the training portion is halved: $D_{\mathrm{fit}}$
($20{,}000$ images) fits the temperature, $J=20$ equal-mass bins, the readout and
its $K=8$ within-bin quantile groups; $D_{\mathrm{rate}}$ ($20{,}000$) estimates
raw cell rates on these frozen cells; the held-out fold $D_{\mathrm{eval}}$
($10{,}000$) supplies group proportions and held-out correctness. The readout is a
low-capacity logistic regression on confidence, margin, entropy and a
$50$-dimensional principal-component projection of the penultimate
representation. Separating $D_{\mathrm{rate}}$ from $D_{\mathrm{fit}}$ keeps rate
labels from choosing their own cells, a selection effect we measured by label
permutation (Appendix~\ref{app:estimators}; \citealp[\S4.4]{perezlebel2023grouping}).
It does not make the plug-in unbiased: at fixed partition and weights the profile is
convex in the cell rates, so unbiased rates give an upward bias in expectation,
while a realisation can fall on either side of the target. The plug-in is neither an upper nor a lower confidence bound.

For model comparisons we summarise it by the fragility area under the curve (FAUC): per fold, per-bin plug-in profiles are combined into
$F^{+}(\rho)=(\sum_b w_b\,\widehat{\mathcal F}^{+}_{B,R}(b,\rho)^2)^{1/2}$ with $w_b$
the bin's share of evaluation examples, averaged over folds and integrated in
$\log(1+\rho)$ over $\rho\in\{0,0.01,0.03,0.1,0.3,1,3\}$. FAUC is a lossy ranking
summary; Theorem~\ref{thm:ident} is never applied to it. We also report $G_R$ with
a binomial-noise correction (Proposition~\ref{prop:gunbiased}), conditionally
unbiased for an empirical-weight version of $G_R$; FAUC uses uncorrected rates.

\paragraph{A split-sample lower confidence bound.}
Moving mass $d$ from group $u$ to group $v$ in a bin is feasible while $d\le p_u$,
costs $d^2(1/p_v+1/p_u)$ in conditional $\chi^2$ and moves reliability by
$d(\theta_v-\theta_u)$; the largest step
$d^\ast=\min\{p_u,\sqrt{\rho\,p_vp_u/(p_v+p_u)}\}$ (zero when either mass is zero) is
non-decreasing in both masses,
so simultaneous confidence bounds can be substituted into it.

\begin{proposition}[Split-sample lower confidence bound]
\label{prop:fscert}
If $\theta_P(b,k)\in[L_{bk},U_{bk}]$ and $p_{bk}\ge l_{bk}$ hold simultaneously
then, on that event and for every $\rho$,
\begin{equation}
  \mathcal F^{+}_{B,R}(b,\rho)\;\ge\;
  \max_{u\neq v}\,[L_{bv}-U_{bu}]_+\,
  \min\Big\{l_{bu},\sqrt{\rho\,l_{bv}l_{bu}/(l_{bv}+l_{bu})}\Big\}
  \;=:\;\mathrm{LCB}^{+}_b(\rho),
  \label{eq:lcb}
\end{equation}
where the square-root factor is read as $0$ when $l_{bv}l_{bu}=0$; with bin-mass bounds
$p_b\ge l_b$ the quadratic means obey
$(\sum_b l_b\,\mathrm{LCB}^{+}_b(\rho)^2)^{1/2}\le(\sum_b p_b\,\mathcal F^{+}_{B,R}(b,\rho)^2)^{1/2}$.
Reversing the correctness comparison gives the downward bound.
\end{proposition}

\noindent
We fit $(B,R)$ on one half of the $50{,}000$ validation images and take exact
one-sided Clopper--Pearson bounds on the other, Bonferroni-corrected over the
$2JK+JK+J=500$ statements of one analysis, so one event covers every pair,
direction and radius (Appendix~\ref{app:fscert}): with probability at least $0.95$
the model-level bound lies below the restricted profile of the partition learned on
the fitting half. Swapping halves gives a second analysis with a different partition
and target; the two \emph{split arms} are reported separately. The bound is a
bin-mass root mean square of per-bin worst-case movements: a value of $0.01$ says
that some admissible reweighting raises bin-level reliability by at least $0.01$ in
quadratic mean, not that overall accuracy moves that much. A positive bound shows
that restricted fragility is not zero; a zero bound does not show it is small.

\section{Experiments}
\label{sec:cannotsee}

We ask whether the plug-in profile orders held-out drift and how much of that
ordering estimation noise could produce, what the full profile adds to its variance
and whether that survives estimation, how fragility and calibration error rank
models, and where the lower confidence bound is positive.

The primary cohort is six frozen ImageNet-1K checkpoints \citep{deng2009imagenet}
(Table~\ref{tab:models}). Twelve more, one from each of twelve further families, were
selected by rules fixed before any of their metrics was computed
(Appendix~\ref{app:breadth}). Each is
evaluated as released ($T{=}1$) and with a temperature fitted on $D_{\mathrm{fit}}$
($T{=}\hat T$), with $T\in\{0.5,2\}$ as stress states. All analyses use the data
roles of Section~\ref{sec:certificate} except the lower confidence bound and the
monotone recalibration, which keep their preregistered protocols. The twelve
additions and every analysis marked post-hoc were designed after the primary
results were known (Appendix~\ref{app:claims}); none re-scores a preregistered
outcome.

\begin{table}[b]
\caption{Primary cohort. ECE: $20$-bin equal-mass out-of-fold estimate; $G_R$ with
the binomial-noise correction of Proposition~\ref{prop:gunbiased}; FAUC of the upward
plug-in profile. Neither $G_R$ nor FAUC is a confidence bound. Additional
checkpoints: Table~\ref{tab:breadth}.}
\label{tab:models}
\vspace{3pt}
\centering\small
\begin{tabular}{lccccccc}
\toprule
& & \multicolumn{3}{c}{$T=1$} & \multicolumn{3}{c}{$T=\hat T$} \\
\cmidrule(lr){3-5}\cmidrule(lr){6-8}
Model & Acc. & ECE & $G_R$ & FAUC & ECE & $G_R$ & FAUC \\
\midrule
ResNet-50 V1    & 0.761 & 0.038 & 0.0004 & 0.042 & 0.024 & 0.0009 & 0.049 \\
ResNet-50 V2    & 0.809 & 0.412 & 0.0133 & 0.135 & 0.032 & 0.0009 & 0.046 \\
ConvNeXt-T      & 0.825 & 0.169 & 0.0081 & 0.112 & 0.031 & 0.0011 & 0.048 \\
ViT-B/16        & 0.811 & 0.057 & 0.0011 & 0.048 & 0.042 & 0.0005 & 0.042 \\
Swin-T          & 0.815 & 0.068 & 0.0015 & 0.056 & 0.032 & 0.0004 & 0.038 \\
EffNet-B0       & 0.777 & 0.069 & 0.0032 & 0.078 & 0.024 & 0.0010 & 0.049 \\
\bottomrule
\end{tabular}
\end{table}

\paragraph{Held-out drift under optimised reweightings.}
We solve the restricted problem on $D_{\mathrm{rate}}$ rates, freeze the maximising
weights and measure their drift on $D_{\mathrm{eval}}$. Across bins, predicted and
realised drift have a median Spearman correlation of \CtrlObsTwoState{} over the $36$
as-released and temperature-scaled checkpoint--state combinations
(\RankTwoStatePrim{} primary, \RankTwoStateBr{} additional; range
\RankTwoStateLo{}--\RankTwoStateHi{}; Appendix~\ref{app:estimators}). Agreement is
weakest where $G_R$ is smallest (Figure~\ref{fig:breadth-gr}) and higher in the stress
states (\RankStateHalf{} at $T=0.5$, \RankStateTwo{} at $T=2$; over all four states
\SpearmanMedian{} and \BrRankMedian{}). The weights are optimised against the
estimated rates, so this checks the fitted construction out of sample, not
natural-shift prediction.

\paragraph{Separating the ordering from estimation noise.}
Binomial noise in cell rates scales with bin accuracy, so statistics of estimated
rates vary across bins even without heterogeneity. Random
bin-preserving directions drawn without correctness ($100$ per bin, radius, fold,
state and checkpoint; Appendix~\ref{app:ext:random}) give a Spearman correlation
between the two-sided plug-in profile and the median held-out movement of
\RsPooledPrim{} on the primary six and \RsPooledAll{} on all eighteen
(\RsFixedLo{}--\RsFixedHi{} at a fixed radius). A within-bin label permutation, applied
separately to $D_{\mathrm{rate}}$ and $D_{\mathrm{eval}}$, removes any association
between correctness and the readout while keeping bin accuracies, cell counts,
proportions and directions, and it largely reproduces these values (post-hoc;
Appendix~\ref{app:r7}): \RsNullPooledMean{} pooled and \RsNullFixedMean{} averaged over
radii on all eighteen (largest replicates \RsNullPooledMax{} and \RsNullFixedMax{}).
The observed values exceed every replicate, but the primary cohort's median
correlation within one checkpoint, state, fold and radius, \RsUnitObsPrim{}, lies
inside the null range: for unsigned random movements the estimated profile mainly
tracks the scale of estimation noise.

After inspecting this random-shift diagnostic, we added an exploratory permutation
check for the signed drift under optimised reweightings: its median agreement is
\CtrlNullMean{} under the null (largest of $100$: \CtrlNullMax{}) against
\CtrlObsTwoState{} observed. The check was not prespecified; permutations are
conditional on fold, so it ignores cross-fold dependence and supports no formal
$p$-value. It indicates that rate noise alone does not produce this agreement; it
concerns fixed bins and readout, says nothing about natural shifts, and cannot
separate level-set heterogeneity from within-bin variation of the reliability relation.

\paragraph{What the full profile adds to its variance.}
On identical cells we compared the plug-in profile with $\sqrt{\rho\hat G}$ and with
its cap at the support endpoint, $\min\{\sqrt{\rho\hat G},\hat z_+\}$ (post-hoc;
Table~\ref{tab:r7-profile}). Up to $\rho=0.1$ every bin is in the
variance-controlled regime and the three coincide; at $\rho=0.3$ \RegTailThree{}
of bin--radius units have left it, and at $\rho\ge1$ almost all. There the variance
law overstates the profile by a median \VexcessThree{} at $\rho=3$ and the capped
version by \VcapExcessThree{}, and bins with matched $\hat G$ differ in profile by a
median \MatchedVarThree{}: the distinction of Theorem~\ref{thm:ident} appears in
the estimated laws. These differences did not materially change rankings here: FAUC from the capped variance orders the $36$ checkpoint--state
combinations almost as FAUC does (Spearman \FaucVcapRank{}), and the two track
optimised and random shifts equally well. The full profile thus changes the size of
large-budget movements, not model rankings.

\paragraph{Whether the difference survives estimation.}
Real data give no ground truth, so we simulated bins with known masses and
propensities, including laws that share mean, variance and endpoints but not tails
(Appendix~\ref{app:r9}), and computed the full plug-in and the capped variance from
the same rates. In population the two coincide in the variance-controlled regime and
for constant or two-point laws, but their noisy estimates need not; beyond that
regime the cap keeps its slack while the plug-in error falls with $n$. For a
four-point law at $\rho=3$, where the slack is \MagSlackB{}, the mean absolute errors
are \MagFullBsmall{} (plug-in) and \MagCapBsmall{} (cap) with $250$ rate labels and
\MagFullBlarge{} and \MagCapBlarge{} with $5000$ (Table~\ref{tab:r9-magnitude}). In
the primary arm the plug-in met the specified practical-advantage criterion in
\MagCondFull{} of \MagCondTotal{} conditions and the cap in none, the plug-in was worse
by at most \MagMaxWorse{}, and no control met the criterion. Both are unreliable under
weak heterogeneity at $n=250$, and filtering rare cells can remove the practical
advantage. On the $32$-atom laws the induced gap bound of
Proposition~\ref{prop:approx} was a median \ResBratioMedOne{} times the true gap at
$\rho=1$, versus \ResAratioMedOne{} for the simple gap bound; at a fixed label budget
estimation error can grow with readout size.

\paragraph{Calibration error and fragility.}
Proposition~\ref{prop:witness} settles determination; here we ask how the two rank
models. As released, ECE and FAUC order the primary six identically (Spearman
\SpEceFaucRawPrim{}) and all eighteen similarly (\SpEceFaucRawAll{}); after
temperature scaling the association reverses (\SpEceFaucTsPrim{}, \SpEceFaucTsAll{};
\LofoTsLo{} to \LofoTsHi{} when one family is deleted at a time), while $G_R$ and
FAUC stay positively associated (\SpGrFaucTsPrim{}, \SpGrFaucTsAll{}). Scaling can
also reorder the top-label score, so it is not a pure recalibration. No primary
pair matches in ECE within the preregistered $5\times10^{-4}$; across all eighteen
\MatchedTs{} scaled and \MatchedRaw{} as-released pairs do, with FAUC ratios up to
\MatchedMaxRatio{}, but also differing $G_R$ (post-hoc; Table~\ref{tab:r7-matched}). A cleaner comparison holds the partition fixed: a
fold-wise strictly increasing map $g(s)=\mathrm{sigmoid}(a\,\mathrm{logit}\,s+c)$,
$a>0$, leaves bins and the frozen readout unchanged within each fold, so the
restricted profile, FAUC and selective-risk fragility are unchanged by
Lemma~\ref{lem:reparam}, while cross-fitted ECE falls by a median $75.1\%$ on the
primary six and \BrMonoMedian{}\% on the ten additional checkpoints where the map
can be fitted (Appendices~\ref{app:monotone},~\ref{app:breadth}). Refitting the
readout afterwards changes FAUC by up to $27\%$.

\paragraph{Where the lower confidence bound is nonzero.}
At $\rho=0.3$ the model-level bound is positive in both split arms for four of the
six primary checkpoints as released --- ResNet-50 V2 ($0.0158$, $0.0171$),
ConvNeXt-Tiny ($0.0096$, $0.0114$), Swin-T and EfficientNet-B0 --- and for
\BrCertBothRaw{} of the twelve additional ones; after temperature scaling, for none
of the six and \BrCertBothTemp{} of the twelve; scaling refits the partition, so the
positive set can change either way. The additional values are at most $0.0042$
except ConvMixer-768/32 ($0.0092$--$0.0110$). Each entry carries a $95\%$ guarantee
for its own arm's target; the counts imply no joint coverage
(Appendix~\ref{app:fscert}, Table~\ref{tab:breadth-cert}).

\section{Natural shift: what the source profile bounds}
\label{sec:natural}

Natural benchmarks need not be reweightings of the source: the propensity of an
input group can itself change. With the source bins and readout frozen and applied
to a target domain, write $p_k,q_k$ for the source and target group proportions in
a bin and $\theta_P,\theta_Q$ for the cell correctness rates.

\begin{proposition}[Composition--conditional decomposition]
\label{prop:decomp}
$\delta_{\mathrm{total}}(b)=\delta_{\mathrm{comp}}(b)+\delta_{\mathrm{cond}}(b)$
exactly, where
$\delta_{\mathrm{comp}}(b)=\sum_k (q_k-p_k)\,[\theta_P(b,k)-\mu_P(b)]$ and
$\delta_{\mathrm{cond}}(b)=\sum_k q_k\,[\theta_Q(b,k)-\theta_P(b,k)]$.
If $q_k=0$ whenever $p_k=0$, then
$|\delta_{\mathrm{comp}}(b)|\le\mathcal F^{\pm}_{B,R}\big(b,D_R(b)\big)$ at the
observed composition radius $D_R(b)=\sum_{k:\,p_k>0} p_k(q_k/p_k-1)^2$, the branch chosen by
the sign of $\delta_{\mathrm{comp}}(b)$.
\end{proposition}

\noindent
The identity is the classical composition-versus-rate decomposition
\citep{kitagawa1955,oaxaca1973}, related to the separation of input shift, outcome
change and poorly supported inputs by \citet{cai2023diagnosing}; here it is taken at a
frozen readout within bins, and the bound holds because $q_k/p_k$ is a feasible
reweighting at radius $D_R(b)$ (Appendix~\ref{app:decomp}). The source profile does
not control $\delta_{\mathrm{cond}}$, which at a finite readout can reflect a
changed labelling mechanism or within-cell variation the partition does not
resolve; under pure covariate shift with common support it vanishes as the readout
approaches $\sigma(X)$ (Appendix~\ref{app:hierarchy}). The bound is a population
statement evaluated with plug-in source rates, without coverage guarantee;
Proposition~\ref{prop:fscert} bounds fragility from below, not this drift from above.

\paragraph{Summaries.}
Per domain we take target-mass-weighted quadratic means
$L_x=(\sum_b q_b\,\delta_x(b)^2)^{1/2}$. These are norms, so
$L_{\mathrm{total}}\le L_{\mathrm{comp}}+L_{\mathrm{cond}}$ only, and
$L_{\mathrm{comp}}/L_{\mathrm{total}}$ is a ratio of norms, not an additive share.
The utilisation $U(b)=|\delta_{\mathrm{comp}}(b)|/\widehat{\mathcal F}^{\pm}_{B,R}(b,D_R(b))$
is undefined, and excluded, when a target group has no source mass or the
denominator vanishes; numerator and denominator use the same plug-in rates, so
$U\le1$ by construction.

\paragraph{Results.}
This analysis uses the original protocol's frozen partition, fitted before the
estimator revision, on ImageNet-C \citep{hendrycks2019benchmarking}, ImageNetV2
\citep{recht2019imagenet} and ImageNet-Sketch \citep{wang2019learning} for the six
primary classifiers. The conditional term is the larger: the mean per-(model,
domain) ratio $L_{\mathrm{comp}}/L_{\mathrm{total}}$ is $41.8\%$, $11.1\%$ and
$9.6\%$, and on ImageNet-C it falls from $50.6\%$ at severity~1 to $33.0\%$ at~5
(Figure~\ref{fig:severity}, Table~\ref{tab:natural}). The median utilisation over
the $9{,}177$ fold-pooled bin records, $8{,}937$ of them from ImageNet-C, is $0.64$
(per-benchmark medians of domain medians $0.63$, $0.37$, $0.89$). Re-estimating
source rates on held-out data inside the same partition moves the ratio by at most
\NatShareShiftDisjoint{} per checkpoint, a check of the rate estimator only. Rerunning the revised
pipeline end to end on ImageNetV2 and ImageNet-Sketch moves the ratio from $12.2\%$
to $12.0\%$ and from $9.8\%$ to $10.0\%$ at matched aggregation, with the conditional
term larger in all $30$ model--fold units of each (Appendix~\ref{app:ext:natural}). ImageNet-C was not rerun, so its severity trend and
the association below remain original-partition results. Removing bins flagged by
the original target-count rule within each domain (post-hoc;
Table~\ref{tab:r8-lowcount}) keeps \LcRetainedMass{} of target mass
on average and gives a mean ratio of \LcRatioRet{} instead of \LcRatioAll{}, a change
of estimand rather than a revised-pipeline rerun. Across its $75$ domains a
preregistered analysis found source FAUC associated with the composition term
(mean within-domain Spearman $+0.63$, interval $[+0.54,+0.72]$) and the ECE interval
including zero, with neither associated with the total or conditional term; the
single-domain benchmarks disagree, so the preregistered outcome was mixed
(Appendix~\ref{app:phase7}). Part of the FAUC association is structural, since both
use source cell-rate contrasts, but the target proportions are observed, so it is
not an identity; we read it as consistent with the decomposition, not as a
validated forecast.

\section{Related work}
\label{sec:related}

\paragraph{Grouping loss.}
Calibration--refinement decompositions separate miscalibration from the
heterogeneity a score cannot express \citep{degroot1983comparison,kull2015novel};
\citet{perezlebel2023grouping} estimate this grouping loss for neural networks and
show that binning inflates it, and later work relates it to decisions and
epistemic uncertainty \citep{perezlebel2025decision,melo2026epistemic}. These
analyses are static; we ask what the heterogeneity permits a shift to do.

\paragraph{Distributionally robust evaluation.}
We use robust expectations over $\varphi$-divergence balls, the exact $\chi^2$
dual and its non-negativity condition as established tools
\citep{bental2013robust,duchi2019variance,duchi2021learning,lam2016robust,gotoh2018robust}.
Closest in purpose, \citet{li2021evaluating} evaluate the worst performance over
subpopulations of a given size defined by core attributes, and
\citet{duchi2023distributionally} define such risks through conditional risks. Our
set instead fixes the score marginal, so the shift is invisible to score-only
monitoring, and budgets the reweighting within each confidence level; a learned
readout plays the role of the attributes inside bins. \citet{namkoong2026local}
restricts admissible transports and obtains first-order Wasserstein sensitivities;
our local law holds exactly on an interval. \citet{subbaswamy2021evaluating}
evaluate worst-case performance when selected conditional distributions shift; here
the preserved marginal is the confidence.

\paragraph{Calibration and performance estimation under shift.}
Temperature scaling recalibrates reported probabilities \citep{guo2017calibration},
and calibration degrades under shift \citep{ovadia2019trust}. Covariate-shift
methods use target inputs to reweight calibration, conformal sets or risk control
\citep{tibshirani2019conformal,park2020calibrated,almeida2025high}, where a
score-preserving shift can be visible; estimators that see only target confidences,
such as average thresholded confidence \citep{garg2022leveraging}, return the same
estimate for every such shift. Multicalibration \citep{hebertjohnson2018multicalibration}
controls calibration on specified groups; fragility measures heterogeneity left
within a level.

\section{Limitations and conclusion}
\label{sec:conclusion}

The theory concerns exact level sets and $\eta$; the empirical objects condition on
bins and a learned readout, and the binned oracle includes within-bin variation of
$\mu_P$, so nothing we report isolates level-set heterogeneity. Plug-in profiles
and FAUC have no coverage guarantee; the lower confidence bound has one per analysis but is
weak after temperature scaling. Our finite-sample comparison with the capped variance
uses synthetic laws. CIFAR-10H
(Appendix~\ref{app:ext:human}) measures a repeated-annotation proxy for $\eta$, not
$\eta$. The eighteen checkpoints are not a sample of models; the twelve additions
and all post-hoc analyses were designed with earlier results known; and temperature
scaling can reorder the top-label score. Nothing here says that natural shifts
attain the worst case or that a source quantity bounds total natural drift.

A fixed confidence distribution does not fix reliability. How far it can move under
a score-preserving shift is governed by the within-level heterogeneity of the
correctness propensity: its variance gives the exact local rate and its centred
law, which the profile identifies, the rest. With a finite
readout the exposure becomes a restricted quantity that can be estimated and
bounded below; held-out drift under optimised reweightings follows the estimate, which an
exploratory rate-noise null does not reproduce, and on natural shifts most reliability change
occurs within cells, outside the source profile's control. The open problem is
resolution: in the one dataset with a proxy for $\eta$, a learned readout exposes a
few percent of the within-bin variance, and Proposition~\ref{prop:approx} bounds what
the rest could add only through an unidentified quantity.
\label{endofmaintext}

\label{startofstatements}
\subsection*{Reproducibility statement}

Proofs, assumptions and degenerate cases are in Appendices~\ref{app:conventions}--\ref{app:witness},
the split-sample lower confidence bound in Appendix~\ref{app:fscert} and the
reparameterisation lemma in Appendix~\ref{app:monotone}. Appendix~\ref{app:estimators}
specifies the data roles, estimators, retention rules and the label-permutation
test of the rate estimator. Six analyses were governed by protocols written and
hashed before their results were computed: the pilot, the multi-architecture study,
the natural-shift study, the natural-target analysis (Appendix~\ref{app:phase7}),
the lower confidence bound (Appendix~\ref{app:fscert}) and the monotone
reparameterisation (Appendix~\ref{app:monotone}). After those results were frozen,
an internal audit found that the original rate estimator selected cells and fitted
their rates on the same labels. The main tables are therefore recomputed under
separated data roles and reported as post-audit, not preregistered, analyses.
Appendix~\ref{app:claims} records, analysis by analysis, what was recomputed and
what was not. The later extensions (Appendices~\ref{app:ext} and~\ref{app:breadth}),
the diagnostics of Appendix~\ref{app:r7} and the known-propensity study of
Appendix~\ref{app:r9} are post-hoc; their protocols were hashed before the
corresponding results were computed, except where Appendix~\ref{app:r7} states
otherwise, and Appendix~\ref{app:r9} records one technical amendment. The released hashes let a reader verify
that files are unchanged; they are not independent timestamps. A supplementary
package containing the protocols, code, checkpoint identifiers and checksums,
frozen result files, the generators of every table, macro and figure in the
appendices, and verification scripts accompanies this work and will be released
publicly. Verification from the frozen results
needs no data; an end-to-end rerun needs the public datasets and feature
extraction. All models are public pretrained checkpoints; none was trained or
fine-tuned by us.

\subsection*{AI use statement}

Generative AI tools were used as research assistants in the development of this
work, including conceptual and theoretical refinement; formulating, refining and
checking mathematical claims and proofs; feedback on experimental methodology;
implementation and debugging of research code; analysis and interpretation of
results; literature assistance; preparation of figures and tables; manuscript
drafting and editing; and adversarial review of claims. All AI-assisted content was
reviewed and verified
by the authors. In particular, theoretical claims were checked against their
stated assumptions and accompanying proofs, experimental claims against frozen
outputs and the corresponding analysis specifications, including recorded
post-hoc amendments, and the reported leakage controls and
numerical identities through explicit computational tests. The authors take full
responsibility for the final manuscript, mathematical claims, experimental
results, and associated artifacts.

\bibliographystyle{preprint}
\label{startofbib}
\bibliography{refs}

\appendix
\clearpage

\section{Notation index}\label{app:notation}

The same letters recur at four levels --- population, binned, restricted and
estimated --- and the table below fixes which object each symbol names. Unadorned
$G_R(b)$ and $\mathcal F^{\pm}_{B,R}(b,\rho)$ are population restricted targets in
the statements of results; tildes and hats mark estimates. In tables and in the
reporting of empirical values, $G_R$ and FAUC denote the reported estimates, as
their captions state, and a bin argument is suppressed only where the surrounding
text has fixed the bin.

\begin{center}\footnotesize\setlength{\tabcolsep}{4pt}
\begin{tabular}{p{0.30\textwidth}p{0.50\textwidth}p{0.13\textwidth}}
\toprule
symbol & meaning & level \\
\midrule
$X$, $C$, $S=s(X)$, $s$ & input; correctness $\mathbf 1\{\hat Y=Y\}$; confidence random variable; its value & population \\
$\eta(X)=\Pr(C=1\mid X)$ & correctness propensity & population \\
$\mu_P(s)$, $e_P(s)=\mu_P(s)-s$ & reliability relation; calibration residual & population \\
$G_P(s)$, $Z=\eta-\mu_P(s)$ & within-level grouping heterogeneity; centred propensity & population \\
$r=\mathrm dQ/\mathrm dP$, $\rho$, $\mathcal A_s(\rho)$ & score-preserving reweighting; conditional $\chi^2$ budget; feasible set & population \\
$\mathcal F^{\pm}_s(\rho)$ & fragility profile at level $s$ & population \\
$z_\pm$, $p_+$, $\rho_\pm$, $\rho_{\mathrm{sat}}$, $M_1,M_2,\kappa$ & extrema and top atom of $Z$; regime boundaries; stop-loss transforms & population \\
\midrule
$B$, $b$, $J$; $R$, $k$, $K$ & confidence partition, bin index, number of bins; readout, group index, number of groups & binned \\
$p_b$, $q_b$ & source and target bin mass & binned \\
$p_{bk}$ (or $p_k$ at fixed $b$), $q_{bk}$, $q_k$ & source and target within-bin group mass & binned \\
$\mu_P(b)$, $G_P(b)$, $\mathcal F^{\pm}_B(b,\rho)$ & binned oracle mean, heterogeneity and profile & binned oracle \\
$\theta_P(b,k)$, $\theta_Q(b,k)$ & source and target cell correctness rates & restricted \\
$G_R(b)$, $\mathcal F^{\pm}_{B,R}(b,\rho)$ & restricted heterogeneity and profile (population lower bounds on binned oracle quantities) & restricted \\
\midrule
$D_{\mathrm{fit}}$, $D_{\mathrm{rate}}$, $D_{\mathrm{eval}}$ & data roles of the revised pipeline & estimation \\
$\tilde\theta$, $\hat\theta$, $\hat p$ & raw binomial cell proportions; plug-in cell rates; evaluation proportions & estimation \\
$\widehat G_w$, $\widehat G_R$ & debiased restricted heterogeneity (Proposition~\ref{prop:gunbiased}) & estimation \\
$\widehat{\mathcal F}^{\pm}_{B,R}(b,\rho)$, $F^{+}(\rho)$, FAUC & plug-in profile; its bin-mass quadratic mean; its log-radius integral & estimation \\
\midrule
$D_R(b)$ & observed within-bin composition radius $\sum_{k:\,p_k>0}p_k(q_k/p_k-1)^2$ & natural shift \\
$\delta_{\mathrm{total}}(b)$, $\delta_{\mathrm{comp}}(b)$, $\delta_{\mathrm{cond}}(b)$ & binwise drift and its exact composition/conditional split & natural shift \\
$L_{\mathrm{total}}$, $L_{\mathrm{comp}}$, $L_{\mathrm{cond}}$; $L_{\mathrm{comp}}/L_{\mathrm{total}}$ & target-mass quadratic means of binwise drift; composition-to-total norm ratio (not a share) & natural shift \\
$U(b)$ & envelope utilisation $|\delta_{\mathrm{comp}}(b)|/\mathcal F^{\pm}_{B,R}(b,D_R(b))$ & natural shift \\
\midrule
$L_{bk}$, $U_{bk}$, $l_{bk}$, $l_b$ & simultaneous lower/upper bounds on $\theta_P(b,k)$, $p_{bk}$, $p_b$ & inference \\
$d$, $d^\ast$; $u$, $v$ & mass moved and largest feasible step; source and destination group & inference \\
$\mathrm{LCB}^{\pm}_b(\rho)$, $\mathrm{LCB}^{+}_{\mathrm{model}}(\rho)$ & split-sample lower confidence bound, per bin and aggregated & inference \\
\midrule
$T$, $\hat T$; $g$, $g_f$, $S'=g(S)$ & temperature and its fit; fold-wise monotone score map and transformed score & interventions \\
$\eta_H$, $\eta_H^{A}$, $\eta_H^{B}$ & human-label proxy for $\eta$ and its two annotator-half versions & Appendix~\ref{app:ext} \\
$G_H(b)$, $G_H^{\times}(b)$ & human-proxy heterogeneity; its cross-half (vote-noise-corrected) form & Appendix~\ref{app:ext} \\
$G_R^{H}(b)$, $\mathcal F^{H}_{B,R}$, $\mathcal F^{H}_{B}$ & restricted and binned quantities built on $\eta_H$ (same-proxy) & Appendix~\ref{app:ext} \\
$a_k$, $t$, $r_k$, $\delta_{\mathrm{random}}$ & random score-preserving direction, step, weights and realised movement & Appendix~\ref{app:ext} \\
\bottomrule
\end{tabular}
\end{center}

\section{Conventions, measurability and what ``at a score level'' means}
\label{app:conventions}

The score $S=s(X)$ may be continuously distributed, so $\{S=s\}$ is typically
$P$-null and every level-set quantity below is defined through a regular
conditional distribution $P(\cdot\mid S=s)$, which exists because the underlying
spaces are standard Borel. Regular conditional distributions are unique only up
to a $P_S$-null set of $s$, so the functions
\[
  \mu_P(s)=\mathbb E_P[\eta\mid S=s],\qquad
  G_P(s)=\operatorname{Var}_P(\eta\mid S=s),\qquad
  \mathcal F^{\pm}_s
\]
are defined $P_S$-almost everywhere and no statement about a single fixed $s$
carries meaning unless $P_S(\{s\})>0$: any version may be edited on a null set.
We therefore fix one version of each throughout, and every level-set assertion in
this paper is quantified either \emph{$P_S$-almost everywhere} or \emph{on a
measurable set of positive $P_S$-measure}. Proposition~\ref{prop:invisible} is
stated in the second form for exactly this reason.

\paragraph{Score-preserving shifts are covariate shifts.}
Throughout, $Q\ll P$ and its density $r=\mathrm dQ/\mathrm dP$ is
\textbf{$\sigma(X)$-measurable}. This is a genuine restriction and it is used in
every transport statement: it says the shift reweights inputs and leaves the
labelling mechanism $P(Y\mid X)$ intact, so that $\eta_Q=\eta_P=\eta$. A density
depending on $Y$ would change $\eta$ itself and Proposition~\ref{prop:transport}
would fail. We record the elementary equivalence used repeatedly below.

\begin{lemma}[Score preservation]\label{lem:preserve}
Let $Q\ll P$ with density $r\ge0$. Then $Q_S=P_S$ if and only if
$\mathbb E_P[r\mid S]=1$ $P$-a.s.
\end{lemma}
\begin{proof}
For measurable $A$, $Q(S\in A)=\mathbb E_P[r\mathbf 1_A(S)]
=\mathbb E_P[\mathbf 1_A(S)\,\mathbb E_P[r\mid S]]$. This equals
$P(S\in A)=\mathbb E_P[\mathbf 1_A(S)]$ for every $A$ iff
$\mathbb E_P[r\mid S]=1$ a.s.
\end{proof}

Consequently the constraint $\mathbb E_P[r\mid S=s]=1$ in~\eqref{eq:set} is not
an extra modelling assumption layered on $Q_S=P_S$: it is the same condition.

\section{Proof of the reliability transport identity}\label{app:transport}

\begin{restate}
Let $Q\ll P$ with $\sigma(X)$-measurable density $r=\mathrm dQ/\mathrm dP$
satisfying $\mathbb E_P[r\mid S=s]=1$ for $P_S$-a.e.\ $s$. Then
$\mu_Q(s)-\mu_P(s)=\mathbb E_P[(r-1)Z\mid S=s]$ for $P_S$-a.e.\ $s$, where
$Z=\eta-\mu_P(s)$.
\end{restate}

\begin{proof}
Change of measure for conditional expectations gives, for bounded measurable $W$,
$\mathbb E_Q[W\mid S]=\mathbb E_P[rW\mid S]/\mathbb E_P[r\mid S]
=\mathbb E_P[rW\mid S]$, the denominator being $1$ by hypothesis. Take $W=C$.
Since $r$ is $\sigma(X)$-measurable and $\sigma(S)\subseteq\sigma(X)$, the tower
property over $\sigma(X)$ gives
$\mathbb E_P[rC\mid S]=\mathbb E_P\big[r\,\mathbb E_P[C\mid X]\mid S\big]
=\mathbb E_P[r\eta\mid S]$; all interchanges are justified because $r\ge0$ is
$P(\cdot\mid S=s)$-integrable with integral $1$ and $\eta\in[0,1]$ is bounded.
Hence
$\mu_Q(s)-\mu_P(s)=\mathbb E_P[(r-1)\eta\mid S=s]$, and since
$\mathbb E_P[r-1\mid S=s]=0$ the constant $\mu_P(s)$ may be subtracted from
$\eta$ without changing the value, giving $\mathbb E_P[(r-1)Z\mid S=s]$.
\end{proof}

No continuity of $s\mapsto\mu_P(s)$ is used, and no assumption is made on the
$\eta$-marginal beyond boundedness.

\section{Proof of score-marginal-invisible drift
(Proposition~\ref{prop:invisible})}\label{app:invisible}

We restate the proposition in the quantified form used in the main text.

\begin{restateinv}
Let $A$ be measurable with $P_S(A)>0$.
\emph{(i)} If $G_P>0$ $P_S$-a.e.\ on $A$, then for every $\rho>0$ there exists
$Q\ll P$ with $\sigma(X)$-measurable density satisfying~\eqref{eq:set} such that
$Q_S=P_S$ exactly, $\mu_Q>\mu_P$ $P_S$-a.e.\ on $A$, and $\mu_Q=\mu_P$
$P_S$-a.e.\ off $A$.
\emph{(ii)} If $G_P=0$ $P_S$-a.e.\ on $A$, then every such $Q$ satisfies
$\mu_Q=\mu_P$ $P_S$-a.e.\ on $A$.
\end{restateinv}

\begin{proof}
\emph{(i)} Fix $\varepsilon\in(0,\min\{1,\sqrt\rho\}]$ and set
\[
  r \;=\; 1+\varepsilon\,\big(\eta-\mu_P(S)\big)\,\mathbf 1\{S\in A\} .
\]
Choosing measurable versions of $\mu_P$ makes $r$ a measurable function of $X$.
Since $\eta\in[0,1]$ and $\mu_P\in[0,1]$ we have $|\eta-\mu_P(S)|\le1$, so
$r\ge1-\varepsilon\ge0$. For $P_S$-a.e.\ $s$, $\mathbb E_P[\eta-\mu_P(s)\mid
S=s]=0$, hence $\mathbb E_P[r\mid S=s]=1$; by Lemma~\ref{lem:preserve} the
measure $\mathrm dQ=r\,\mathrm dP$ is a probability measure with $Q_S=P_S$
exactly. Its conditional budget is
$\mathbb E_P[(r-1)^2\mid S=s]=\varepsilon^2G_P(s)\le\varepsilon^2\le\rho$ for
$s\in A$ and $0$ otherwise, so $r\in\mathcal A_s(\rho)$ for a.e.\ $s$. Finally
Proposition~\ref{prop:transport} gives
\[
  \mu_Q(s)-\mu_P(s)\;=\;\varepsilon\,\mathbb E_P[Z^2\mid S=s]\;=\;\varepsilon\,G_P(s),
\]
which is strictly positive for $P_S$-a.e.\ $s\in A$ and zero off $A$.

\emph{(ii)} If $G_P(s)=0$ then $Z=0$ $P(\cdot\mid S=s)$-a.s., so
$\mathbb E_P[(r-1)Z\mid S=s]=0$ for \emph{every} admissible $r$, and
Proposition~\ref{prop:transport} gives $\mu_Q(s)=\mu_P(s)$. This holds for
$P_S$-a.e.\ $s\in A$.
\end{proof}

\paragraph{Why the quantifiers matter.}
The pointwise reading ``there exists $Q$ with $\mu_Q(s)\neq\mu_P(s)$ at this
particular $s$'' is not a well-posed statement when $P_S(\{s\})=0$, because
$\mu_Q$ and $\mu_P$ are only determined up to $P_S$-null sets and can be edited
at $s$ at will. Part~(ii) also genuinely requires $Q\ll P$: a $Q$ charging inputs
outside the support of $P$ can change reliability even where $G_P=0$, and such a
$Q$ is not a reweighting of the source at all. This is one population-level
mechanism the conditional term of Section~\ref{sec:natural} can reflect; at a
finite readout that term can equally arise from partition coarseness, and the two
are not separable from the data.

\paragraph{What is and is not new here.}
Given Theorem~\ref{thm:profile} the mathematical content is immediate
($\mathcal F^\pm_s>0\iff G_P(s)>0$), and in static form the equivalence
``$G_P\equiv0$ iff the score is sufficient for correctness'' is standard in the
grouping-loss literature. Its role is to provide the dynamic interpretation used in
Section~\ref{sec:theory}.

\section{Proof of the full fragility profile (Theorem~\ref{thm:profile})}
\label{app:profile}

Throughout we work under $P(\cdot\mid S=s)$ and suppress $s$. Recall
$Z\in[-1,1]$, $\mathbb E Z=0$, $G=\mathbb E Z^2>0$, $z_\pm$ the essential
extrema, $p_+=\Pr(Z=z_+)$, and $M_1,M_2,\kappa$ as in~\eqref{eq:stoploss}. Since
$\mathbb EZ=0$ and $G>0$ we have $z_-<0<z_+$, so $\rho_+=G/z_-^2\in(0,\infty)$.
We use the convention $1/0:=+\infty$, so $p_+=0$ gives
$\rho_{\mathrm{sat}}=+\infty$.

We must evaluate
$\mathcal F^+(\rho)=\sup\{\mathbb E[rZ]:r\ge0,\ \mathbb E[r]=1,\
\mathbb E[(r-1)^2]\le\rho\}$.

\begin{lemma}[Concavity and monotonicity]\label{lem:concave}
$\rho\mapsto\mathcal F^+(\rho)$ is non-decreasing and concave on $[0,\infty)$,
with $\mathcal F^+(0)=0$.
\end{lemma}
\begin{proof}
Monotonicity is immediate since the feasible sets are nested. For concavity, let
$r_i$ be feasible at $\rho_i$ with $\mathbb E[r_iZ]\ge\mathcal
F^+(\rho_i)-\epsilon$ and $\lambda\in[0,1]$. Then
$r_\lambda=\lambda r_1+(1-\lambda)r_2$ satisfies $r_\lambda\ge0$ and
$\mathbb E[r_\lambda]=1$, and by convexity of $t\mapsto t^2$,
$\mathbb E[(r_\lambda-1)^2]\le\lambda\rho_1+(1-\lambda)\rho_2$; it is therefore
feasible at $\lambda\rho_1+(1-\lambda)\rho_2$ and attains at least
$\lambda\mathcal F^+(\rho_1)+(1-\lambda)\mathcal F^+(\rho_2)-\epsilon$.
\end{proof}

\subsection{A verification argument}

We prove optimality by direct verification. The following bound holds for every
feasible reweighting and is attained by the proposed construction, which proves
optimality and attainment together.

\begin{lemma}[Verification bound]\label{lem:verify}
Let $\tau<z_+$ so that $M_1(\tau)>0$, and let $r$ be any $r\ge0$ with
$\mathbb E[r]=1$ and $\mathbb E[(r-1)^2]\le\kappa(\tau)$. Then
\[
   \mathbb E[rZ]\;\le\;\tau+\frac{M_2(\tau)}{M_1(\tau)},
\]
with equality for $r^\ast=(Z-\tau)_+/M_1(\tau)$, which is feasible and spends the
budget exactly.
\end{lemma}
\begin{proof}
Write $Z=\tau+(Z-\tau)_+-(\tau-Z)_+$. Since $r\ge0$ and $(\tau-Z)_+\ge0$,
\[
  \mathbb E[rZ]\;=\;\tau+\mathbb E\big[r(Z-\tau)_+\big]-\mathbb E\big[r(\tau-Z)_+\big]
  \;\le\;\tau+\mathbb E\big[r(Z-\tau)_+\big].
\]
By Cauchy--Schwarz, $\mathbb E[r(Z-\tau)_+]\le\|r\|_2\sqrt{M_2(\tau)}$, and
$\|r\|_2^2=\mathbb E[r^2]=1+\mathbb E[(r-1)^2]\le1+\kappa(\tau)
=M_2(\tau)/M_1(\tau)^2$. Combining,
$\mathbb E[r(Z-\tau)_+]\le M_2(\tau)/M_1(\tau)$, which is the claim.
For $r^\ast$: it is non-negative, $\mathbb E[r^\ast]=1$ by construction,
$\mathbb E[(r^\ast-1)^2]=\mathbb E[(r^\ast)^2]-1=M_2/M_1^2-1=\kappa(\tau)$, and
both inequalities above are equalities because $r^\ast(\tau-Z)_+=0$ pointwise and
$r^\ast$ is proportional to $(Z-\tau)_+$. Its value is
$\mathbb E[r^\ast Z]=\big(M_2(\tau)+\tau M_1(\tau)\big)/M_1(\tau)
=\tau+M_2(\tau)/M_1(\tau)$.
\end{proof}

\begin{lemma}[Monotonicity of $\kappa$]\label{lem:kappa}
$\kappa$ is continuous and non-decreasing on $(-\infty,z_+)$, with
$\kappa(\tau)\to0$ as $\tau\to-\infty$. It is \emph{strictly} increasing on
$(-\infty,\tau_{\max})$, where $\tau_{\max}=\sup\{\tau:\Pr(Z>\tau)>p_+\}$, and is
constant equal to $1/p_+-1$ on $[\tau_{\max},z_+)$. Hence
$\kappa:(-\infty,\tau_{\max})\to(0,\rho_{\mathrm{sat}})$ is a continuous
strictly increasing bijection.
\end{lemma}
\begin{proof}
$M_2$ is continuously differentiable with $M_2'=-2M_1$ everywhere (dominated
convergence; the integrand $(Z-\tau)_+^2$ has the bounded $\tau$-derivative
$-2(Z-\tau)_+$). $M_1$ is convex with one-sided derivatives
$M_1'{}_+(\tau)=-\Pr(Z>\tau)$ and $M_1'{}_-(\tau)=-\Pr(Z\ge\tau)$, equal off the
countable set of atoms. Writing $\bar F$ for either one-sided survival value,
\[
  \kappa'(\tau)=\frac{M_2'M_1-2M_2M_1'}{M_1^{3}}
  =\frac{2\big(M_2\bar F-M_1^{2}\big)}{M_1^{3}}\;\ge\;0 ,
\]
because Cauchy--Schwarz on $\{Z>\tau\}$ (resp.\ $\{Z\ge\tau\}$) gives
$M_1=\mathbb E[(Z-\tau)\mathbf 1]\le\sqrt{M_2\bar F}$. Equality forces $Z-\tau$
to be a.s.\ constant on that event. On the open event this means that only one
value of $Z$ lies above $\tau$; on the closed event, if $Z$ has an atom at $\tau$,
it additionally forces $\Pr(Z>\tau)=0$, and if it has none the two events coincide.
Either way, equality is impossible while more than the top atom survives above
$\tau$, so $\kappa'>0$ strictly for $\tau<\tau_{\max}$, and for $\tau\ge\tau_{\max}$ we have
$M_1=p_+(z_+-\tau)$, $M_2=p_+(z_+-\tau)^2$ and $\kappa\equiv1/p_+-1$. For
$\tau\le z_-$, $M_1=-\tau$ and $M_2=G+\tau^2$, so $\kappa=G/\tau^2\to0$.
\end{proof}

\subsection{The three regimes}

\emph{Regime (i).} Dropping $r\ge0$, Cauchy--Schwarz gives
$\mathbb E[(r-1)Z]\le\|r-1\|_2\|Z\|_2\le\sqrt{\rho G}$ with equality iff
$r-1=\sqrt{\rho/G}\,Z$. That candidate satisfies $r\ge0$ iff
$1+\sqrt{\rho/G}\,z_-\ge0$, i.e.\ iff $\rho\le G/z_-^2=\rho_+$. Hence
$\mathcal F^+(\rho)=\sqrt{\rho G}$ \emph{with equality on the closed interval}
$[0,\rho_+]$, attained by $r^\ast=1+\sqrt{\rho/G}\,Z$. This proves
Corollary~\ref{cor:local}, whose germ statement follows on dividing by $\rho$
and letting $\rho\downarrow0$. Equivalently, by Lemma~\ref{lem:kappa} the root of
$\kappa(\tau_\rho)=\rho$ satisfies $\tau_\rho=-\sqrt{G/\rho}\le z_-$ exactly when
$\rho\le\rho_+$, and Lemma~\ref{lem:verify} then returns $\sqrt{\rho G}$; the two
descriptions agree.

\emph{Regime (ii).} For $\rho_+<\rho<\rho_{\mathrm{sat}}$, Lemma~\ref{lem:kappa}
provides a unique $\tau_\rho\in(z_-,\tau_{\max})$ with $\kappa(\tau_\rho)=\rho$,
and Lemma~\ref{lem:verify} shows the supremum equals
$\tau_\rho+M_2(\tau_\rho)/M_1(\tau_\rho)$ and is attained by the
truncated-affine $r^\ast=(Z-\tau_\rho)_+/M_1(\tau_\rho)$.

\emph{Regime (iii).} For $\rho\ge\rho_{\mathrm{sat}}$ the reweighting
$r=\mathbf 1\{Z=z_+\}/p_+$ is feasible --- it spends exactly
$1/p_+-1\le\rho$ --- and attains $\mathbb E[rZ]=z_+$, which is the trivial upper
bound $\mathbb E[rZ]\le z_+\mathbb E[r]=z_+$.

\paragraph{Empty regimes, and why that is not an error.}
The three regimes partition $[0,\infty)$ but any of the last two may be empty.
If $p_+=0$ then $\rho_{\mathrm{sat}}=\infty$, regime (iii) is empty, $\kappa$ is
strictly increasing on all of $(-\infty,z_+)$, and
$\mathcal F^+(\rho)\uparrow z_+$ \emph{without} attainment. If $Z$ is two-valued
then $\rho_+=\rho_{\mathrm{sat}}$ exactly --- writing $z_-=-p_+z_+/(1-p_+)$ gives
$G=p_+z_+^2/(1-p_+)$ and hence $G/z_-^2=(1-p_+)/p_+$ --- so regime (ii) is empty
and the profile is $\sqrt{\rho G}$ followed immediately by saturation, the two
branches meeting continuously at $z_+$. No expression of the form $1/p_+$ is
evaluated at $p_+=0$ anywhere above.

\paragraph{Degenerate case.} If $G=0$ then $Z=0$ a.s., every admissible
reweighting gives $\mathbb E[rZ]=0$, and $\mathcal F^\pm\equiv0$; the thresholds
$\rho_\pm=G/z_\mp^2$ are $0/0$ and undefined, which is why $G>0$ is assumed.

\section{Proof of profile identifiability (Theorem~\ref{thm:ident})}
\label{app:ident}

The derivative identity below is often quoted as an envelope-theorem
consequence, and it can also be read off the scalar dual recalled under ``What is
classical here'' by Danskin's theorem. We prove it instead by direct
differentiation of the explicit parametrisation, because that route makes the
behaviour at atoms of $Z$ explicit.

\begin{lemma}[Derivative of the profile]\label{lem:envelope}
On $(0,\rho_{\mathrm{sat}})$, $\mathcal F^+$ is differentiable with
$\frac{\mathrm d}{\mathrm d\rho}\mathcal F^+(\rho)=\tfrac12M_1(\tau_\rho)$.
\end{lemma}
\begin{proof}
Write $g(\tau)=\tau+M_2(\tau)/M_1(\tau)$, so that
$\mathcal F^+(\rho)=g(\kappa^{-1}(\rho))$ on $(0,\rho_{\mathrm{sat}})$ by
Lemmas~\ref{lem:verify} and~\ref{lem:kappa}. Fix $\tau<\tau_{\max}$ and let
$\bar F$ denote $\Pr(Z>\tau)$ for right derivatives and $\Pr(Z\ge\tau)$ for left
derivatives. Using $M_2'=-2M_1$ and $M_1'=-\bar F$,
\[
  g'(\tau)=1+\frac{M_2'M_1-M_2M_1'}{M_1^{2}}
          =-1+\frac{M_2\bar F}{M_1^{2}},
  \qquad
  \kappa'(\tau)=\frac{2\big(M_2\bar F-M_1^{2}\big)}{M_1^{3}} ,
\]
and therefore
\[
  \frac{g'(\tau)}{\kappa'(\tau)}
  =\frac{\big(M_2\bar F-M_1^{2}\big)/M_1^{2}}
        {2\big(M_2\bar F-M_1^{2}\big)/M_1^{3}}
  =\frac{M_1(\tau)}{2},
\]
the cancellation being legitimate because $M_2\bar F-M_1^2>0$ strictly for
$\tau<\tau_{\max}$ (Lemma~\ref{lem:kappa}). The right-hand side does not depend
on which one-sided survival value was used. Since $\kappa$ is a strictly
increasing bijection, the chain rule for one-sided derivatives gives
$\mathcal F^{+\prime}_{\pm}(\rho)=g'_\pm(\tau_\rho)/\kappa'_\pm(\tau_\rho)
=\tfrac12M_1(\tau_\rho)$ for both signs; the two coincide, so $\mathcal F^+$ is
differentiable at every $\rho\in(0,\rho_{\mathrm{sat}})$, including at radii
whose $\tau_\rho$ is an atom of $Z$.
\end{proof}

\begin{lemma}[Truncation recovery]\label{lem:tau}
For $\rho\in(0,\rho_{\mathrm{sat}})$,
$\tau_\rho=\mathcal F^+(\rho)-2(\rho+1)\frac{\mathrm d}{\mathrm
d\rho}\mathcal F^+(\rho)$.
\end{lemma}
\begin{proof}
$\rho+1=\kappa(\tau_\rho)+1=M_2/M_1^2$, so $(\rho+1)M_1=M_2/M_1$; subtracting
this from $\mathcal F^+=\tau_\rho+M_2/M_1$ leaves $\tau_\rho$. Now substitute
$M_1=2\,\mathrm d\mathcal F^+/\mathrm d\rho$ from Lemma~\ref{lem:envelope}.
\end{proof}

\begin{proof}[Proof of Theorem~\ref{thm:ident}]
We reconstruct $\mathrm{Law}(Z)$ from the function $\mathcal F^+$ alone.

\emph{Step 1 (upper endpoint).} $\mathbb E[rZ]\le z_+\mathbb E[r]=z_+$ for every
admissible $r$; taking $r$ uniform on $\{Z>z_+-\varepsilon\}$, an event of
positive probability by definition of the essential supremum and of finite
budget $1/\Pr(Z>z_+-\varepsilon)-1$, shows the bound is approached. Hence
$z_+=\lim_{\rho\to\infty}\mathcal F^+(\rho)$.

\emph{Step 2 (top atom).} By regime~(iii) and Lemma~\ref{lem:kappa},
$\mathcal F^+(\rho)=z_+$ exactly when $\rho\ge1/p_+-1$. Thus
$\rho_{\mathrm{sat}}=\inf\{\rho:\mathcal F^+(\rho)=z_+\}$ and
$p_+=1/(1+\rho_{\mathrm{sat}})$, with $p_+=0$ when the infimum is infinite.

\emph{Step 3 (the stop-loss transform).} By Lemma~\ref{lem:kappa},
$\rho\mapsto\tau_\rho$ is a continuous strictly increasing bijection from
$(0,\rho_{\mathrm{sat}})$ onto $(-\infty,\tau_{\max})$. Lemmas~\ref{lem:envelope}
and~\ref{lem:tau} return the pair $\big(\tau_\rho,\,M_1(\tau_\rho)\big)$ from
$\mathcal F^+$ and its derivative alone, so the graph of $M_1$ over
$(-\infty,\tau_{\max})$ is traced exactly once.

\emph{Step 4 (from $M_1$ to the law).} $M_1(\tau)=\int_\tau^{z_+}\Pr(Z>t)\,
\mathrm dt$ is convex and non-increasing with right derivative $-\Pr(Z>\tau)$
everywhere; a convex function determines its one-sided derivatives at every
point, so $\Pr(Z>\tau)$ is determined for all $\tau<\tau_{\max}$. On
$[\tau_{\max},z_+)$ the survival function equals $p_+$, supplied by Steps~1--2.
The survival function of $Z$ is therefore determined on all of $\mathbb R$, and
injectivity follows.
\end{proof}

\paragraph{Why the upward profile alone suffices.} The natural objection is that
$\mathcal F^{+}$ interrogates only upper truncations $M_1(\tau)=\mathbb
E[(Z-\tau)_+]$, and that regime~(i), where $\mathcal F^{+}(\rho)=\sqrt{\rho G}$
depends on the law through $G$ alone, can carry nothing further. Both
observations are correct, and neither obstructs identification. Step~4 recovers
$\Pr(Z>\tau)$ for \emph{all} $\tau<\tau_{\max}$, not merely near the top, because
$\tau_\rho=-\sqrt{G/\rho}\to-\infty$ as $\rho\downarrow0$. For $\tau\le z_-$ one
has $(Z-\tau)_+=Z-\tau$ identically, so $M_1(\tau)=-\tau$ for \emph{every}
centred law: that range carries no information and needs none, since the survival
function is $1$ throughout it. The informative range is exactly
$\tau\in(z_-,\tau_{\max})$, which regime~(ii) traverses, and its lower endpoint
is pinned by the end of the variance-controlled regime, $\tau_{\rho_+}=z_-$, i.e.\
$z_-=-\sqrt{G/\rho_+}$. An atom at $z_-$ is likewise recovered, as the jump of
the survival function there, $M_1$ being convex with one-sided derivatives
everywhere. The lower tail is thus determined not by probing downward, but by
knowing where the upward probe ceases to be informative. Nothing about the law
is left undetermined, and $\mathcal F^-$ is not needed.

\paragraph{What is classical here.} Step~4 is the classical uniqueness of a
distribution given its stop-loss (integrated survival) transform
\citep{gushchin2018integrated}, and identifiability of a law from a full
one-parameter family of risk measures is well known. The closed form of
Theorem~\ref{thm:profile} is likewise not ours. With $f_2(t)=(t-1)^2/2$ the budget
$\mathbb E[(r-1)^2]\le\rho$ reads $D_{f_2}(Q\,\|\,P)\le\rho/2$, and the Cressie--Read
dual \citep[Lemma~1]{duchi2021learning} gives
\[
  \mathcal F^+(\rho)=\inf_{\eta\in\mathbb R}\Big\{\eta+\sqrt{(1+\rho)\,M_2(\eta)}\Big\},
\]
the second-order higher-moment coherent risk measure of
\citet{krokhmal2007higher}. Its first-order condition is $\kappa(\eta)=\rho$ and
its value is $\tau_\rho+M_2(\tau_\rho)/M_1(\tau_\rho)$, which is regime~(ii) of
Theorem~\ref{thm:profile}; Lemma~\ref{lem:envelope} also follows from it by
Danskin's theorem, given the uniqueness of the minimiser supplied by
Lemma~\ref{lem:kappa}. We state Theorem~\ref{thm:profile} because the calibration
argument needs its regime structure. We claim neither the stop-loss uniqueness
theorem nor the $\chi^2$ robust dual, but the fact that the entire
score-preserving fragility curve can be inverted to recover the centred
within-score propensity law, and its calibration reading.

\paragraph{Failure mode.} If $\kappa$ were constant on an interval strictly
inside $(-\infty,\tau_{\max})$ then $\tau_\rho$ would not be injective and a
segment of $M_1$ would be unrecoverable. Lemma~\ref{lem:kappa} rules this out:
flatness requires $Z-\tau$ a.s.\ constant above $\tau$, which happens only once a
single atom remains.

\paragraph{Regularity, and the scope of ``complete''.}
$\mathcal F^+\in C^1(0,\rho_{\mathrm{sat}})$ even when $Z$ is purely atomic, by
Lemma~\ref{lem:envelope}: the profile is smoother than the law generating it.
Note what is and is not identified. $\mathcal F^+$ is a functional of
$Z=\eta-\mu_P(s)$ and is unchanged if $\eta$ is translated, so it does not
determine $\mu_P(s)$. It is not silent about the mean either: $\eta\in[0,1]$
forces $\operatorname{supp}(Z)\subseteq[-\mu_P(s),1-\mu_P(s)]$, so recovering
$z_\pm$ confines $\mu_P(s)$ to $[-z_-,\,1-z_+]$, an interval that is
non-degenerate in general and collapses to a point when $\eta$ is
$\{0,1\}$-valued. We do not claim this exhausts the relations between the two.
Not determining the mean is in any case not a gap: $\mu_P(s)$ is exactly what the
\emph{reliability relation} reports. The pair (reliability relation, profile) therefore determines
the full conditional law of $\eta$ given $S=s$, and over general propensity laws
neither component determines the other.

We are deliberately precise about which object supplies the mean. A scalar
calibration error is an \emph{aggregate} of the reliability relation over the
score --- $\operatorname{ECE}\approx\mathbb E_S|\mu_P(S)-S|$ and its relatives ---
and no such aggregate determines $\mu_P(s)$ at a given level. So the
identification statement above is not ``ECE plus the fragility curve identifies
the conditional law''; it is that the reliability relation and the profile do,
one supplying the conditional mean at each level and the other the centred law
there. What scalar calibration error \emph{cannot} do --- determine the profile
--- is the claim of this paper, and it is unaffected.

\paragraph{Identifiability is not stability.} Theorem~\ref{thm:ident} is a
population statement about the exact curve. The reconstruction runs through
$\mathrm d\mathcal F^+/\mathrm d\rho$ and we prove no modulus of continuity for
the inverse map. We never attempt the inversion on an estimated curve: the
seven-point radius grid of Appendix~\ref{app:estimators} is far too coarse, which
is why we report a deliberately lossy scalar instead. Nothing here supports a
claim of stable recovery.

\paragraph{Numerical confirmation.} On the frozen synthetic laws the derivative
identity holds to $\le3.0\times10^{-7}$ (central-difference limited) and the
reconstructed survival function to $\le2.7\times10^{-6}$, with $(z_+,p_+)$
recovered to six decimals. The one-sided form of Lemma~\ref{lem:envelope} was
checked at every atom of every frozen law (max deviation $6.1\times10^{-16}$) and
over randomly generated laws; the worst case observed there,
$1.0\times10^{-9}$, occurs where Cauchy--Schwarz is nearly tight so that
$\kappa'\approx0$, and recomputing that case in exact rational arithmetic returns
$0$ on both sides.

\section{Downward fragility and directional asymmetry}\label{app:downward}
Applying Appendices~\ref{app:profile}--\ref{app:ident} to $-Z$ gives
$\mathcal F^-$, with $\rho_-=G/z_+^2$ and
$\rho^-_{\mathrm{sat}}=1/\Pr(Z=z_-)-1$. Because $\mathbb EZ=0$ forces
$z_-<0<z_+$ but not $|z_-|=|z_+|$, the two thresholds differ whenever the
conditional law is asymmetric, and the upward and downward profiles saturate at
different levels. Either component determines the law by
Theorem~\ref{thm:ident}; reporting both is a convenience, not a requirement.

\section{Divergence-universality of the local modulus}\label{app:fdiv}
Let $f$ be convex with $f(1)=0$, twice differentiable at $1$ with $f''(1)>0$.
Since $\mathbb E[r-1]=0$ we may take $f'(1)=0$ without loss. For the ambiguity
set $\mathbb E[f(r)]\le\rho$,
\[
\mathcal F_f^{\pm}(\rho)=\sqrt{\frac{2\rho\,G}{f''(1)}}+o(\sqrt\rho),
\qquad\rho\downarrow0 ,
\]
a known expansion of $\varphi$-divergence robust expectations
\citep{lam2016robust,gotoh2018robust}, for which we claim no novelty. We give the
argument because a Taylor expansion at $u=0$ does not supply it on its own: a
feasible $r$ need not be uniformly near $1$, since mass $\varepsilon$ at weight
$1+a$ costs only $\varepsilon f(1+a)$. Normalise $f(1)=f'(1)=0$, write $u=r-1$ and
$a=f''(1)>0$, and treat $G=0$ separately (then $Z=0$ a.s.\ and both sides vanish).

Fix $\delta\in(0,1)$ and split at $|u|\ge\delta$. Because $f$ is convex with
$f(1)=0$, the difference quotient $u\mapsto f(1+u)/u$ is non-decreasing, so
$f(1+u)\ge\big(f(1+\delta)/\delta\big)u$ for $u\ge\delta$ and
$f(1+u)\ge\big(f(1-\delta)/\delta\big)|u|$ for $u\le-\delta$. With the
\emph{secant} constant
\[
  c^{\ast}_\delta=\min\Big\{\tfrac{f(1+\delta)}{\delta},\,
                            \tfrac{f(1-\delta)}{\delta}\Big\}>0
\]
both tails therefore obey $f(1+u)\ge c^{\ast}_\delta|u|$. Convexity with
$f(1)=f'(1)=0$ also gives $f\ge0$ everywhere, so the near region contributes
non-negatively to $\mathbb E[f(r)]$ and feasibility gives
\[
  \mathbb E\big[|u|\,\mathbf 1\{|u|\ge\delta\}\big]\;\le\;\rho/c^{\ast}_\delta ,
  \qquad\text{so}\qquad
  \mathbb E\big[uZ\,\mathbf 1_{\mathrm{far}}\big]\le\|Z\|_\infty\rho/c^{\ast}_\delta .
\]
A half-derivative constant $\tfrac12\min\{f'(1+\delta),|f'(1-\delta)|\}$ would
also work, but only from $|u|\ge2\delta$: at $|u|=\delta$ it can exceed the true
secant slope, as $f(1+u)=u^2/2+u^4$ shows. The secant constant avoids the
mismatch.

On the near region, twice differentiability at $1$ gives
$f(1+u)\ge(\tfrac a2-\varepsilon_\delta)u^2$ for $|u|<\delta$ with
$\varepsilon_\delta\to0$ as $\delta\downarrow0$, so
$\mathbb E[u^2\mathbf 1_{\mathrm{near}}]\le\rho/(\tfrac a2-\varepsilon_\delta)$ and
Cauchy--Schwarz bounds that part by
$\sqrt{\rho G/(\tfrac a2-\varepsilon_\delta)}$, uniformly over feasible $r$.
Adding the two,
\[
  \mathcal F^{\pm}_f(\rho)\;\le\;
  \sqrt{\frac{\rho\,G}{\tfrac a2-\varepsilon_\delta}}
  \;+\;\frac{\|Z\|_\infty}{c^{\ast}_\delta}\,\rho .
\]
Dividing by $\sqrt\rho$ and letting $\rho\downarrow0$ \emph{first} kills the
second term for each fixed $\delta$; letting $\delta\downarrow0$ afterwards gives
$\limsup_{\rho\downarrow0}\mathcal F^{\pm}_f(\rho)/\sqrt\rho\le\sqrt{2G/a}$.

For the lower bound fix $\epsilon\in(0,1)$ and take
$r=1\pm tZ$ with $t=(1-\epsilon)\sqrt{2\rho/(aG)}$. Since $|Z|\le1$, $r\ge0$ for
$\rho$ small; $\mathbb E[r]=1$ because $\mathbb EZ=0$; and
$\mathbb E[f(r)]=\tfrac a2t^2G\,(1+o(1))=(1-\epsilon)^2\rho\,(1+o(1))\le\rho$ for
$\rho$ small, the slack $\epsilon$ absorbing the higher-order terms that a
leading-order $t$ would leave uncontrolled. Its value is
$tG=(1-\epsilon)\sqrt{2\rho G/a}$, so
$\liminf_{\rho\downarrow0}\mathcal F^{\pm}_f(\rho)/\sqrt\rho\ge(1-\epsilon)\sqrt{2G/a}$;
now let $\epsilon\downarrow0$. Numerical agreement for the Kullback--Leibler
generator is recorded below as a sanity check on the algebra, not as evidence for
general $f$. Its role is to show $G$ is not an artefact of
the $\chi^2$ ball.

\paragraph{Why $\chi^2$ gives an exact interval.} Pearson $\chi^2$ has a
quadratic generator, so the stationarity condition $f'(r)=a+bZ$ yields an
\emph{affine} extremal reweighting and hence the exact finite-radius square-root
law of Corollary~\ref{cor:local}. General smooth $f$-divergences share the same
local modulus only asymptotically.

We deliberately do \emph{not} claim that $\chi^2$ is the unique such divergence.
The natural one-line argument --- $r=(f')^{-1}(a+bZ)$ is affine iff $f'$ is
affine --- constrains $f'$ only on the range of $r$ that the optimiser actually
visits. Choosing $f$ to agree with the quadratic generator on that range and to
be steeper outside it gives a non-quadratic generator whose ball contains the
$\chi^2$ optimiser and is contained in the $\chi^2$ ball, hence reproduces the
exact law on the same interval; we verified this numerically to $10^{-13}$. A
genuine uniqueness statement would have to quantify over all non-degenerate laws
of $Z$, work modulo affine additions and positive rescalings of $f$ (which leave
the divergence unchanged under $\mathbb E[r]=1$), and establish a value-level
converse rather than an optimiser-level one. We do not need such a statement and
do not make one.

\section{Readout hierarchy (Proposition~\ref{prop:hierarchy})}
\label{app:hierarchy}
Let $\sigma(R_1)\subseteq\sigma(R_2)\subseteq\sigma(X)$. By the tower property
$\mathbb E[\eta\mid B,R_1]=\mathbb E\big[\mathbb E[\eta\mid B,R_2]\;\big|\;B,R_1\big]$,
so conditional Jensen gives
$\operatorname{Var}(\mathbb E[\eta\mid B,R_1]\mid B)\le
\operatorname{Var}(\mathbb E[\eta\mid B,R_2]\mid B)\le\operatorname{Var}(\eta\mid B)$,
which is $G_{R_1}\le G_{R_2}\le G$.

For the profiles, let $\mathcal A_R(\rho)$ be the admissible reweightings that
are $\sigma(B,R)$-measurable. Then $\mathcal
A_{R_1}(\rho)\subseteq\mathcal A_{R_2}(\rho)\subseteq\mathcal A_X(\rho)$, and for
$r\in\mathcal A_R$ the tower property gives $\mathbb E[(r-1)\eta\mid B]=\mathbb
E[(r-1)\mathbb E[\eta\mid B,R]\mid B]$, so the restricted and unrestricted
objectives agree on $\mathcal A_R$. A supremum over a larger set is at least as
large, giving $\mathcal F_{B,R_1}\le\mathcal F_{B,R_2}\le\mathcal F_B$. \qed

\paragraph{No monotonicity for realised drift.} Under refinement the
restricted profile increases, but the realised composition drift
$\delta_{\mathrm{comp}}(R)=\operatorname{Cov}(\mathbb E[w\mid B,R],\mathbb
E[\eta\mid B,R]\mid B)$ need not, because covariance is not monotone under
conditioning. Consequently $|\delta_{\mathrm{cond}}(R_2)|\le
|\delta_{\mathrm{cond}}(R_1)|$ is false in general, and we do not use it.

\paragraph{Covariate-shift limit.} Under pure covariate shift with common
support, $\theta_Q\to\theta_P\to\eta$ as $R\uparrow\sigma(X)$, so
$\delta_{\mathrm{cond}}(R)\to0$. At any finite $R$, however, a nonzero
$\delta_{\mathrm{cond}}$ conflates mechanism change with partition coarseness,
and the two are not separable from the data.

\subsection{Readout approximation (Proposition~\ref{prop:approx})}\label{app:approx}

Fix $b$ with $P(B=b)>0$ and write $\mathbb E_b$ for expectation under
$P(\cdot\mid B=b)$. The bins and readout are fitted on $D_{\mathrm{fit}}$ and are
fixed maps given that sample; all statements concern the population given those
maps. Let $m_R=\mathbb E_b[\eta\mid R]$, defined on cells of positive probability
(other cells are null sets), so that $m_R=\theta_P(b,R)$ because $R$ is a function of
$X$. Let $\mu_b=\mathbb E_b\eta$ and $H_R=\mathbb E_b(\eta-m_R)^2=\mathbb
E_b[\operatorname{Var}(\eta\mid R)]$; the law of total variance gives $H_R=G_P-G_R$.
Write $\mathcal A_b(\rho)$ for the $X$-measurable $r\ge0$ with $\mathbb E_b r=1$ and
$\mathbb E_b(r-1)^2\le\rho$, and $\mathcal A_{b,R}(\rho)$ for its members that are
functions of $R$ on $\{B=b\}$; all are square integrable.

\paragraph{Proof.} Take $r\in\mathcal A_b(\rho)$ and let $\bar r=\mathbb E_b[r\mid R]$.
Then $\bar r\ge0$, $\mathbb E_b\bar r=1$, and $\bar r$ is a function of $R$. With
$u=\mathbb E_b(\bar r-1)^2$ and $v=\mathbb E_b(r-\bar r)^2$, the cross term
$\mathbb E_b[(\bar r-1)(r-\bar r)]=\mathbb E_b[(\bar r-1)\,\mathbb E_b[r-\bar r\mid R]]$
vanishes, so $u+v=\mathbb E_b(r-1)^2\le\rho$ and $\bar r\in\mathcal A_{b,R}(u)$.
Conditioning on $R$ twice,
\[
\mathbb E_b[(r-1)(\eta-\mu_b)]
=\mathbb E_b[(\bar r-1)(m_R-\mu_b)]+\mathbb E_b[(r-\bar r)(\eta-m_R)],
\]
because $\mathbb E_b[(r-\bar r)(m_R-\mu_b)]=0$. The first term is at most
$\mathcal F^{+}_{B,R}(b,u)$, since $\bar r$ is feasible at budget $u$, and by
Cauchy--Schwarz the second is at most $\sqrt{vH_R}\le\sqrt{(\rho-u)H_R}$ in absolute
value. Taking the supremum over $r$ gives the middle bound for the upward
direction. For the downward direction the same split applies to $\mu_b-\eta$: the
first term is at most $\mathcal F^{-}_{B,R}(b,u)$ and the second is bounded in
absolute value as before. The restricted profile is non-decreasing in its budget,
which gives the right-hand bound. By Cauchy--Schwarz,
$\mathcal F^{\pm}_{B,R}(b,u)\le\sqrt{uG_R}$, so the middle bound is also at most
$\sup_u\{\sqrt{uG_R}+\sqrt{(\rho-u)H_R}\}=\sqrt{\rho G_P}$; on the variance-controlled
interval of the binned oracle Theorem~\ref{thm:profile} gives
$\mathcal F^{\pm}_B=\sqrt{\rho G_P}$, so the bound is attained there. For nested
readouts, apply the argument on $\sigma(B,R_2)$ with $m_{R_2}$ in place of $\eta$:
$\mathbb E_b[m_{R_2}\mid R_1]=m_{R_1}$ and $\mathbb E_b(m_{R_2}-m_{R_1})^2=G_{R_2}-G_{R_1}$.
Proposition~\ref{prop:hierarchy} gives the lower bounds. \qed

\paragraph{Degenerate cases and sharpness.} At $\rho=0$ every quantity is $0$. If
$H_R=0$, $\eta$ is a function of $R$ on the bin and the restricted and oracle
profiles coincide; if $G_R=0$ the restricted profile vanishes and the bound reduces
to $\sqrt{\rho G_P}$. The constant in the simple bound cannot be improved
uniformly: a readout constant on the bin has $H_R=G_P$ and gap $\sqrt{\rho G_P}$ on the
oracle's variance-controlled interval. Outside that interval the middle bound need not
be tight; it is an upper bound, not a general identity. Summed
over bins with weights $p_b$, Minkowski's inequality gives
$(\sum_bp_b\mathcal F_B^2)^{1/2}-(\sum_bp_b\mathcal F_{B,R}^2)^{1/2}\le
\sqrt{\rho\,\mathbb E[\operatorname{Var}(\eta\mid B,R)]}$.

\paragraph{What the bound does not provide.} $H_R$ depends on $\eta$ inside cells.
With one label per input it is not identified without assumptions on $\eta$; the only
assumption-free envelope, $H_R(b)\le\mathbb E_b[\theta_P(1-\theta_P)]$, is attained
by $\{0,1\}$-valued propensities and on ImageNet bins is far larger than $G_R$, so it
makes the bound uninformative there. Substituting an estimate of $H_R$ gives no
confidence statement; the bound does not address binning (the oracle here is the
binned one), gives no stability for the inverse map of Theorem~\ref{thm:ident}, does
not control natural total drift, and cannot be combined with
Proposition~\ref{prop:fscert} into a two-sided interval. The projection step is the
one used to define subpopulation risks through conditional risks
\citep{duchi2023distributionally,subbaswamy2021evaluating}; the residual term is the
one \citet{perezlebel2023grouping} note cannot be estimated. We checked the
inequalities, both directions and the degenerate cases numerically against an
independent dual solver on $4000$ random finite laws and partitions (code in the
supplementary material); no violation exceeded $10^{-11}$.

\section{Readout refinement in the frozen pilot}\label{app:ladder}

Proposition~\ref{prop:hierarchy} says that, for nested readouts, a finer one exposes more restricted fragility.
The frozen single-model pilot evaluated four readouts, drawn from three partition
families, all of the same cardinality ($K=8$) and of increasing intended
informativeness; the
plug-in quantities order accordingly. All numbers are means over the twenty
confidence bins and five folds, upward direction; the last column is
the within-radius Spearman correlation between predicted and realised drift
across bins.

\begin{center}\small
\begin{tabular}{llcccc}
\toprule
readout $R$ & what it sees & $G_R$ (deb.) & $\mathcal F^{+}_{B,R}(0.3)$
 & realised & Spearman \\
\midrule
R0 hash null  & sample id only        & $-0.0002$ & $0.0095$ & $0.0057$ & $0.27$--$0.59$ \\
R1 $K$-means  & penultimate features  & $\phantom{-}0.0009$ & $0.0190$ & $0.0153$ & $0.66$--$0.85$ \\
R2 global     & correctness readout   & $\phantom{-}0.0110$ & $0.0571$ & $0.0586$ & $0.995$--$0.998$ \\
R2 within-bin & readout, per-bin      & $\phantom{-}0.0129$ & $0.0622$ & $0.0625$ & $0.994$--$0.997$ \\
\bottomrule
\end{tabular}
\end{center}

These are original-protocol plug-in values and they vary across readouts, rising
by a factor of $6.5$ from the null readout to the richest one, with the realised
held-out drift the frozen weights induce tracking them. The null readout yields
essentially nothing, as intended: its debiased $G_R$ is slightly negative. R0 takes no correctness at all, so it is a null for whether the
machinery invents fragility from a label-free partition. It is not a null for the
selection effect of Appendix~\ref{app:estimators}, which needs the readout's own
labels to be permuted; that control is reported there.

\paragraph{What the table does not show.}
These four readouts are of increasing intended informativeness but they are
\emph{not nested}
$\sigma$-algebras: a hash partition, a $K$-means partition of the
penultimate space and the quantile groups of a fitted readout do not refine one
another, and the global and within-bin variants of R2 do not refine each other
either. Proposition~\ref{prop:hierarchy} therefore does not order these four, and
we do not present the table as an instance of it. The genuinely nested statement
available here is the trivial one --- a readout constant within each confidence
bin is contained in all of them and exposes exactly zero. For each admissible readout,
Proposition~\ref{prop:hierarchy} guarantees that its population restricted
heterogeneity and fragility lower-bound the corresponding binned oracle
quantities. The qualifier matters. Under exact
conditioning on $S$, an additional readout measurable with respect to $S$ alone
adds no within-level heterogeneity, since $\sigma(R)\subseteq\sigma(S)$ gives
$\mathbb E[\eta\mid S,R]=\mathbb E[\eta\mid S]$ and hence
$\operatorname{Var}(\mathbb E[\eta\mid S,R]\mid S)=0$. The finite construction conditions
on the bin $B$, not on $S$: a score-derived readout that separates confidence
values inside a bin can expose a nonzero quantity, namely within-bin variation of
$\mu_P$, and we make no zero claim for such a readout. The tabulated numbers are
cross-fitted plug-in estimates of those targets, not confidence bounds; the separate
lower confidence bound of Appendix~\ref{app:fscert} addresses that question. What
the table does show is that these plug-in values depend strongly on the
representation one is willing to look at.

\subsection{ECE and FAUC across checkpoint--state combinations: an exploratory description}

Figure~\ref{fig:ecefauc} plots the twenty-four post-audit checkpoint--state combinations of the primary cohort. Marginally, source
ECE and FAUC are associated (Pearson \PearsonAll{}, Spearman \SpearmanAll{});
conditioning on $G_R$ removes it, the partial correlation between ECE and FAUC
being \PartialEceAll{} while $G_R$ retains \PartialGAll{} given ECE (Pearson on
OLS residuals). On the twelve rows of the two primary states the same statistics
are \PartialEcePrim{} and \PartialGPrim{}.

None of this establishes that ECE fails to determine fragility, and a larger
$|\text{partial}|$ is not a stronger result. If $x$ varies over a non-degenerate
interval and $y=x^2$, $z=y+x/2$, then $y$ is an exact function of $x$ while the
partial Pearson correlation of $x$ and $y$ given $z$ equals $-1$: a strong
negative residual association is a statement about curvature in a linear residual
model, not about determination. The determination question is settled by
Proposition~\ref{prop:witness}, which exhibits laws agreeing in $e_P$ and $G_P$
whose profiles differ. The unit is one (checkpoint, calibration state) combination, the control variable
is as named, every residualisation includes an intercept, and the rank versions
apply the same construction to rank-transformed columns with average ties.
Deleting one \emph{checkpoint} together with all of its states --- ResNet-50 V1
and V2 are two checkpoints of one architecture family, so this is not a family-level
deletion --- leaves $20$ of the $24$ rows, or $10$ of the $12$, and gives
\begin{center}\footnotesize\setlength{\tabcolsep}{4pt}
\begin{tabular}{lccc}
\toprule
parent set & $n$ left & range of $\mathrm{p}(\mathrm{ECE},\mathrm{FAUC}\mid G_R)$
 & range of $\mathrm{p}(G_R,\mathrm{FAUC}\mid\mathrm{ECE})$ \\
\midrule
all four states ($24$) & $20$ & $-0.125$ to $+0.140$ & $+0.840$ to $+0.896$ \\
primary two states ($12$) & $10$ & $-0.846$ to $-0.289$ & $+0.889$ to $+0.971$ \\
\bottomrule
\end{tabular}
\end{center}
\noindent
so the near-zero value on the full set is stable to checkpoint deletion while the
larger primary-set value is not, and the two ranges belong to different parent
sets and must not be read against one another. Six checkpoints observed in two to
four states each are neither independent nor a sample from a population. We report
the numbers because they were computed, and make no claim from them.

\subsection{A score-only comparator, and why it is not a floor}

Replacing $R$ by within-bin quantile groups of the confidence itself gives a
readout with $\sigma(R_S)\subseteq\sigma(S)$. Under exact conditioning on $S$ such
a readout exposes zero; at finite $J$ it exposes the within-bin variation of
$\mu_P$ that binning admits. It is a \emph{comparator} and not a bound in either
direction: $\sigma(B,R_S)$ and $\sigma(B,R)$ are not nested --- the learned
readout takes confidence as one of $53$ inputs but its $K$ quantile groups need
not retain the confidence ordering --- so Proposition~\ref{prop:hierarchy} orders
neither against the other, and the difference of the two FAUCs estimates no
population quantity. In particular it is not the binning-induced part of the
restricted profile, and nothing in this paper separates within-score heterogeneity
from bin width.

The comparison is also strongly estimator-dependent, which is the main reason to
distrust any arithmetic built on it. The two protocols use the same bin and group counts, $J=20$, $K=8$ and five
folds, but differ in the rate-estimation sample and in the data used to fit the
readout. Their numerical difference is therefore not an isolated estimate of
selection bias.

\begin{center}\footnotesize\setlength{\tabcolsep}{4.5pt}
\begin{tabular}{lcccccc}
\toprule
& \multicolumn{3}{c}{inner-split rates} & \multicolumn{3}{c}{same-fold rates} \\
\cmidrule(lr){2-4}\cmidrule(lr){5-7}
Model & $G_R$ & FAUC & comparator & $G_R$ & FAUC & comparator \\
\midrule
ResNet-50 V1    & 0.00039 & 0.0423 & 0.0398 & 0.00050 & 0.0361 & --- \\
ResNet-50 V2    & 0.01336 & 0.1353 & 0.0414 & 0.01285 & 0.1319 & --- \\
ConvNeXt-T      & 0.00826 & 0.1121 & 0.0400 & 0.00807 & 0.1097 & --- \\
ViT-B/16        & 0.00109 & 0.0478 & 0.0409 & 0.00126 & 0.0447 & --- \\
Swin-T          & 0.00143 & 0.0565 & 0.0380 & 0.00168 & 0.0545 & --- \\
EffNet-B0       & 0.00318 & 0.0784 & 0.0416 & 0.00342 & 0.0779 & --- \\
\bottomrule
\end{tabular}
\end{center}

\noindent
The split-rate $G_R$ column carries an MC standard error of $1.2$--$7.3\times
10^{-5}$ over eight independent inner splits; the same-fold column is
deterministic given the folds and reproduces the \emph{original-protocol}
model summary exactly --- not the post-audit Table~\ref{tab:models}, which uses
the revised data roles. The
two do not differ with a constant sign, because the split also halves the sample
fitting $\hat\theta$ and refits the readout on half the fold: the difference mixes
the selection effect of Appendix~\ref{app:estimators} with a change of estimator
and of partition. The comparator ratio moves correspondingly, from $1.06$--$3.27$
under split rates to $1.39$--$4.21$ under same-fold rates. We therefore report the
comparator as context for what a score-only readout already exposes at $J=20$,
and draw no quantitative conclusion from the difference.

\section{Selective-risk transport (Proposition~\ref{prop:selective})}
\label{app:selective}

Let $t$ be a threshold with $P(S\ge t)>0$ and $\mathcal R_P(t)=\Pr_P(C=0\mid S\ge t)$.

\begin{proof}
Coverage: $Q(S\ge t)=P(S\ge t)$ is immediate from $Q_S=P_S$. For the risk, tower
over $\sigma(S)$ and use $Q_S=P_S$ in both numerator and denominator:
\[
  \mathcal R_Q(t)=1-\frac{\mathbb E_{P_S}\big[\mu_Q(S)\mathbf 1\{S\ge t\}\big]}{P(S\ge t)},
  \qquad
  \mathcal R_P(t)=1-\frac{\mathbb E_{P_S}\big[\mu_P(S)\mathbf 1\{S\ge t\}\big]}{P(S\ge t)},
\]
and subtracting gives the stated identity with $\delta=\mu_Q-\mu_P$. Since the
ambiguity set constrains the conditional budget at $P_S$-a.e.\ level, the
pointwise bounds $-\mathcal F^-_s(\rho)\le\delta(s)\le\mathcal F^+_s(\rho)$ hold
a.e., and integrating them over $\{S\ge t\}$ against $P_S$ gives
\[
  -\frac{\mathbb E_{P_S}[\mathcal F^+_S(\rho)\mathbf 1\{S\ge t\}]}{P(S\ge t)}
  \;\le\;\mathcal R_Q(t)-\mathcal R_P(t)\;\le\;
  \frac{\mathbb E_{P_S}[\mathcal F^-_S(\rho)\mathbf 1\{S\ge t\}]}{P(S\ge t)} .
\]
\end{proof}

\paragraph{Threshold, coverage and the reported percentage.} Bins are ordered by
ascending confidence and a threshold sits on a bin boundary, which is what makes
coverage exactly preservable under $Q_B=P_B$. For a target coverage the bin index
is $b_t=\arg\min_b|\sum_{b'\ge b}p_{b'}-\text{target}|$, chosen on the
\emph{source} bin masses with no interpolation; the achieved coverage is recorded
alongside. Then $\mathcal R_P(t)=1-\sum_{b\ge b_t}p_b\mu_P(b)/\sum_{b\ge b_t}p_b$ and
$\mathcal R_Q$ likewise with $\mu_Q$, and the reported quantity is the relative increase
$\big(\mathcal R_Q(t)-\mathcal R_P(t)\big)/\mathcal R_P(t)$, evaluated at the same $b_t$ and left undefined
if $\mathcal R_P(t)\le0$. The quoted increases use the \emph{downward} optimiser, the
direction that raises accepted-set risk. Bin arrays are already pooled over
folds, so no further fold aggregation enters.

The identity was verified to $5.05\times10^{-16}$ over $576$ (model, state,
$\rho$, direction, coverage) combinations. Empirically $39/576$ cells exceed the
envelope by a median $2.2\times10^{-4}$, i.e.\ $0.7\%$ of the bound; this is the
known $\hat\theta$ estimation error, since the empirical envelope is built from
estimated cell rates, and not a failure of the population statement.

\paragraph{Provenance: descriptive, not preregistered.} No dedicated protocol
governed this analysis. It re-reads the frozen per-model outputs over the full
grid of $2$ calibration states $\times$ $6$ positive radii $\times$ $2$ directions
$\times$ $4$ coverage levels ($50$, $70$, $80$ and $90\%$). Any single value from
this grid is a descriptive selection, not a prespecified endpoint, and the main
text quotes none of them. An earlier draft highlighted one such value; to make that
selection visible: in the temperature-scaled state, the four model pairs whose
ECEs agree to within $10^{-3}$ give $96$ (pair, radius, coverage) combinations in
the downward direction, over which the ratio of relative risk increases has median
$1.30$ and reaches $2$ in only $6$. The highlighted Swin-T/ConvNeXt-Tiny ratio,
$2.29$ at $\rho=0.3$ and $90\%$ coverage, lies in that upper tail (maximum $2.31$);
the same pair gives $1.91$, $1.68$ and $1.13$ at $80$, $70$ and $50\%$. All quoted
percentages are realised held-out increases under the frozen adverse weights, not
the envelope, which is larger: for ConvNeXt-Tiny at $90\%$ coverage the bound on
the risk difference is $0.0191$ against a realised $0.0152$.

\section{Composition--conditional decomposition
(Proposition~\ref{prop:decomp})}\label{app:decomp}

\begin{proof}
Fix a bin $b$ and write the sums over the $K$ groups. Since
$\sum_kq_k=\sum_kp_k=1$,
\[
  \sum_k(q_k-p_k)\big[\theta_P(b,k)-\mu_P(b)\big]
  =\sum_kq_k\theta_P(b,k)-\mu_P(b),
\]
because the $\mu_P(b)$ terms cancel and $\sum_kp_k\theta_P(b,k)=\mu_P(b)$ by
definition. Adding
$\delta_{\mathrm{cond}}(b)=\sum_kq_k\theta_Q(b,k)-\sum_kq_k\theta_P(b,k)$ gives
$\sum_kq_k\theta_Q(b,k)-\mu_P(b)=\mu_Q(b)-\mu_P(b)=\delta_{\mathrm{total}}(b)$.
\end{proof}

Both sums run over all $K$ groups with no retention filter, so the identity is
exact by construction; it closes numerically to $5.1\times10^{-16}$ over
$92{,}942$ bin records. Only $\delta_{\mathrm{comp}}$ has the form
$\mathbb E_p[(w-1)(\theta_P-\mu_P)]$ with $w=q_k/p_k$ satisfying
$\mathbb E_p[w]=1$, $w\ge0$ and $\mathbb E_p[(w-1)^2]=D_R(b)$; it is therefore a
feasible point of the programme defining $\mathcal F^{\pm}_{B,R}\!\big(b,D_R(b)\big)$ and is
bounded by it. Under the stated support conditions, and up to numerical error, the
envelope check of Section~\ref{sec:natural} therefore holds, and utilisation rather
than containment is the informative quantity.

\paragraph{From bins to Table~\ref{tab:natural}.} The reported summaries are
built in four levels, and conflating them is the easiest way to misread the
table. \emph{(i)} Bin level, per (model, source fold, domain, bin): the exact
$\delta_x(b)$ and, where the envelope is defined, $U(b)$
of~Section~\ref{sec:natural}. \emph{(ii)} Fold pooling: the five source-fold replicates
are averaged, their spread retained as the source-fold variation diagnostic.
\emph{(iii)} Domain level, per (model, domain): $L_x=(\sum_b q_b\delta_x(b)^2)^{1/2}$
with $q_b$ the target bin mass renormalised over the domain's bins and
non-finite entries excluded; the domain's utilisation is the \emph{median} of
$U(b)$ over its bins; and its composition-to-total norm ratio is
$L_{\mathrm{comp}}/L_{\mathrm{total}}$.
\emph{(iv)} Benchmark level: $L_{\mathrm{total}},L_{\mathrm{comp}},
L_{\mathrm{cond}}$ are \emph{means} over the (model, domain) pairs, the reported
median $U$ is the \emph{median} of the per-domain medians, and the
composition-to-total norm ratio quoted in Section~\ref{sec:natural} is the
\emph{mean of the per-domain ratios}. That last convention is why that ratio and
the tabulated norms are not in agreement: for ImageNet-C the mean of ratios is $41.8\%$ while the ratio of
means is $32.8\%$. $U$ is undefined, and excluded, wherever the envelope is
undefined --- a support violation ($q_k>0$ with $p_k=0$), or a vanishing
denominator; no value is clipped, and none exceeds $1$, as it cannot when
numerator and envelope use the same plug-in rates. In the fold-pooled records of the
three primary benchmarks neither exclusion occurs: there is no support violation
and no vanishing denominator among the $9{,}177$ records. Bins with no target
example do not enter at all ($63$ of the $9{,}000$ ImageNet-C (model, domain, bin)
combinations). Under the preregistered rule, bins with fewer than $200$ target
examples in a fold are flagged rather than dropped. The pooled table retains the
low-count flag from the first available fold, whereas its reported target count is
averaged over folds; $5{,}508$ of the $8{,}937$ ImageNet-C
records are flagged, none of the ImageNetV2 or ImageNet-Sketch records, so many
ImageNet-C bin-level values rest on small target samples.

Each $L_x$ is the $q$-weighted $L^2$ norm $\|\delta_x\|_Q$, so the binwise
identity of Proposition~\ref{prop:decomp} gives
$L_{\mathrm{total}}=\|\delta_{\mathrm{comp}}+\delta_{\mathrm{cond}}\|_Q$ and hence
only the triangle inequality $L_{\mathrm{total}}\le L_{\mathrm{comp}}+
L_{\mathrm{cond}}$ at the domain level, with equality only when one component is a
non-negative multiple of the other. The exact additivity therefore holds bin by
bin and not among the tabulated norms.

\paragraph{Low-target-count sensitivity (post-hoc).} After the count above was
reported, a diagnostic was specified in writing before it was computed: within each
ImageNet-C (model, domain), drop the flagged bins, renormalise the target bin mass
over the remaining ones and recompute the three norms. This diagnostic uses that
frozen flag, not a threshold applied to the averaged count, and no other threshold. Table~\ref{tab:r8-lowcount} gives the result. The
flagged bins carry little mass: \LcRetainedMass{} of the target mass is kept on
average (\LcRetainedMin{}--\LcRetainedMax{} over the $450$ units). The mean ratio
moves from \LcRatioAll{} to \LcRatioRet{}, the conditional norm stays the larger in
\LcCondRet{} of $450$ units (\LcCondAll{} with all bins), and the severity trend
remains (\LcSevOneRet{} at severity~1, \LcSevFiveRet{} at severity~5). The
recomputed all-bin norms match the frozen summaries to $3\times10^{-16}$. Restricting
to higher-count bins changes the estimand; the check does not remove target-rate
noise from the retained bins, does not address the source-rate selection of the
original partition, does not replace a revised-pipeline rerun of ImageNet-C and
carries no interval or coverage statement. The original all-bin results remain the
reported analysis.

\begin{table}[h]
\caption{Low-target-count sensitivity of the ImageNet-C summaries (post-hoc; original partition, fold-pooled records). ``Unflagged bins'' drops, within each (model, domain), the bins carrying the stored low-count flag of the pooled table (set by the preregistered per-fold rule of fewer than $200$ target examples and taken from the first available fold; not a threshold on the fold-averaged count) and renormalises the target bin mass over the rest; every one of the $450$ (model, domain) units keeps at least one bin. This changes the estimand to higher-count bins; it is not a revised-pipeline rerun and carries no interval or coverage statement.}
\label{tab:r8-lowcount}
\vspace{3pt}\centering\footnotesize\setlength{\tabcolsep}{4pt}
\begin{tabular}{lccccc}
\toprule
bins used & target mass kept & mean $L_{\mathrm{comp}}/L_{\mathrm{total}}$ & $L_{\mathrm{cond}}>L_{\mathrm{comp}}$ & severity 1 & severity 5 \\
\midrule
all & $100\%$ & $41.8\%$ & 400/450 & $50.6\%$ & $33.0\%$ \\
unflagged & $77.9\%$ ($55.8\%$--$98.6\%$) & $43.9\%$ & 401/450 & $52.7\%$ & $37.0\%$ \\
\bottomrule
\end{tabular}
\end{table}

\paragraph{The matched-ECE pair (original protocol; not reproduced).} Under the
original protocol, temperature-scaled ResNet-50 V2 and Swin-T had ECE $0.0308$ and
$0.0307$ with FAUC $0.046$ and $0.031$, a factor of $1.45$, and this was reported
as the sharpest static form of the dissociation. It does not survive the estimator
revision: in the post-audit numbers of Table~\ref{tab:models} that pair has ECE
$0.0324$ and $0.0318$, a gap of $5.3\times10^{-4}$, so it fails the preregistered
$5\times10^{-4}$ matching tolerance, and no other primary pair meets it in either
state. We record the historical value here and make no current empirical claim
from it. A post-hoc enumeration of all pairs among the eighteen checkpoints of
Appendix~\ref{app:breadth}, at the same tolerance, is in Table~\ref{tab:r7-matched};
the dissociation itself rests on Proposition~\ref{prop:witness}.

\begin{figure}[h]
\centering
\includegraphics[width=0.62\textwidth]{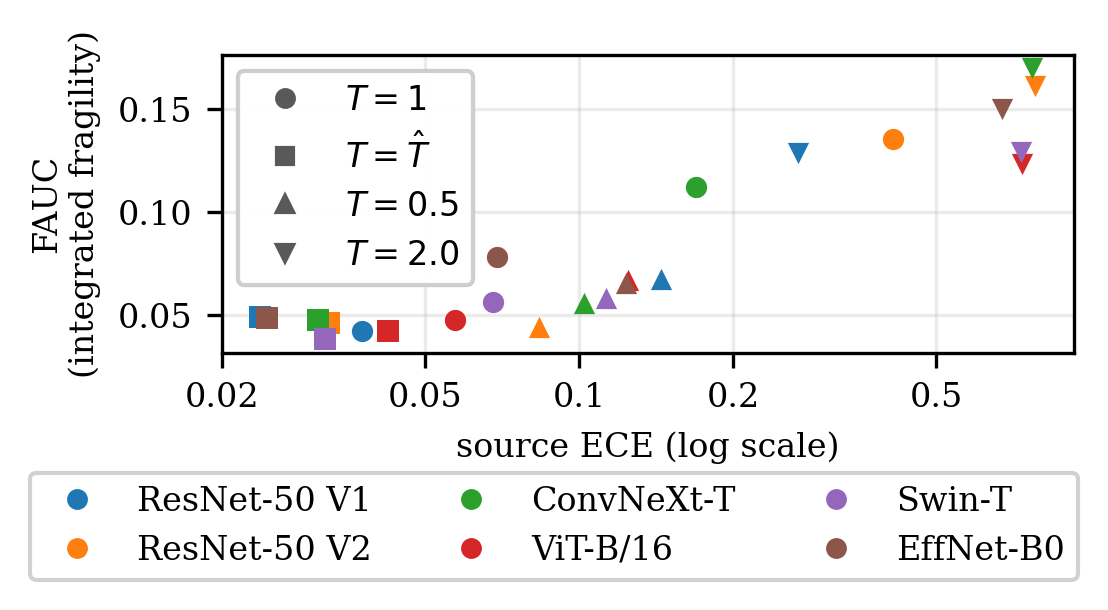}
\caption{\textbf{Source calibration error against restricted fragility,
exploratory.} One point per checkpoint and calibration state, all post-audit and
regenerated from the same source as Table~\ref{tab:models}; the six checkpoints
appear four times each and are not independent observations. Marginally the two
are associated (Pearson \PearsonAll{}, Spearman \SpearmanAll{}); conditioning on
$G_R$ removes it (partial \PartialEceAll{}, while $G_R$ retains \PartialGAll{}
given ECE). Appendix~\ref{app:ladder} gives both parent sets, the rank versions
and the leave-one-checkpoint-out ranges, and why none of this establishes
non-determination.}
\label{fig:ecefauc}
\end{figure}

\paragraph{Two further targets, exploratory.} The protocol also declared
ImageNet-R \citep{hendrycks2021many} and ImageNet-A \citep{hendrycks2021natural}
as \emph{exploratory} and never pooled with the primary benchmarks: both cover a
different $200$-class vocabulary, so the clean source reference is restricted to
those classes before any source quantity is estimated, while the classifier
remains the unchanged $1000$-way model. The composition-to-total norm ratio is $9.7\%$ and $5.9\%$ on them
($L_{\mathrm{total}}=0.293$ and $0.485$), so the dominance of the conditional term is if anything stronger on the
harder shifts. ImageNet-A in particular was built by adversarial filtration
against a ResNet-50 and is not an ordinary covariate shift; we report it with
that caveat and draw no further claim from either.

\section{The witness laws of Proposition~\ref{prop:witness}}\label{app:witness}

All four witnesses are two-point laws at a single confidence level, verified
exactly against the audited enumeration solver.

\begin{center}\small
\begin{tabular}{llccccc}
\toprule
& law of $\eta$ at $S=s$ & $s$ & $\mu_P$ & $e_P=\mu_P-s$ & $G_P$ & $z_+$ \\
\midrule
(i)    & $\{0.2,1.0\}$ w.p.\ $\tfrac12$ & $0.6$ & $0.6$ & $0$ & $0.16$ & $0.4$ \\
(ii)   & $\equiv0.9$                    & $0.6$ & $0.9$ & $0.3$ & $0$ & --- \\
(iii)a & $\equiv0.6$                    & $0.5$ & $0.6$ & $0.1$ & $0$ & --- \\
(iii)b & $\{0.2,1.0\}$ w.p.\ $\tfrac12$ & $0.5$ & $0.6$ & $0.1$ & $0.16$ & $0.4$ \\
(iv)a  & $\{0.4,1.0\}$ w.p.\ $(\tfrac23,\tfrac13)$ & $0.6$ & $0.6$ & $0$ & $0.08$ & $0.4$ \\
(iv)b  & $\{\tfrac13,0.9\}$ w.p.\ $(\tfrac{9}{17},\tfrac{8}{17})$ & $0.6$ & $0.6$ & $0$ & $0.08$ & $0.3$ \\
\bottomrule
\end{tabular}
\end{center}

The pair (iv) is the finite-support witness for Theorem~\ref{thm:ident}: the two
laws agree in \emph{both} the calibration residual and the grouping variance
(exactly, $G_P=0.08$ for both), both profiles equal $\sqrt{0.08\,\rho}$ on
$[0,1.125]$, and they separate by exactly $0.100$ for $\rho\ge2$, where one
saturates at $0.4$ and the other at $0.3$. For (iv)b, $\mathbb E\eta=\tfrac13\cdot
\tfrac9{17}+0.9\cdot\tfrac8{17}=0.6$, $\mathbb E\eta^2-0.36=0.08$, $z_-=-\tfrac4{15}$ and
$\rho_+=\rho_{\mathrm{sat}}=\tfrac98$; for (iv)a, $\rho_+=\rho_{\mathrm{sat}}=2$. Note that witness (i) has $G_P=0.16$
while the largest attainable value at $\mu_P=0.6$ is $0.6\cdot0.4=0.24$; it is
positively fragile, not maximally fragile.

\section{Cross-fitting, leakage test and estimator conventions}
\label{app:crossfit}\label{app:estimators}\label{app:diagnostics}
Each of the six checkpoints reproduces its published top-1 accuracy to within
$0.02$ points. Five folds, seed $20270214$; $J=20$ equal-mass confidence bins;
$K=8$ groups.
\emph{Original protocol.} As first specified, each training portion fitted the bin
edges, the correctness readout (logistic regression on confidence, margin, entropy
and a $50$-dimensional PCA of the penultimate representation) and its within-bin
quantile boundaries, and cell rates $\theta_P$ were then taken from that same
training portion with Beta--Binomial shrinkage toward the bin mean ($m=10$).
Group proportions came from the held-out portion using $(B,R)$ only, and cells
were retained at training count $\ge30$ and evaluation count $\ge10$. This is the
pipeline behind the original-protocol results retained in
Appendix~\ref{app:ladder}, and it is \emph{not} the method used for
Table~\ref{tab:models}: the revision below separates the data that fit the cells
from the data that estimate their rates, and drops the shrinkage so that
Proposition~\ref{prop:gunbiased} applies.

\paragraph{Leakage test.} Permuting the held-out fold's correctness labels and
rerunning the entire construction leaves every partition assignment and every
weight bit-identical (verified by SHA-256 over the weight arrays). The
permutation must be confined to the fold under test: permuting all folds at once
also scrambles each fold's training labels and produces a false positive.

\paragraph{FAUC, in full.} The reported FAUC is reproducible from the released
per-model outputs as follows. For each state, partition $R$ (primary
\texttt{R2\_within}), radius $\rho$ and direction, the per-bin plug-in profile
$\widehat{\mathcal F}^{\pm}_{B,R}(b,\rho)$ is solved on the plug-in cell rates
$\hat\theta$. Within each outer fold, bins are combined by
$w_b = n_b/\sum_{b'} n_{b'}$, the bin's share of retained evaluation examples in
that fold, into
\[
  F^{\pm}_f(\rho)\;=\;\Big(\textstyle\sum_b w_b\,
  \widehat{\mathcal F}^{\pm}_{B,R}(b,\rho)^2\Big)^{1/2},
\]
a quadratic mean over analysable bins, and the model-level curve is the unweighted
mean $F^{\pm}(\rho)=\frac15\sum_fF^{\pm}_f(\rho)$ over the five folds.
The radius grid is fixed at $\rho\in\{0,0.01,0.03,0.1,0.3,1,3\}$ --- seven points,
including $\rho=0$ where $F=0$ --- and is the same grid in every experiment that
reports FAUC. Writing $x_j=\log(1+\rho_j)$, so that $x$ runs over
$[0,\log4]$,
\[
  \mathrm{FAUC}\;=\;\sum_{j=0}^{5}\tfrac12\big[F^{+}(\rho_j)+F^{+}(\rho_{j+1})\big]
  \,(x_{j+1}-x_j),
\]
the trapezoidal rule in log-radius. Three conventions matter for exact
reproduction: the headline FAUC uses the \emph{upward} profile $F^{+}$ only (the
downward integral is recorded separately as a directional-asymmetry diagnostic);
the integral is \emph{not} normalised by $\log4\approx1.386$, so FAUC has the
units of $F$ scaled by that interval; and $F$ is built from $\hat\theta$, not
from the debiased $G_R$. Recomputing this from the frozen per-model profiles
reproduces every FAUC in Table~\ref{tab:models} exactly (maximum absolute
deviation $0$).

\paragraph{ECE, in full.} Every sample contributes exactly once, through the
fold in which it was an evaluation point, using that fold's frozen temperature
and frozen equal-mass bin edges; the resulting out-of-fold arrays are then pooled
over all $n=50{,}000$ validation images and
\[
  \mathrm{ECE}\;=\;\sum_{b=1}^{J}\frac{n_b}{n}\,
  \big|\bar c_b-\bar s_b\big|,\qquad J=20,
\]
with $\bar c_b$ the mean correctness and $\bar s_b$ the mean confidence in bin
$b$, empty bins skipped. The binning is \emph{equal-mass} (adaptive), on the same
$J=20$ preregistered bins as the fragility construction, with edges fitted on
each training fold; an equal-width $15$-bin variant is recorded as a secondary
metric and is not used in any comparison. Because the temperature is fitted on
the training fold and the ECE contribution is taken on the held-out fold,
temperature fitting and ECE evaluation never share data. The as-released and
temperature-scaled states use the identical estimator, differing only in the
confidences fed to it.

\paragraph{Data roles.} Every supervised object is fitted on data disjoint from
the data used to estimate cell rates. Within each outer fold $f$ (five folds,
seed $20270214$) the training portion is split once, $50/50$, into
\begin{center}\small
\begin{tabular}{lll}
\toprule
role & size & what is fitted or estimated on it \\
\midrule
$D_{\mathrm{fit}}$  & $20{,}000$ & temperature $T$; bin edges; correctness readout;
                                   within-bin group boundaries \\
$D_{\mathrm{rate}}$ & $20{,}000$ & raw cell rates $\tilde\theta_{bk}$ and counts
                                   $n_{bk}$, on the frozen cells \\
$D_{\mathrm{eval}}$ & $10{,}000$ & evaluation proportions $\hat p$; held-out
                                   correctness scoring frozen weights \\
\bottomrule
\end{tabular}
\end{center}
This is a change from the protocol as first specified, where the readout and the
temperature were both fitted on the whole training portion and $\hat\theta$ was
then formed from those same labels. Throughout, $\tilde\theta$ denotes raw binomial
cell proportions and $\hat\theta$ the cell rates a pipeline plugs into the profile:
the shrunk rates under the original protocol, and $\tilde\theta$ itself under the
revised roles. The temperature matters here even though
scaling preserves the argmax and therefore leaves correctness invariant: $T$ was
fitted by NLL on all $40{,}000$ training labels, so a rate-estimation point
contributed to the map that decides which bin it lands in, and no conditioning
argument recovers independence afterwards. The effect of one point on a scalar
fitted from $40{,}000$ is small, but the statement ``conditionally on the frozen
objects the rates are independent binomials'' is simply unavailable under that
design. We therefore refit $T$ on $D_{\mathrm{fit}}$ and report the cached-$T$
variant separately as a sensitivity variant.

\paragraph{What the plug-in estimates, in general.} With
$\hat\theta$ the vector of cell rates in bin $b$, weights $w$ and
$A_w=\operatorname{diag}(w)-ww^{\top}$, the plug-in is
$\hat\theta^{\top}A_w\hat\theta$. For fixed $w$, writing
$\beta=\mathbb E[\hat\theta]-\theta$ and $\Sigma=\operatorname{Cov}(\hat\theta)$,
\begin{equation}
  \mathbb E\big[\hat\theta^{\top}A_w\hat\theta\big]-\theta^{\top}A_w\theta
  \;=\;\underbrace{2\theta^{\top}A_w\beta+\beta^{\top}A_w\beta}_{\text{bias of }\hat\theta}
  \;+\;\underbrace{\operatorname{tr}(A_w\Sigma)}_{\text{estimation noise}} .
  \label{eq:gbias}
\end{equation}
Only when the cells are fixed and the $\hat\theta_{bk}$ are unbiased and mutually
independent does the second term reduce to $\sum_kw_k(1-w_k)v_k$ with
$v_k=\operatorname{Var}(\hat\theta_{bk})$; a shared shrinkage target correlates
cells, and a label-supervised readout fitted on the rate-estimation labels makes
$\beta\neq0$ in the direction that inflates the plug-in.

\begin{proposition}[Conditional unbiasedness]
\label{prop:gunbiased}
Fix the partition, non-negative weights $w$ with $\sum_kw_k=1$, and cell counts
$n_k\ge2$. Suppose the $\tilde\theta_k$ are independent binomial proportions with
mean $\theta_k$ on those counts. Then
\[
  \widehat G_w=\tilde\theta^{\top}A_w\tilde\theta
   -\sum_k w_k(1-w_k)\,\frac{\tilde\theta_k(1-\tilde\theta_k)}{n_k-1}
  \quad\text{satisfies}\quad
  \mathbb E\big[\widehat G_w\mid\text{the fixed objects}\big]=\theta^{\top}A_w\theta .
\]
\end{proposition}
\begin{proof}
$\mathbb E[\tilde\theta^{\top}A_w\tilde\theta]=\theta^{\top}A_w\theta+
\operatorname{tr}(A_w\Sigma)$ and $\Sigma=\operatorname{diag}(v)$ by
independence, so $\operatorname{tr}(A_w\Sigma)=\sum_k(w_k-w_k^2)v_k$ with
$v_k=\theta_k(1-\theta_k)/n_k$. For a binomial proportion,
$\mathbb E[\tilde\theta_k(1-\tilde\theta_k)]=\theta_k(1-\theta_k)(n_k-1)/n_k$, so
$\mathbb E[\tilde\theta_k(1-\tilde\theta_k)/(n_k-1)]=v_k$ exactly; $n_k\ge2$ is
what makes the divisor meaningful. Subtracting gives the claim.
\end{proof}

\paragraph{Which estimand, and the rules that define it.} The proposition is
conditional on the realised weights, so $\widehat G_w$ is unbiased for the
\emph{empirical-weight} quantity $\theta^{\top}A_{\hat p}\theta$, not for the
population $G_R$ of the same cells. In an idealised model with fixed full support, no retention and independent
multinomial counts $N\sim\mathrm{Mult}(m,p)$, $\hat p=N/m$, one has
$\mathbb E[A_{\hat p}]=(1-1/m)A_p$, so a rate-noise-only correction retains a
$(1-1/m)$ factor against the population target; with $m\approx500$ evaluation
points per bin that factor is $0.998$. This delimits one term in that idealised
model. It is not a bound on the total departure from a population target of the
actual pipeline, which also selects cells by count, renormalises over the
retained ones and refits the partition on each fold; we do not claim such a
bound, and the estimand we report is the empirical-weight one. Three further conventions are part of the estimand and
not incidental. Cells enter only when $n^{\mathrm{rate}}_{bk}\ge30$ and
$n^{\mathrm{eval}}_{bk}\ge10$, and $\hat p$ is renormalised over the retained
cells, so the target is the restricted quantity on that retained sub-support;
retention uses counts, not labels, but it is still data-dependent and the
proposition is conditional on its outcome. Bins with fewer than two retained
cells are dropped. Negative estimates are retained, never clipped. Per-bin values
are combined by evaluation count within a fold and then averaged unweighted over
the five folds. Finally, the revised partitions are fitted on $D_{\mathrm{fit}}$ and can differ
from those of the original protocol. Differences between original-protocol and
post-audit estimates therefore need not estimate the bias of either procedure.

\paragraph{How large the selection effect is, measured rather than argued.}
We re-ran the pipeline with correctness replaced by labels independent of the
input, so the population $G_R$ of any induced partition is exactly $0$ and
anything reported is bias. Two randomisations are used and they answer slightly
different questions. A \emph{permutation} of the observed correctness vector fixes
the total number of successes and so samples without replacement; it is the
natural selection diagnostic but not the independent-binomial structure
Proposition~\ref{prop:gunbiased} assumes. Drawing \emph{i.i.d.}\ Bernoulli
pseudo-labels at the model's accuracy supplies that structure exactly. Both agree.
All uncertainties are across replicate randomisations on the same cached
predictions, so they describe randomisation error conditional on that cache, not
dataset sampling error; overlapping folds are not independent replicates and are
not counted as such.

\begin{center}\footnotesize\setlength{\tabcolsep}{4pt}
\begin{tabular}{llccc}
\toprule
randomisation & estimator & ResNet-50 V1 & ViT-B/16 & Swin-T \\
\midrule
permutation ($8\times$) & same-fold rates & $6.1\pm1.9$ & $6.9\pm1.7$ & $4.3\pm1.4$ \\
                        & inner-split rates & $0.7\pm0.9$ & $-2.3\pm2.1$ & $-1.2\pm1.8$ \\
\midrule
i.i.d.\ Bernoulli ($10\times$) & same-fold rates & $6.7\pm1.5$ & $4.4\pm1.0$ & --- \\
                        & inner-split rates & $0.6\pm1.7$ & $-0.9\pm1.7$ & --- \\
\bottomrule
\end{tabular}
\end{center}

\noindent
Entries are mean $\widehat G_R$ under the null $\pm$ Monte-Carlo standard error,
in units of $10^{-5}$; the standard deviation across replicates is
$\sqrt{8}$ or $\sqrt{10}$ times larger. The same-fold estimator is positive at
$3.2$--$4.6$ standard errors where the truth is zero; the inner-split estimator is
within one. A matched simulation with adaptive grouping --- a label-supervised
logistic readout on $12$ features, $K=8$ within-bin quantile groups, $2000$ points
per bin, $400$ replicates, truth taken as the population $G_R$ of the \emph{same}
fitted partition --- reproduces the ordering in all three regimes an estimator has
to survive: bias $+1.15\pm0.04$, $+2.62\pm0.20$ and $+2.54\pm0.19$ ($\times
10^{-3}$) for the same-fold estimator under a null, under genuine heterogeneity
and under unequal cell counts, against $-0.02\pm0.05$, $+0.28\pm0.28$ and
$+0.34\pm0.28$ for the inner split. Unequal cell counts do not change the
conclusion, and the selection term is larger when real heterogeneity is present
than under the null, so it cannot be removed by subtracting a constant. None of
this \emph{proves} unbiasedness: simulation at this precision shows only that no
bias is detected at that precision. Exact unbiasedness is what
Proposition~\ref{prop:gunbiased} supplies, under the conditions it states.

\paragraph{What the old and new estimators do not measure against each other.}
Table~\ref{tab:models} reports two \emph{calibration states} under one estimator;
the estimator contrast is the one between that table and the original-protocol
values retained in Appendix~\ref{app:ladder}. On real labels the difference
between the two \emph{protocols} is not a bias estimate. The inner split also
halves the sample that fits $\hat\theta$ and refits every nuisance on half the
fold, so the two protocols estimate restricted quantities of cells that can
differ; empirically their difference does not even keep a constant sign across
checkpoints. We read neither as a correction of the other.

\paragraph{Held-out rank agreement: aggregation.} For one
(checkpoint, state, fold, radius) we take the Spearman correlation across the
retained bins between the predicted upward per-bin profile and the drift realised
on $D_{\mathrm{eval}}$ under the frozen weights. A radius is skipped if fewer than
three bins are retained or either array is constant; a bin needs at least two
retained cells; $\rho=0$ is excluded because both arrays vanish identically there.
Within a fold we take the median over the six positive radii, then the mean over
the five outer folds, giving one value per (checkpoint, state). The figures quoted
in Section~\ref{sec:cannotsee} are the median, minimum and maximum of those
twenty-four values, with none excluded. ``Fold'' always means an outer
cross-fitting fold; the two halves of Appendix~\ref{app:fscert} are called
\emph{arms} and are a separate construction.

\paragraph{Unbiased rates do not give an unbiased profile.} For fixed $p$ and
$\rho$, $\theta\mapsto\mathcal F^{+}(p,\theta)$ is a supremum of linear functions
of $\theta$, hence convex, so Jensen gives
$\mathbb E[\mathcal F^{+}(p,\hat\theta)]\ge\mathcal F^{+}(p,\theta)$ whenever
$\hat\theta$ is conditionally unbiased, and the same holds for a non-negatively
weighted aggregate of per-bin profiles. This is an inequality in expectation: a
single realisation need not exceed the truth, and it does not imply that an
inner-split FAUC computed on a \emph{different} partition exceeds the original
one. The inequality is strict in the cases that matter: with
two equal-mass cells $\mathcal F^{+}(\rho)=\tfrac12|\theta_1-\theta_2|
\min\{\sqrt\rho,1\}$, so at $\theta_1=\theta_2$ the true profile is zero while
independent rates from $n=200$ per cell give $\mathbb E[\widehat{\mathcal F}^{+}]
=0.010$ at $\rho=0.3$ and $0.018$ at $\rho=3$. Splitting removes a selection mechanism from $\widehat G_R$; it does not make
FAUC unbiased, and it does not make it conservative. Halving the rate sample
increases the variance of $\hat\theta$, which at a \emph{fixed} partition and set
of weights increases the expected plug-in profile; that comparison is not the same
as comparing the split and original pipelines, which also differ in partition, and
we draw no ordering between those. We therefore keep three objects apart throughout: the
population restricted profile, the plug-in estimate of it, and the independent
split-sample lower confidence bound of Appendix~\ref{app:fscert}, which alone
carries a coverage guarantee and whose per-(model, state, split arm) statement we
do not transfer to any point estimate.

Negative values are \textbf{retained, not clipped}. The reported model-level value is
\[
  \widehat G_R^{\mathrm{model}}=\frac1{5}\sum_{f=1}^{5}
  \frac{\sum_b n^{\mathrm{ev}}_{f,b}\,\widehat G_{R,f}(b)}
       {\sum_b n^{\mathrm{ev}}_{f,b}},
\]
an evaluation-count-weighted mean over analysable bins within each fold, then an
unweighted mean over the five folds. Recomputing this reproduces every $G_R$ in
Table~\ref{tab:models} exactly, and recomputing ECE end to end from the cached
predictions reproduces Table~\ref{tab:models} to $<3\times10^{-17}$.

\paragraph{Debiased versus plug-in heterogeneity.} $G_R$ is reported with the
correction of Proposition~\ref{prop:gunbiased},
$\sum_kw_k(1-w_k)\tilde\theta_k(1-\tilde\theta_k)/(n_k-1)$, whereas the fragility
profile is computed from the plug-in rates with no correction. The two therefore
use different estimators and are not numerically interchangeable; no statement in
this paper compares $\sqrt{\rho\,G_R}$ with a plug-in fragility quantity as though
they estimated the same thing. (The percentage agreements quoted in earlier
versions of this appendix referred to the original-protocol shrunk estimator and
have been removed rather than recomputed, since no current claim uses them.)

\paragraph{Score-marginal diagnostics.} The construction preserves the binned
marginal exactly, $\max_b|Q_B(b)-P_B(b)|<10^{-14}$ throughout (pilot
$5.7\times10^{-16}$, multi-architecture $8.9\times10^{-15}$). The continuous
marginal is not preserved. In the pilot, under the primary partition at
$\rho=1$, the weighted-versus-source Kolmogorov--Smirnov statistic is $0.0051$
and the Wasserstein-1 distance $0.0014$; at $\rho=0.1$ they are $0.0017$ and
$0.0005$. These are reported as diagnostics, never as preservation.

\paragraph{Effective sample size.} In the pilot the minimum effective sample
size across retained (bin, radius) pairs is $622$ and the median $2{,}099$; the
minimum across the six architectures is $618$, and no bin falls below $30$. No
reported quantity rests on a cell the reweighting has emptied.

\subsection{Split-sample lower confidence bound}\label{app:fscert}

Proposition~\ref{prop:hierarchy} orders \emph{population} quantities. The plug-in
profile estimates the restricted profile but carries no finite-sample
guarantee, so a separate preregistered analysis supplies one. Its protocol was
written and hashed before any of its results were computed
(Appendix~\ref{app:claims}).

\paragraph{The pairwise construction.} Fix a bin $b$ and the learned readout $R$.
For an ordered pair of groups (source $u$, destination $v$), move mass $d$:
$q_v=p_v+d$, $q_u=p_u-d$, all other groups unchanged. Total bin mass is preserved;
non-negativity requires $d\le p_u$; the conditional Pearson $\chi^2$ cost is
$d^2(1/p_v+1/p_u)$; and since the perturbation has zero total mass, the centring
in $Z$ is immaterial and the reliability movement is exactly $d(\theta_v-\theta_u)$.
The budget $d^2(1/p_v+1/p_u)\le\rho$ and $d\le p_u$ together give the largest
feasible step
\[
  d^\ast(p_v,p_u,\rho)\;=\;\min\Big\{p_u,\ \sqrt{\rho\,p_vp_u/(p_v+p_u)}\Big\}.
\]
Because this move is a feasible point of the programme whose supremum defines
$\mathcal F^{+}_{B,R}(b,\rho)$, we get $\mathcal F^{+}_{B,R}(b,\rho)\ge
(\theta_v-\theta_u)_+\,d^\ast$ for every pair, with $d^\ast=0$ when $p_vp_u=0$. When
$p_u$ attains the minimum the
budget is automatically satisfied, since $\sqrt{\rho p_vp_u/(p_v+p_u)}\ge p_u$
rearranges to $\rho\ge p_u^2/p_v+p_u$, which is the cost of $d=p_u$.

\paragraph{Monotonicity of $d^\ast$.} $(p_v,p_u)\mapsto p_vp_u/(p_v+p_u)=(1/p_v+1/p_u)^{-1}$
is increasing in each argument, so the square-root term is increasing in both, and
$\min\{p_u,\cdot\}$ is a minimum of two functions each non-decreasing in $p_u$ and
non-decreasing in $p_v$. Hence $d^\ast$ is coordinatewise non-decreasing, and
substituting lower bounds for $p_v,p_u$ can only decrease it.

\paragraph{The simultaneous event and its family.} On the certification half,
conditionally on the realised counts, the correct counts in cell $(b,k)$ are
Binomial$(n_{bk},\theta_P(b,k))$, the cell counts within bin $b$ are
Binomial$(n_b,p_{bk})$, and the bin counts are Binomial$(n,p_b)$. We take exact
one-sided Clopper--Pearson bounds and combine by Bonferroni over
\[
  \underbrace{2JK}_{\theta\ \text{lower and upper}}+\underbrace{JK}_{p_{bk}\ \text{lower}}
  +\underbrace{J}_{p_b\ \text{lower}}\;=\;500 \quad (J=20,\ K=8)
\]
one-sided statements, each at $\alpha/500=10^{-4}$, giving an event $E$ with
$\Pr(E)\ge0.95$. On $E$, $(\theta_v-\theta_u)_+\ge[L_{bv}-U_{bu}]_+$ and, by the
monotonicity above, $d^\ast(p_v,p_u,\rho)\ge d^\ast(l_{bv},l_{bu},\rho)$; both
factors are non-negative, so their product is a lower bound, and
$\mathrm{LCB}^{+}_b(\rho)$ of~\eqref{eq:lcb} follows.

\paragraph{Why the maximisation and the radius grid are free.} Every pair yields a
bound valid on the \emph{same} event $E$, so their maximum is valid on $E$; the
selection is not a multiple test and needs no correction. Identically, the bounds
$L,U,l$ do not depend on $\rho$, so re-evaluating~\eqref{eq:lcb} across the grid
re-uses one event. The downward bound reverses the correctness comparison,
$[L_{bu}-U_{bv}]_+$, and likewise re-uses $E$.

\paragraph{Model level.} On $E$ we have $p_b\ge l_b\ge0$ together with
$\mathcal F^{+}_{B,R}(b,\rho)\ge\mathrm{LCB}^{+}_b(\rho)\ge0$, so termwise
$p_b\,\mathcal F^{+}_{B,R}(b,\rho)^2\ge l_b\,\mathrm{LCB}^{+}_b(\rho)^2$. Summing and taking the monotone
square root gives
\[
  \Big(\sum_b l_b\,\mathrm{LCB}^{+}_b(\rho)^2\Big)^{1/2}
  \;\le\;
  \Big(\sum_b p_b\,\mathcal F^{+}_{B,R}(b,\rho)^2\Big)^{1/2}.
\]
The $l_b$ lie inside the same family, so the complete aggregation is covered by
the single event.

\paragraph{Validity is conditional on the fitted nuisances, and also unconditional.}
The bin edges, the readout and, where applicable, the temperature are learned on the
fitting half only. Conditionally on those frozen objects, the certification half is
an i.i.d.\ sample from the induced population and $\Pr(E\mid\text{fitted objects})
\ge0.95$. Because the two halves are sample-separated, iterating the expectation over
the fitting half gives $\Pr(E)\ge0.95$ unconditionally as well.

\paragraph{Split protocol.} One balanced random partition of the $50{,}000$
validation indices into halves (split seed $20260902$, declared in the protocol). Two
arms are reported \emph{separately}: A fits and B certifies, then B fits and A
certifies. The two arms learn different $(B,R)$ and therefore certify different
population targets; they are never pooled, averaged, or selected between, and the
$95\%$ statement is per $(\text{model, state, arm})$ analysis --- not joint over the
six architectures, the two states, or the two arms.

\paragraph{Two deliberate differences from the plug-in pipeline.} First, the
temperature used here is fitted on the fitting half; it is a separate split-fitted
validation state mirroring the fitted state of Table~\ref{tab:models}, not the same
five-fold nuisance fit. Second, \textbf{no cell-retention filter is applied}: all
$K=8$ groups enter, and a small or empty cell simply receives a vacuous
Clopper--Pearson interval and $l_{bk}=0$. Count-based retention would be
data-dependent selection that the coverage proof does not cover. Consequently the
bound is a lower confidence bound for the restricted profile of the learned
finite readout and is \emph{not} an error bar on the retained-cell FAUC
estimator.

\paragraph{Validation.} The constructed reweighting was checked over $2\times10^5$
random supports: cost formula to $2.2\times10^{-15}$, objective exact, non-negativity
$\min q=0$, mass preserved to $4.4\times10^{-16}$, budget never exceeded. Synthetic
coverage against the true $\mathcal F^{\pm}$ from the audited solver gave $0$
violations out of $9{,}000$ in each of four regimes (wide cells, tight cells, small
$n$, skewed $p$) and $0$ out of $2{,}000$ for the model-level aggregation. Edge cases
pass: zero counts, $p\to0$, equal $\theta$, $\theta\in\{0,1\}$, $\rho=0$, and
saturation $d=p_u$.

\paragraph{Result.} The preregistered outcome categories were A (positive model-level
bound at $\rho=0.3$ for $\ge4$ of $6$ architectures in at least one primary state),
B (one to three), C (none). \textbf{Outcome A} was observed.

\begin{center}\small
\begin{tabular}{lcccc}
\toprule
& \multicolumn{2}{c}{as released $T{=}1$} & \multicolumn{2}{c}{split-fitted $T{=}\hat T$} \\
\cmidrule(lr){2-3}\cmidrule(lr){4-5}
Model & A fits/B cert. & B fits/A cert. & A fits/B cert. & B fits/A cert. \\
\midrule
ResNet-50 V1    & $0$      & $0$       & $0$ & $0$ \\
ResNet-50 V2    & $0.0158$ & $0.0171$  & $0$ & $0$ \\
ConvNeXt-T      & $0.0096$ & $0.0114$  & $2\times10^{-5}$ & $0$ \\
ViT-B/16        & $0$      & $0$       & $0$ & $0$ \\
Swin-T          & $0.0011$ & $0.00014$ & $0$ & $0$ \\
EfficientNet-B0 & $0.0018$ & $0.0014$  & $0$ & $0$ \\
\bottomrule
\end{tabular}
\end{center}

\noindent
$\mathrm{LCB}^{+}_{\mathrm{model}}(0.3)$; each entry carries its own $95\%$
statement. The two architectures at zero are the two with least
partition-accessible heterogeneity, matching the $G_R$ ordering of
Table~\ref{tab:models}. Where the bound is positive it recovers roughly $10$--$25\%$
of the corresponding unfiltered certification-half plug-in value at small $\rho$;
the strongest single bin (ResNet-50 V2, bin~4) gives $\mathrm{LCB}^{+}_b(0.3)=0.0414$ against a plug-in $0.1044$.

\paragraph{Two distinct reasons the numbers are flat in $\rho$.} The \emph{set} of
bins with $\mathrm{LCB}^{+}_b(\rho)>0$ does not change with $\rho>0$, because
positivity requires only a strictly positive lower-bounded correctness gap
$[L_{bv}-U_{bu}]_+>0$ together with $l_{bv},l_{bu}>0$, and neither depends on $\rho$.
Separately, the \emph{values} stop growing once $d^\ast$ saturates at $p_u$, which is
the mass constraint binding rather than the budget. The two effects are unrelated.

\section{Extended scope, assumptions and claim boundaries}\label{app:claims}

\paragraph{What the estimator revision does and does not reach.} The revision of
Appendix~\ref{app:estimators} changes where cell rates are estimated. Results
divide into four groups and we do not treat them alike.
\emph{(i) Population statements and algebraic identities} ---
Propositions~\ref{prop:transport}, \ref{prop:invisible}, \ref{prop:selective},
\ref{prop:witness}, \ref{prop:hierarchy} and \ref{prop:decomp},
Theorems~\ref{thm:profile} and \ref{thm:ident}, Corollary~\ref{cor:local} --- are
about the population and are untouched by any estimator.
\emph{(ii) Identities between arrays.} The monotone-reparameterisation analysis is
exactly zero because both sides are computed from literally the same arrays once
the bin index and readout are bit-identical; that remains true under any rate
estimator, and it was never an independent empirical finding.
\emph{(iii) The split-sample lower confidence bound} of Appendix~\ref{app:fscert}
draws its counts from a certification half disjoint from the half fitting its
nuisances, and its code path does not call the revised rate estimator. It was not
rerun; the reported values are those originally computed.
\emph{(iv) Everything numerical that reads source cell rates} --- levels, model
rankings, ratios, correlations, envelope utilisation and the
composition/conditional split --- had to be recomputed or withdrawn, and the
recomputation is reported with the analysis it belongs to. We do not claim that a
common estimator gives a common bias across models: the same procedure applied to
six architectures can be inflated by different amounts, which is precisely why
orderings had to be re-derived rather than assumed stable.

A closing identity is not evidence for its parts. If source rates carry error $e$,
then at fixed $p,q$ and readout
\[
  \widehat{\delta}_{\mathrm{comp}}-\delta_{\mathrm{comp}}=(q-p)^{\top}e,\qquad
  \widehat{\delta}_{\mathrm{cond}}-\delta_{\mathrm{cond}}=-q^{\top}e,\qquad
  \widehat{\delta}_{\mathrm{total}}-\delta_{\mathrm{total}}=-p^{\top}e,
\]
so $\widehat\delta_{\mathrm{comp}}+\widehat\delta_{\mathrm{cond}}=
\widehat\delta_{\mathrm{total}}$ still holds exactly while the components move.
With $p=(\tfrac12,\tfrac12)$, $q=(1,0)$, $\theta_P=(\tfrac12,\tfrac12)$,
$\theta_Q=(\tfrac34,\tfrac34)$ the truth is
$(\delta_{\mathrm{total}},\delta_{\mathrm{comp}},\delta_{\mathrm{cond}})=
(\tfrac14,0,\tfrac14)$; substituting $\hat\theta_P=(0.7,0.3)$ gives
$(\tfrac14,\tfrac15,\tfrac1{20})$ --- same total, same exact closure, reversed
dominance. The numerical closure of $5\times10^{-16}$ reported in
Appendix~\ref{app:decomp} therefore certifies the algebra and nothing about the
components, which is why Appendix~\ref{app:phase7} re-runs the decomposition under
the revised rate estimator instead of inferring stability from closure.
The population construction preserves the continuous score marginal exactly;
the finite-sample construction preserves only the binned marginal. The
population targets of our empirical heterogeneity and fragility quantities are
partition-restricted population lower bounds on the corresponding binned oracle
quantities, never the oracle $G_P(s)$ itself; that ordering is a population
statement, and the cross-fitted plug-in estimate of it carries no finite-sample
conservativeness guarantee. In the finite construction $G_P(b)=\operatorname{Var}(\eta\mid
B=b)$ and $\mathcal F_B$ are the objects of Section~\ref{sec:theory} for the binned
statistic $B$; they include within-bin variation of $\mu_P$ and are not lower
bounds on the continuous-level $G_P(s)$ or $\mathcal F_s$. ECE and the natural-drift
norms are ordinary out-of-fold estimates, not confidence bounds. The controlled shifts are constructed
adversarially and are not forecasts of natural shift; the composition component
of natural drift is bounded by the population source profile under the support
condition of Proposition~\ref{prop:decomp}, total drift is not. At finite
readout the conditional term conflates mechanism change with coarseness.
Temperature scaling of a multiclass top-label score is not a monotone
reparameterisation and can reorder examples, so it alters resolution as well as
calibration; we therefore do not treat it as a calibration-only intervention,
and states with $T\in\{0.5,2\}$ are stress tests rather than evidence about
recalibration. The fragility profile determines the centred conditional law of
the correctness propensity $\eta$ and not its mean; the law of the label $C$
given $S=s$ is Bernoulli and is fixed by the mean alone. The hierarchy of
Section~\ref{sec:theory} is informative only where $\eta$ is non-degenerate within
a level set. If the label were a deterministic function of $X$ and the classifier
deterministic, then $\eta\in\{0,1\}$, $Z$ would be two-valued with
$z_+=1-\mu_P(s)$ and $z_-=-\mu_P(s)$, and one would have exactly
\[
  G_P(s)=\mu_P(s)\big(1-\mu_P(s)\big),\qquad
  \mathcal F^{\pm}_s(\rho)=\min\!\big\{\sqrt{\rho\,G_P(s)},\,\pi^{\pm}\big\},
\]
with $\pi^{+}=1-\mu_P(s)$ and $\pi^{-}=\mu_P(s)$: the general statements
specialise to functions of $\mu_P(s)$ within that family. This is a degenerate
sub-family, not a failure of the dynamic question, and the converse does not
follow --- stochastic labelling alone does not guarantee non-trivial within-level
heterogeneity. We make no claim about how deterministic any particular labelling
mechanism is. Our empirical objects sit at a coarsened readout where
$\theta_P(b,k)$ averages over many inputs and need not be $\{0,1\}$-valued, even
when labels are deterministic. No claim is made that
any architecture is universally more robust.

\paragraph{Chronology of what was fixed when.} Six protocol artefacts were written and
cryptographically hashed before the results they govern: the pilot,
multi-architecture, natural-shift, natural-target disambiguation,
split-sample lower confidence bound (Appendix~\ref{app:fscert}), and monotone
score-reparameterisation (Appendix~\ref{app:monotone}) protocols. Each fixes its estimand, estimator, controls and
admissible outcomes in advance; the fifth also fixes its split seed before any
split was drawn, and the sixth fixes its transformation family and
parameterisation, its fitting objective, its cross-fitting scheme, the radius
grid, the numerical tolerances, the endpoint list and its three admissible
outcomes before any new real-data result was inspected. The model checkpoints are public pretrained weights and were
frozen before any analysis. Two post-audit analyses are not preregistered and are labelled here so that
no number in this paper depends on an unlabelled one. The first is the score-only
comparator of Appendix~\ref{app:ladder}, which reruns the frozen construction
with $R$ replaced by within-bin confidence quantiles; it introduces no new model,
dataset or training run, and it is a comparator with no ordering relation to the
learned-readout profile. The second is the cell-rate estimator audit reported in
Appendix~\ref{app:estimators}: it permutes the training-fold correctness labels
before the readout fit and the rate estimation, repeats the frozen pipeline over
eight permutations, and compares three cell-rate estimators. Four later extensions --- random score-preserving shifts, a CIFAR-10H case study,
a revised-partition recomputation of two natural-shift benchmarks
(Appendix~\ref{app:ext}) and a twelve-checkpoint architecture-breadth cohort
(Appendix~\ref{app:breadth}) --- each had a protocol hashed before its results
were computed; none is preregistered, and the quantities defined after inspection
are marked post-hoc where they appear. The diagnostics added in the final revision
(Appendix~\ref{app:r7}) are post-hoc throughout; their protocol and its first
amendment were hashed before the corresponding results were computed, the second
amendment (the controlled-shift null) after the random-shift null had been seen.
The low-target-count sensitivity of Appendix~\ref{app:decomp} was specified in
writing after the flagged-bin count had been reported and before it was computed;
it is post-hoc and leaves the original all-bin analysis as the reported one. An earlier internal note described the affected results as differing only in
level, with internal comparisons unaffected; that statement is withdrawn. What was
and was not recomputed is set out below, analysis by analysis.

\begin{center}\footnotesize\setlength{\tabcolsep}{4pt}
\begin{tabular}{p{0.34\textwidth}p{0.58\textwidth}}
\toprule
analysis & provenance \\
\midrule
Table~\ref{tab:models}; controlled-shift profiles, FAUC and held-out rank
agreement & \textbf{recomputed} under the revised data roles; orderings, ratios
and correlations re-derived rather than assumed stable \\
Figure~\ref{fig:ecefauc} and the associations of
Appendix~\ref{app:ladder} & \textbf{recomputed} from the same source as
Table~\ref{tab:models} \\
Split-sample lower confidence bound (Appendix~\ref{app:fscert}) & \textbf{original
analysis, not rerun}; its counts come from a certification half disjoint from the
half fitting its nuisances and its code path does not call the revised estimator \\
Monotone reparameterisation (Appendix~\ref{app:monotone}) & \textbf{original
protocol, not rerun}, retained for its structural and calibration endpoints; its
baseline is the original-protocol as-released state \\
Natural-shift summaries (Table~\ref{tab:natural}, Figure~\ref{fig:severity},
Appendix~\ref{app:decomp}) & \textbf{original protocol, fixed partition, not
rerun}; Appendix~\ref{app:phase7} adds an independent-rate sensitivity check of
the qualitative dominance only, and Appendix~\ref{app:ext} recomputes ImageNetV2
and ImageNet-Sketch under the revised roles \\
Pilot readout ladder, selective-risk grid, refitted-readout diagnostic &
\textbf{original protocol, not rerun}; reported descriptively \\
Random score-preserving shifts and CIFAR-10H case study
(Appendix~\ref{app:ext}) & \textbf{post-audit extensions, not preregistered};
protocols hashed before results; the rank summary and the same-proxy comparison
are post-hoc \\
Architecture-breadth cohort (Appendix~\ref{app:breadth}) & \textbf{post-audit
extension, not preregistered}; checkpoint list and protocol hashed before any
metric was computed; the six-checkpoint cohort is unchanged \\
Post-hoc diagnostics of the final revision (Appendix~\ref{app:r7}) & \textbf{post-hoc};
recomputed from the frozen partitions and results; no model inference; no
preregistered outcome re-scored \\
\bottomrule
\end{tabular}
\end{center}

\noindent
A historical label records where a number came from; it does not mean the number
was recomputed, and a structural invariance that holds under any rate estimator
does not give a historical statistic a finite-sample guarantee. Not everything
else reported here was preregistered either: the
synthetic theorem audit preceded the protocols, and the readout-refinement
comparison of Appendix~\ref{app:ladder}, the witness verification of
Appendix~\ref{app:witness}, the corruption-block dependence sensitivity of
Appendix~\ref{app:phase7}, the numerical confirmations in
Appendix~\ref{app:ident}, the selective-risk grid of Appendix~\ref{app:selective}
(which had no dedicated protocol and is reported descriptively) and the
partial-correlation statistic (Pearson correlation of OLS residuals, with intercept,
over all 24 raw, unranked (model, state) rows, without clustering adjustment)
reported in Appendix~\ref{app:ladder} are
post-hoc analyses of already-frozen outputs. They
are labelled as such, are reported descriptively, and are not used as primary
evidence for any claim in the main text.

\section{Preregistered natural-target disambiguation}\label{app:phase7}

\paragraph{Preregistration.} The protocol was written and hashed before any
result was computed, and fixes the endpoint, the statistic, the precondition and
the three admissible outcomes in advance. No result was inspected until the
analysis was run.

\paragraph{Why the analysis exists.} As released, the three candidate source
metrics rank the six checkpoints identically
(pairwise Spearman $+1.000$ for ECE--$G_R$, ECE--FAUC and $G_R$--FAUC), so any
comparison among them is uninformative. This is a design-power limitation, not
evidence that the metrics are equivalent. A post-audit breadth extension in
Appendix~\ref{app:breadth} shows that these identical rankings do not persist over
the eighteen-checkpoint cohort.

\paragraph{Source states.} Two states are used: \emph{as released} ($T{=}1$), and
the \emph{fitted temperature-scaled} state whose per-model, per-fold temperature is
read unchanged from the controlled study and never refitted. Quoting the
fold-averaged value for each model, these are $1.13$, $0.62$, $0.76$, $0.86$,
$0.84$ and $0.84$ for ResNet-50 V1/V2, ConvNeXt-T, ViT-B/16, Swin-T and
EfficientNet-B0; the five per-fold temperatures of a model differ by at most
$0.007$, and the analysis uses the per-fold values, not these averages. Temperature is never fitted on
any target data. Under fitted scaling the identical ranking is broken:
ECE--$G_R$ $=-0.257$, ECE--FAUC $=-0.486$, $G_R$--FAUC $=+0.771$. Because
temperature scaling also reorders a multiclass top-label score, this state breaks
the identical ranking without isolating calibration from resolution, and we use it
only as a matched-calibration comparison.

\paragraph{Unit of analysis and statistic.} For each target domain the six models
are ranked by each source metric and by the realised drift norm, and a Spearman
correlation is taken across models; the reported statistic is the mean of these
within-domain correlations, with a domain-level bootstrap interval ($2{,}000$
resamples). Six models alone give a single correlation almost no power, and
pooling model~$\times$~domain observations would ignore domain clustering;
averaging a within-domain statistic over the $75$ ImageNet-C domains respects the
clustering.

\paragraph{Primary endpoint.} $L_{\mathrm{comp}}$ only, because the composition
channel is the only part of natural drift that the source profile bounds.

\begin{center}
\begin{tabular}{lcccc}
\toprule
benchmark & domains & ECE & $G_R$ & FAUC \\
\midrule
ImageNet-C & 75 & $-0.093$ $[-0.18,+0.01]$ & $+0.595$ & $+0.634$ $[+0.54,+0.72]$ \\
ImageNet-Sketch & 1 & $+0.257$ & $+0.829$ & $+0.486$ \\
ImageNetV2 & 1 & $+0.257$ & $-0.029$ & $+0.200$ \\
\bottomrule
\end{tabular}
\end{center}

\paragraph{Dependence sensitivity (post-hoc, not preregistered).} The
preregistered interval resamples the $75$ ImageNet-C domains independently, but
$75=15$ corruptions $\times\,5$ severities and the five severities of one
corruption share a generative process and the same underlying base images, so
they are not independent sampling units and the preregistered interval is
plausibly too narrow. We therefore recompute the same statistic under a
\emph{corruption-block} bootstrap, resampling the $15$ corruption families with
replacement and taking all five severities of a drawn family together. The
preregistered analysis is unchanged; this is reported as a sensitivity check.

\begin{center}
\begin{tabular}{lccc}
\toprule
predictor & mean $\varrho$ & domain bootstrap (preregistered) & corruption-block (post-hoc) \\
\midrule
ECE   & $-0.093$ & $[-0.185,+0.006]$ & $[-0.240,+0.064]$ \\
$G_R$ & $+0.595$ & $[+0.510,+0.672]$ & $[+0.429,+0.717]$ \\
FAUC  & $+0.634$ & $[+0.542,+0.715]$ & $[+0.464,+0.765]$ \\
\bottomrule
\end{tabular}
\end{center}

Blocking widens every interval by a factor of $1.6$--$1.8$, which is the expected
direction. All three readings survive it: the FAUC
interval still excludes zero, the ECE interval still contains zero, and the two
still do not overlap. The result is also not carried by one corruption:
$14$ of the $15$ families have a positive mean within-domain correlation for
FAUC, ranging from $+0.98$ (JPEG compression) to $+0.38$ (contrast), with fog the
single negative family at $-0.20$. This addresses dependence among the five
severities of a corruption; it does \emph{not} address the dependence induced by
all fifteen families sharing the same underlying base images, so it is reported
as a sensitivity analysis and not as a corrected interval. Reproduced by
\texttt{posthoc\_checks/block\_bootstrap.py}.

\paragraph{Single-domain limitation.} ImageNet-Sketch and ImageNetV2 provide one
target domain each, so a domain-level bootstrap resamples a single point and the
bracketed quantities are not usable intervals. They are descriptive, they
disagree with each other, and no claim rests on them.

\paragraph{Null results on the ungoverned channels.} On ImageNet-C in the fitted
state, association with \emph{total} drift is $+0.128$ (ECE), $+0.022$ ($G_R$),
$+0.055$ (FAUC), and association with the \emph{conditional} term is likewise
near zero. Source fragility is associated with the channel the theory governs
and with neither channel outside it. Such an association would not violate the
decomposition --- which is an algebraic identity and constrains no correlation ---
but it would weaken the channel-specific reading we give here.

\begin{figure}[h]
\centering
\includegraphics[width=0.48\textwidth]{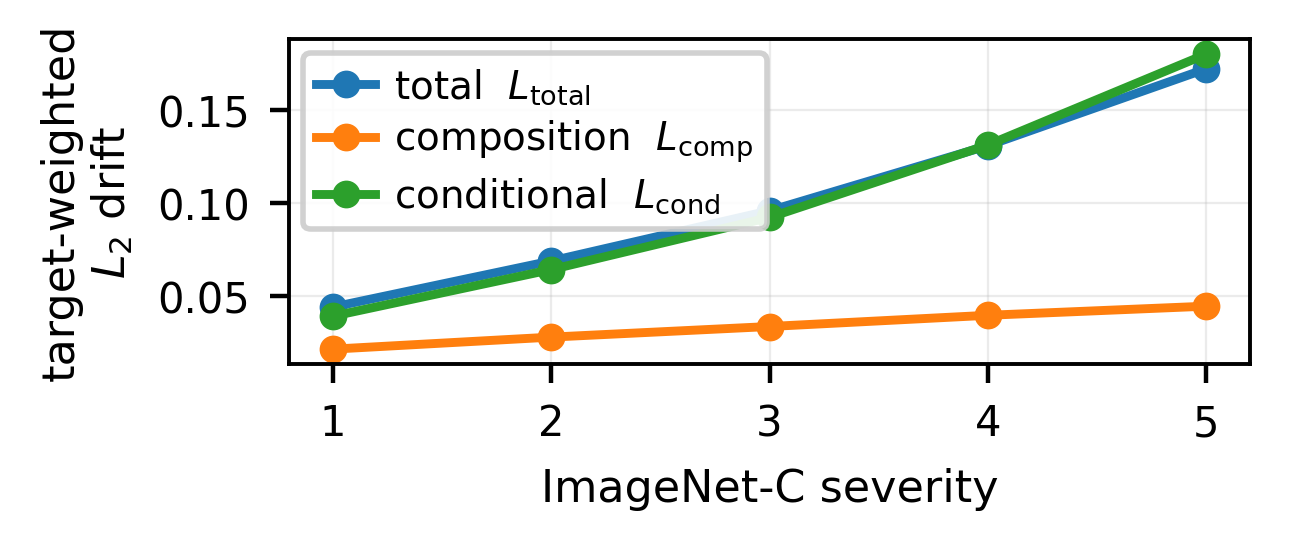}
\caption{\textbf{The composition-to-total norm ratio decreases with corruption
severity.} \emph{Original protocol, fixed partition; not recomputed.} ImageNet-C,
macro-averaged over fifteen corruptions and six architectures. The conditional
term tracks the total; the governed composition term grows far more slowly, and
$L_{\mathrm{comp}}/L_{\mathrm{total}}$ falls from $50.6\%$ at severity~1 to
$33.0\%$ at~5. Norms are not additive, so this is a ratio and not a share.}
\label{fig:severity}
\end{figure}

\begin{table}[h]
\caption{\emph{Original protocol, fixed partition; not recomputed under the
revised data roles.} Reliability drift under natural shift, averaged over six
checkpoints and domains, with $L$ and $U$ as defined in
Section~\ref{sec:natural}. The independent-rate check of
Appendix~\ref{app:phase7} tests the qualitative dominance only; it does not
regenerate these coefficients. The $L$ are
quadratic means of binwise drift, so $L_{\mathrm{total}}\neq L_{\mathrm{comp}}+
L_{\mathrm{cond}}$ despite the binwise identity being exact.}
\label{tab:natural}
\vspace{3pt}
\centering\small
\begin{tabular}{lcccc}
\toprule
Target & $L_{\mathrm{total}}$ & $L_{\mathrm{comp}}$ & $L_{\mathrm{cond}}$ & median $U$ \\
\midrule
ImageNet-C (75 domains) & 0.103 & 0.034 & 0.102 & 0.63 \\
ImageNetV2              & 0.058 & 0.006 & 0.055 & 0.37 \\
ImageNet-Sketch         & 0.240 & 0.024 & 0.223 & 0.89 \\
\bottomrule
\end{tabular}
\end{table}

\paragraph{Source rates independent of the frozen partition (post-audit).}
The decomposition reads source cell rates, so it belongs to group~(iv) of
Appendix~\ref{app:claims}. The natural-shift branch keeps the \emph{original}
frozen partition $(T_{\mathrm{old}},B_{\mathrm{old}},R_{\mathrm{old}})$, because
target penultimate features are not cached and the partition cannot be refitted
for target domains. That partition was fitted on the whole outer-training fold
$\{\text{fold}\neq f\}$, and its $\theta_{\mathrm{shrunk}}$ was estimated on the
same $40{,}000$ rows, so those rates are not independent of the cells they
describe. The source proportions $p_{bk}$ are unaffected: they use held-out
$(B,R)$ counts only, never labels.

Independence is recoverable here without new inference, because the outer
held-out fold never entered the fit. Reconstructing the index sets from the
recorded fold assignment and checking the reconstruction against the stored
cell counts (exact for all $30$ model--fold pairs), the held-out fold shares
$0$ rows with the fitting set, while a half of the training fold shares
$20{,}000$ of $20{,}000$. We therefore re-estimate the rates of the
\emph{same} frozen cells on the held-out fold and recompute the binwise
decomposition over $77$ cached target domains $\times$ five folds
($385$ domain--fold units per checkpoint, $380$ for ViT-B/16).

\begin{center}\footnotesize\setlength{\tabcolsep}{4.5pt}
\begin{tabular}{lcccccc}
\toprule
& \multicolumn{2}{c}{original (training-fold rates)}
& \multicolumn{2}{c}{held-out rates, $p_{bk}$ as frozen}
& \multicolumn{2}{c}{held-out split: rates $\perp$ $p$} \\
\cmidrule(lr){2-3}\cmidrule(lr){4-5}\cmidrule(lr){6-7}
Checkpoint & $L_{\mathrm{comp}}/L_{\mathrm{total}}$ & cond.\ larger
           & $L_{\mathrm{comp}}/L_{\mathrm{total}}$ & cond.\ larger
           & $L_{\mathrm{comp}}/L_{\mathrm{total}}$ & cond.\ larger \\
\midrule
ResNet-50 V1 & 0.210 & $100\%$ & 0.187 & $100\%$ & 0.190 & $100\%$ \\
ResNet-50 V2 & 0.469 & $67\%$  & 0.474 & $71\%$  & 0.438 & $81\%$  \\
ConvNeXt-T   & 0.927 & $70\%$  & 0.913 & $72\%$  & 0.956 & $76\%$  \\
ViT-B/16     & 0.222 & $100\%$ & 0.208 & $100\%$ & 0.232 & $100\%$ \\
Swin-T       & 0.343 & $99\%$  & 0.327 & $100\%$ & 0.347 & $99\%$  \\
EffNet-B0    & 0.299 & $100\%$ & 0.272 & $100\%$ & 0.280 & $100\%$ \\
\bottomrule
\end{tabular}
\end{center}

\noindent
The composition-to-total norm ratio is the mean over domain--fold units of
$L_{\mathrm{comp}}/L_{\mathrm{total}}$, each $L$ a target-mass-weighted quadratic
mean of binwise drifts; ``cond.\ larger'' is the fraction of domain--fold units
with $L_{\mathrm{cond}}>L_{\mathrm{comp}}$. Moving to independent rates changes the norm ratio
$L_{\mathrm{comp}}/L_{\mathrm{total}}$ by at most \NatShareShiftIndep{} in the
middle column and \NatShareShiftDisjoint{} in the right one, as an absolute change
in that ratio. The ratio is not a share of an additive decomposition and is not
confined to $[0,1]$: the components can cancel, so with
$\delta_{\mathrm{comp}}=0.2$ and $\delta_{\mathrm{cond}}=-0.1$ the total is $0.1$
and the ratio is $2$. Nor does the table say that no domain--fold unit reverses:
the fraction of units in which the conditional term is the larger moves between
columns (for example $67\%\to71\%\to81\%$ for ResNet-50 V2), and what is stable
is that for every checkpoint a majority of units keep the conditional term
larger. The binwise identity closes to
$\le6.1\times10^{-16}$ in every variant, which checks the algebra and not the
components: with source-rate error $e$ the components move by $(q-p)^{\top}e$ and
$-q^{\top}e$ while the identity still closes exactly
(Appendix~\ref{app:claims}). Domain--fold units are not independent
observations, and no interval is attached to these means. The exercise holds the
partition fixed, so it establishes insensitivity to the rate estimator, not to
the choice of partition or readout.

\paragraph{Structural coupling.} $\delta_{\mathrm{comp}}=\sum_k(q_k-p_k)
[\theta_P(b,k)-\mu_P(b)]$ is a functional of the same source cell-rate estimates
that define $G_R$ and the profile, so part of the association is structural. The
target-side factor $(q_k-p_k)$ is genuinely observed and is not produced by the
fragility optimiser, so the association is not definitional. The contrast we rely
on is that ECE --- which collapses the cell rates to bin-averaged correctness and
uses none of the within-bin contrasts --- has an interval that includes zero.

\paragraph{Preregistered result.} Under the prespecified decision rule, this
corresponds to Outcome~B (mixed): the direction is not stable across all three
benchmarks. ImageNet-C --- the only benchmark with repeated
target domains, and hence the only one supporting uncertainty quantification
under the preregistered analysis --- supports the preregistered direction. This
protocol fixed the natural-target analysis before its results were inspected, and
as preregistered that comparison is not reattempted; the later preregistered phases
of Appendices~\ref{app:fscert} and~\ref{app:monotone} address separate
finite-sample and monotone-reparameterisation questions.

\section{Monotone score reparameterisation}\label{app:monotone}

Section~\ref{sec:cannotsee} reports fragility across calibration states obtained
by multiclass temperature scaling. That family is not a reparameterisation of the
scalar top-label score: rescaling the logits can reorder examples, so it moves
resolution as well as calibration (Appendix~\ref{app:claims}). This appendix
records a preregistered intervention designed to remove exactly that ambiguity,
by acting \emph{only} on the already-computed scalar score through a strictly
increasing map.

\subsection{Invariance under a fixed strictly increasing bijection}

The following is elementary and is \textbf{not claimed as novel mathematics}. It
is stated because it converts the empirical predictions below from hypotheses
into consequences. In this lemma only, subscripts $S$ and $S'$ on $\mu$, $G$ and
$\mathcal F^{\pm}$ index the score representation; the source law $P$ is fixed.

\begin{lemma}[Score reparameterisation]\label{lem:reparam}
Let $S$ be a scalar score with range $I$, let $g:I\to I'$ be a strictly increasing
measurable bijection with measurable inverse, and put $S'=g(S)$. Then
\begin{enumerate}\itemsep2pt
\item $\sigma(S')=\sigma(S)$;
\item for $Q\ll P$, $\;Q_{S'}=P_{S'}$ if and only if $Q_S=P_S$, and
      $\mathbb E[r\mid S'=g(s)]=\mathbb E[r\mid S=s]$ for $r=dQ/dP$;
\item $\mu_{S'}(g(s))=\mu_S(s)$ and
      $\mathrm{Law}(\eta-\mu_{S'}(S')\mid S'{=}g(s))
       =\mathrm{Law}(\eta-\mu_S(S)\mid S{=}s)$;
\item $G_{S'}(g(s))=G_S(s)$ and, for every $\rho\ge0$,
      $\mathcal F^{\pm}_{S'}(g(s),\rho)=\mathcal F^{\pm}_S(s,\rho)$;
\item $\{S'\ge g(t)\}=\{S\ge t\}$, so coverage, the accepted set and the
      selective risk of any threshold policy are unchanged;
\item calibration is \emph{not} invariant:
      $\mu_{S'}(g(s))-g(s)=\mu_S(s)-g(s)$, which differs from $\mu_S(s)-s$
      whenever $g(s)\neq s$.
\end{enumerate}
\end{lemma}

\begin{proof}
(i) $S'=g(S)$ is $\sigma(S)$-measurable and $S=g^{-1}(S')$ is
$\sigma(S')$-measurable, so the two $\sigma$-fields coincide. (ii) For Borel
$A\subset I'$, $\{S'\in A\}=\{S\in g^{-1}(A)\}$ with $g^{-1}(A)$ Borel, so
$Q_{S'}(A)=Q_S(g^{-1}(A))$ and likewise for $P$; conversely $g(A')$ is Borel for Borel $A'\subset I$ because $g$ is a Borel isomorphism onto $I'$. Conditional
expectations depend on the conditioning $\sigma$-field alone, so by (i)
$\mathbb E[r\mid S']=\mathbb E[r\mid S]$ as random variables, and likewise
$\mathbb E[(r-1)^2\mid S']=\mathbb E[(r-1)^2\mid S]$: the score-preserving
constraint and the $\chi^2$ budget define the same feasible set of weights.
(iii) is (i) applied to $\eta$, with the centring constant equal by the first
identity. (iv) The feasible set is the same by (ii) and the objective
$\mathbb E[(r-1)(\eta-\mu)\mid S=s]$ is the same function of $r$ by (iii), so the
two programmes coincide; taking $\rho\to0$ gives the variance statement.
(v) is strict monotonicity. (vi) is immediate, and is what makes the intervention
useful: everything at the level of the $\sigma$-field is frozen while the score
\emph{value}, which is what calibration compares reliability against, is free.
\end{proof}

\subsection{Cross-fitted implementation, and why the lemma is applied fold by fold}

For each architecture and each fold $f$ we fit
\[
  g_f(s)=\mathrm{sigmoid}\big(a_f\,\mathrm{logit}(s)+c_f\big),
  \qquad a_f=e^{\alpha_f}>0,
\]
on the other four folds by the binary negative log-likelihood of correctness, and
apply the frozen $g_f$ to the held-out fold. Monotonicity holds by construction:
the optimiser is unconstrained in $(\alpha_f,c_f)$ and $a_f>0$ for every real
$\alpha_f$, so no monotonicity constraint can be active. No held-out correctness
is used to fit any $g_f$. This is calibration of the scalar
confidence-as-correctness probability; it does not touch the multiclass logits,
the predicted labels, the penultimate representations or correctness.

Because $g_f$ differs across folds, \textbf{Lemma~\ref{lem:reparam} is applied
within each held-out fold, conditional on that fold's fitted map}, and not to the
pooled out-of-fold score. The pooled score is a \emph{fold-wise} strictly
monotone reparameterisation, not a single global one --- the same convention the
main experiment already uses for its per-fold temperature. We therefore make no
$\sigma$-field claim about the pooled score without fold identity, and the
ordering statements below are within-fold statements.

\subsection{Primary structural analysis}

Within each fold the raw equal-mass bin boundaries are pushed through that fold's
map, $\mathrm{edge}'_j=g_f(\mathrm{edge}_j)$; the transformed-score quantiles are
\emph{not} independently refitted. The frozen $T{=}1$ \texttt{R2\_within} readout
is retained verbatim, since $R$ is an auxiliary observable readout and not part of
the scalar-score intervention. Table~\ref{tab:monoinv} reports the outcome.

\begin{table}[h]
\caption{Structural invariance under the fold-wise monotone map, all six
architectures. Counts are over all $250{,}000$ train and evaluation bin
assignments per model. Every entry is exactly zero; the preregistered tolerance
was $10^{-12}$.}
\label{tab:monoinv}
\vspace{3pt}
\centering\small
\begin{tabular}{lcccccccccc}
\toprule
& \multicolumn{4}{c}{changed assignments / violations}
& \multicolumn{6}{c}{max absolute difference} \\
\cmidrule(lr){2-5}\cmidrule(lr){6-11}
& bins & frozen $R$ & rank & ties & cells & $G_R$ & $\mathcal F^{\pm}$
& FAUC$^+$ & FAUC$^-$ & sel.\ risk \\
\midrule
ResNet-50 V1    & 0 & 0 & 0 & 0 & $0$ & $0$ & $0$ & $0$ & $0$ & $0$ \\
ResNet-50 V2    & 0 & 0 & 0 & 0 & $0$ & $0$ & $0$ & $0$ & $0$ & $0$ \\
ConvNeXt-T      & 0 & 0 & 0 & 0 & $0$ & $0$ & $0$ & $0$ & $0$ & $0$ \\
ViT-B/16        & 0 & 0 & 0 & 0 & $0$ & $0$ & $0$ & $0$ & $0$ & $0$ \\
Swin-T          & 0 & 0 & 0 & 0 & $0$ & $0$ & $0$ & $0$ & $0$ & $0$ \\
EffNet-B0       & 0 & 0 & 0 & 0 & $0$ & $0$ & $0$ & $0$ & $0$ & $0$ \\
\bottomrule
\end{tabular}
\end{table}

\paragraph{What this does and does not establish.} Once the bin index $B$ and the
readout $R$ are bit-identical, equality of the downstream quantities is
\emph{arithmetic}: the cell table, $G_R$, $\mathcal F^{\pm}$, both FAUCs and
selective-risk fragility are then computed from literally the same arrays. We do
not present those equalities as independent empirical discoveries, and the analysis
verifies the structural prediction only for the \emph{fixed} restricted readout it
uses; it is not evidence that an arbitrary fitted readout is invariant
(Appendix~\ref{app:monosec}). The empirical checks are: the learned $g_f$
materially improves held-out calibration; within every held-out fold it preserves
the score ordering exactly; pushing the bin boundaries through $g_f$ reproduces the
raw bin assignment bit-identically; and the frozen readout assignments are
bit-identical, the last being a bookkeeping check since $R$ is retained by
construction.

\subsection{Calibration endpoints}

\begin{table}[h]
\caption{Cross-fitted calibration under the fold-wise monotone map. ECE is the
$20$-bin equal-mass estimator of Appendix~\ref{app:estimators} with the bin
partition held fixed; ECE$_{15}$ is the equal-width secondary metric. All six
architectures exceed the preregistered $20\%$ relative-reduction bar.}
\label{tab:monoece}
\vspace{3pt}
\centering\footnotesize\setlength{\tabcolsep}{4.5pt}
\begin{tabular}{lccccccc}
\toprule
& \multicolumn{2}{c}{ECE} & & \multicolumn{2}{c}{binary NLL}
& \multicolumn{2}{c}{ECE$_{15}$} \\
\cmidrule(lr){2-3}\cmidrule(lr){5-6}\cmidrule(lr){7-8}
& $T{=}1$ & monotone & reduction & $T{=}1$ & monotone & $T{=}1$ & monotone \\
\midrule
ResNet-50 V1    & 0.03614 & 0.01288 & $64.4\%$ & 0.39579 & 0.37929 & 0.03688 & 0.01349 \\
ResNet-50 V2    & 0.41178 & 0.02162 & $94.8\%$ & 0.80867 & 0.38869 & 0.41178 & 0.02064 \\
ConvNeXt-T      & 0.16933 & 0.01477 & $91.3\%$ & 0.44668 & 0.35481 & 0.16933 & 0.01312 \\
ViT-B/16        & 0.05589 & 0.02577 & $53.9\%$ & 0.36461 & 0.34621 & 0.05605 & 0.02328 \\
Swin-T          & 0.06794 & 0.01989 & $70.7\%$ & 0.36261 & 0.34134 & 0.06794 & 0.01740 \\
EffNet-B0       & 0.06892 & 0.01408 & $79.6\%$ & 0.39254 & 0.37488 & 0.06847 & 0.01156 \\
\bottomrule
\end{tabular}
\end{table}

The median relative reduction is $75.1\%$ and the binary NLL improves for all six.
This experiment was run under the original protocol and has not been re-run under
the revised data roles. Its as-released baseline reproduces the \emph{original-protocol}
as-released state bit-identically (maximum absolute difference $0$ over ECE, $G_R$
and both FAUCs); that baseline is not the post-audit column of
Table~\ref{tab:models}, whose ResNet-50 V1 as-released ECE is $0.038$ against the
$0.03614$ of Table~\ref{tab:monoece}. The two baselines are different analyses and
we do not treat them as one. Nothing in the structural conclusion depends on which
baseline is used: the analysis shows that pushing a frozen partition through a strictly
increasing map leaves the downstream arrays identical, which is an arithmetic
consequence of Lemma~\ref{lem:reparam} and holds under either rate estimator.

\subsection{Fitted maps and numerical diagnostics}

\begin{table}[h]
\caption{Fitted parameters and rank diagnostics. Spearman and Kendall are
\emph{within-fold}, where a single $g_f$ acts; no optimisation reached a numerical
boundary in any of the $30$ fold fits. The edge gap is
$\max_j|g(\mathrm{edge}_j)-\mathrm{edge}^{\mathrm{refit}}_j|$, the distance between
the pushed-forward boundaries the primary analysis uses and independently refitted
transformed-score quantiles.}
\label{tab:monodiag}
\vspace{3pt}
\centering\footnotesize\setlength{\tabcolsep}{4.5pt}
\begin{tabular}{lccccccc}
\toprule
& $a_f$ range & $c_f$ range & Spearman & Kendall & \multicolumn{2}{c}{pooled ties}
& edge gap \\
\cmidrule(lr){6-7}
& & & & & before & after & \\
\midrule
ResNet-50 V1    & $0.710$--$0.719$ & $-0.128$--$-0.091$ & $1.000$ & $1.000$ & 29 & 4 & $2.2{\times}10^{-9}$ \\
ResNet-50 V2    & $1.430$--$1.457$ & $+2.554$--$+2.570$ & $1.000$ & $1.000$ & 29 & 4 & $2.2{\times}10^{-9}$ \\
ConvNeXt-T      & $1.342$--$1.352$ & $+1.026$--$+1.043$ & $1.000$ & $1.000$ & 29 & 4 & $2.3{\times}10^{-9}$ \\
ViT-B/16        & $1.343$--$1.360$ & $+0.124$--$+0.155$ & $1.000$ & $1.000$ & 29 & 4 & $0.7{\times}10^{-9}$ \\
Swin-T          & $1.251$--$1.264$ & $+0.351$--$+0.377$ & $1.000$ & $1.000$ & 29 & 4 & $1.9{\times}10^{-9}$ \\
EffNet-B0       & $1.081$--$1.091$ & $+0.476$--$+0.495$ & $1.000$ & $1.000$ & 29 & 4 & $1.8{\times}10^{-9}$ \\
\bottomrule
\end{tabular}
\end{table}

\paragraph{Pooled ties.} The pooled tie count falls from $29$ to $4$. This is
\emph{not} a tie break by a monotone map: samples sharing a raw score but lying in
different folds receive different $g_f$ and so separate in the pooled score.
Within every fold the tie breaks are zero. The arithmetic is exact --- for
ResNet-50 V1, $4$ of the $29$ duplicate-value groups lie entirely inside one fold
and $25$ span two or more, predicting precisely the $4$ pooled ties observed.

\paragraph{Pushed edges versus refitted quantiles.} The $\sim\!2\times10^{-9}$ gap
in Table~\ref{tab:monodiag} is expected and is reported as a distance only: linear
interpolation between order statistics does not commute with a nonlinear $g$, so
$g(\mathrm{quantile}(S))\neq\mathrm{quantile}(g(S))$. The primary analysis uses the
pushed-forward edges, as preregistered.

\subsection{Deterministic stress maps}

Two label-free maps fixed in advance, $g_{\mathrm{flat}}(s)=
\mathrm{sigmoid}(0.5\,\mathrm{logit}\,s)$ and $g_{\mathrm{sharp}}(s)=
\mathrm{sigmoid}(2\,\mathrm{logit}\,s)$, act on the scalar score alone and are not
temperature scaling. They are controls, not candidate transformations
(Table~\ref{tab:monostress}), and no outcome category depends on them.

\begin{table}[h]
\caption{Deterministic stress maps. Large deliberate changes in numerical
confidence leave every structural quantity exactly unchanged.}
\label{tab:monostress}
\vspace{3pt}
\centering\footnotesize\setlength{\tabcolsep}{4.5pt}
\begin{tabular}{lcccccccc}
\toprule
& \multicolumn{4}{c}{$g_{\mathrm{flat}}$} & \multicolumn{4}{c}{$g_{\mathrm{sharp}}$} \\
\cmidrule(lr){2-5}\cmidrule(lr){6-9}
& ECE & bins chg.\ & merged & $\max|\Delta\mathcal F|$
& ECE & bins chg.\ & merged & $\max|\Delta\mathcal F|$ \\
\midrule
ResNet-50 V1    & 0.04713 & 0 & 0 & $0$ & 0.09074 & 0 & 146 & $0$ \\
ResNet-50 V2    & 0.36709 & 0 & 0 & $0$ & 0.46333 & 0 & 0 & $0$ \\
ConvNeXt-T      & 0.23640 & 0 & 0 & $0$ & 0.10666 & 0 & 0 & $0$ \\
ViT-B/16        & 0.16995 & 0 & 0 & $0$ & 0.05472 & 0 & 0 & $0$ \\
Swin-T          & 0.16788 & 0 & 0 & $0$ & 0.05107 & 0 & 0 & $0$ \\
EffNet-B0       & 0.14557 & 0 & 0 & $0$ & 0.07057 & 0 & 0 & $0$ \\
\bottomrule
\end{tabular}
\end{table}

The map $g_{\mathrm{sharp}}$ merges $146$ previously distinct ResNet-50 V1 scores.
The map is strictly increasing as a function on $[0,1]$; its \emph{finite-precision
evaluation} saturates near one, and we therefore do not claim that every
deterministic stress-map implementation is one-to-one in floating point. The
merging occurs inside the top bin, so no bin assignment changes and no structural
quantity moves. No exact $S=0$ or $S=1$ occurs in any of the six models, and no
clipping rule is used anywhere.

\subsection{Secondary analysis: refitting the readout}\label{app:monosec}

\begin{table}[h]
\caption{Secondary end-to-end analysis, in which the bin quantiles \emph{and} the
low-capacity logistic readout are refitted on the transformed confidence
coordinate. Diagnostic only.}
\label{tab:monosec}
\vspace{3pt}
\centering\small
\begin{tabular}{lcccc}
\toprule
& ARI$(R)$ & exact $R$ agreement & relative $|\Delta G_R|$ & relative $|\Delta$FAUC$^+|$ \\
\midrule
ResNet-50 V1    & 0.136 & 0.379 & $108.7\%$ & $27.2\%$ \\
ResNet-50 V2    & 0.615 & 0.795 & $1.9\%$ & $0.5\%$ \\
ConvNeXt-T      & 0.584 & 0.775 & $9.9\%$ & $4.5\%$ \\
ViT-B/16        & 0.611 & 0.796 & $25.4\%$ & $11.3\%$ \\
Swin-T          & 0.575 & 0.770 & $27.5\%$ & $13.4\%$ \\
EffNet-B0       & 0.672 & 0.832 & $13.6\%$ & $7.2\%$ \\
\bottomrule
\end{tabular}
\end{table}

The fitted finite-sample readout is parameterisation-dependent
(Table~\ref{tab:monosec}): refitting its
logistic, standardised feature pipeline and its within-bin group thresholds after
changing the confidence coordinate can alter $R$, and with it the \emph{restricted
plug-in} profile, even though the population score $\sigma$-field is
unchanged and, by Lemma~\ref{lem:reparam}, so is the population fragility target
under a bijective reparameterisation. This is not a failure of
the lemma, and the frozen-$R$ primary result is correspondingly not evidence that
an arbitrary learned readout must itself be parameterisation invariant. The
largest effect is ResNet-50 V1, which also has the smallest $G_R$, so a small
absolute change is a large relative one. We report this analysis only as a diagnostic and
base no claim on it.

\subsection{Preregistration}

The protocol was written and cryptographically hashed before any result from this
phase was computed, and prespecified the transformation family and
parameterisation, fitting objective, cross-fitting scheme, radius grid, numerical
tolerances, evaluation endpoints, and outcome criteria. A suite of synthetic
validation tests was required to pass before any six-model summary was computed;
the full protocol and verification record are provided in the supplementary
material. The observed result satisfied the prespecified criterion for Outcome~A.

\section{Post-audit empirical extensions}\label{app:ext}

This appendix reports three analyses carried out after every result above was
frozen. Each has a written protocol that was cryptographically hashed before its
results were computed, but each was designed with the paper's results known, so
none is preregistered and none is used as a preregistered endpoint. Checkpoints,
seeds, radii, benchmarks, bin counts and tolerances were fixed in those protocols
and were not changed after results were seen. Two quantities below were defined
after the corresponding results had been inspected and are marked post-hoc where
they appear: the same-proxy comparison of Appendix~\ref{app:ext:human} and the
rank summary of Appendix~\ref{app:ext:random}. Every number in this appendix is
read from one result file released with the code, which also records the protocol
hashes; the three tables are generated from it.

\subsection{Random score-preserving shifts}\label{app:ext:random}

The controlled shifts of Section~\ref{sec:cannotsee} maximise the estimated
movement given $\hat\theta$, so their agreement with held-out drift is partly
built in. This analysis replaces the optimised weights with random ones and keeps
everything else. Within each retained bin, with evaluation-fold group proportions
$p_k$ (counts only, no labels), draw $z_k\sim\mathcal N(0,1)$, centre
$a_k=z_k-\sum_jp_jz_j$ so that $\sum_kp_ka_k=0$, and set $r_k=1+t\,a_k$ with
\[
  t=\min\Big\{\big(\rho/\textstyle\sum_kp_ka_k^2\big)^{1/2},\;
  \min_{k:\,a_k<0}(-1/a_k)\Big\},
\]
so that $r\ge0$, $\sum_kp_kr_k=1$ and $\sum_kp_k(r_k-1)^2\le\rho$. No sign flip
or orientation by $\hat\theta$, correctness, $G_R$, ECE, FAUC or drift is
applied. Each of the $720{,}000$ shifts was checked against all three constraints
at tolerance $10^{-12}$, with no violation. The seed ($90210$), the number of
directions ($100$) and the per-cell random streams were fixed in the protocol.
The data roles are those of Appendix~\ref{app:estimators}:
$\widehat{\mathcal F}^{\pm}_{B,R}(b,\rho)$ is computed from $D_{\mathrm{rate}}$
rates alone, the weights are frozen, and the realised movement
$\delta_{\mathrm{random}}=\sum_kp_k(r_k-1)\theta^{\mathrm{eval}}_k$ is read from
held-out correctness on $D_{\mathrm{eval}}$. All twenty bins were retained in
each of the $60$ (checkpoint, state, fold) units, so there are
$6\times2\times5\times6\times20=7{,}200$ bin--radius cells with $100$ shifts each.
The cell, not the shift, is the unit: shifts in one cell share every estimated
quantity and all of their held-out labels.

\paragraph{Ordering.} For each cell we compare the two-sided plug-in profile
$\max(\widehat{\mathcal F}^{+}_{B,R},\widehat{\mathcal F}^{-}_{B,R})$ with the
median of $|\delta_{\mathrm{random}}|$ over its $100$ directions; the Spearman
correlation over the $7{,}200$ cells is \RsPooledPrim{}. With the largest of the $100$
movements in place of the median it is $0.925$; we treat that as secondary,
because a maximum over $100$ held-out responses is itself a selection. The
protocol did not declare a rank summary, so both correlations are post-hoc
descriptions; the cells are not independent, and we attach no interval to them.
Figure~\ref{fig:ext-random}(a) shows the cells against the diagonal. Most lie far
below it, as expected for non-optimised directions; the median random movement is
about a quarter of the profile. Most of the correlation does not come from
within-bin heterogeneity. Pooling over radii adds the common $\sqrt\rho$ scale, and
the within-bin label-permutation null of Appendix~\ref{app:r7}, which has no
heterogeneity at all, reproduces most of the rest because binomial rate noise
scales with bin accuracy in both quantities (Table~\ref{tab:r7-null}). The ordering
is therefore a weak test of the profile; the controlled shifts, whose signed drift is
centred at zero under the same null, are the informative one.

\paragraph{Sharpness.} Table~\ref{tab:ext-random} reports the per-shift ratio
$|\delta_{\mathrm{random}}|/\widehat{\mathcal F}^{\pm}_{B,R}$ on the branch
selected by the sign of $\delta_{\mathrm{random}}$, the rule fixed in the
protocol. Overall, $96.1\%$ of ratios are below one; where a ratio exceeds one it
does so by a mean of $0.30$. These are descriptive frequencies over dependent
shifts, and since $\widehat{\mathcal F}$ is a plug-in estimate a ratio above one is
not a coverage failure. Exceedance falls monotonically across $G_R$ quartiles,
from $9.8\%$ to $0.63\%$, and it also falls with radius, from about $5\%$ at
radii up to $0.3$ to $0.55\%$ at $\rho=3$. The quartile pattern is consistent with
cell-rate noise mattering most when the restricted heterogeneity it is compared
with is small; the analysis cannot separate that explanation from others. By checkpoint the median ratio lies between $0.254$ and $0.281$ and the
exceedance between $2.6\%$ and $5.3\%$. Direction sampling is a separate source
of variation: across the $100$ directions of a cell the ratio has a median
standard deviation of $0.220$, which reflects the random directions and says
nothing about dataset sampling.

\paragraph{What this does not show.} The random shifts live on the same finite
support and obey the same $\chi^2$ geometry as the optimised ones; they are not
natural shifts, they say nothing about how close natural shifts come to the worst
case, and they give $\widehat{\mathcal F}$ no finite-sample coverage. They remove
one coupling --- the optimisation of the weights against $\hat\theta$ --- and
nothing else.

\begin{figure}[h]
\centering
\includegraphics[width=\textwidth]{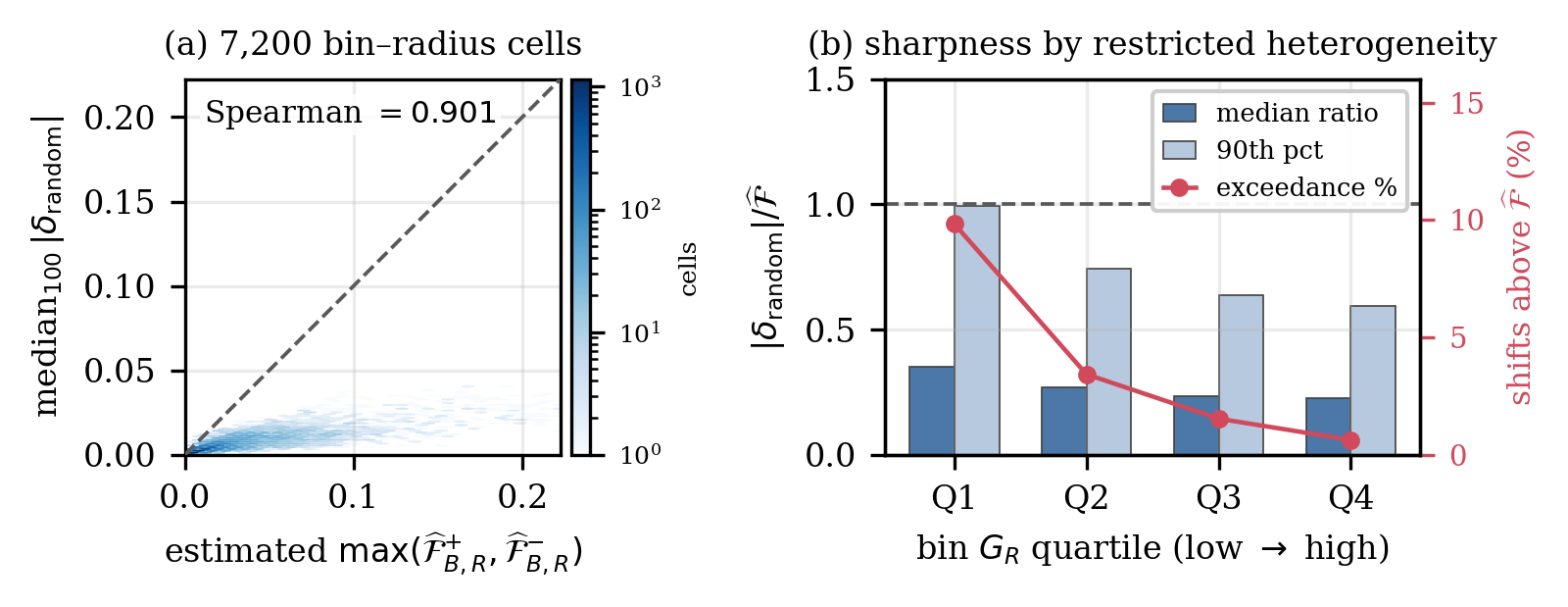}
\caption{\textbf{Random score-preserving shifts, post-audit.} (a) One point per
bin--radius cell: two-sided plug-in profile against the median held-out movement
over $100$ random directions (Spearman $0.901$, most of which a within-bin
label-permutation null reproduces; Appendix~\ref{app:r7}); dashed line $y=x$, equal axes.
(b) Per-shift ratio $|\delta_{\mathrm{random}}|/\widehat{\mathcal F}$ by quartile
of the bin's $G_R$: median and 90th percentile (bars, left axis) and the
percentage of shifts with ratio above one (line, right axis). Descriptive; the
shifts are not independent.}
\label{fig:ext-random}
\end{figure}

\begin{table}[h]
\caption{Random score-preserving shifts: the per-shift ratio $|\delta_{\mathrm{random}}|/\widehat{\mathcal F}^{\pm}_{B,R}$, with the branch chosen by the sign of $\delta_{\mathrm{random}}$. Shifts within a bin--radius cell share data and are not independent; the percentages are descriptive and are not coverage rates. Post-audit, not preregistered.}
\label{tab:ext-random}
\vspace{3pt}
\centering\small
\begin{tabular}{lrccc}
\toprule
stratum & shifts & median ratio & 90th pct. & ratio $>1$ (\%) \\
\midrule
all & 720,000 & 0.264 & 0.741 & 3.86 \\
\midrule
$G_R$ Q1 $[-0.00148,\,0.00012]$ & 180,000 & 0.352 & 0.994 & 9.84 \\
$G_R$ Q2 $[0.00012,\,0.00066]$ & 180,000 & 0.270 & 0.742 & 3.43 \\
$G_R$ Q3 $[0.00066,\,0.00235]$ & 180,000 & 0.232 & 0.638 & 1.54 \\
$G_R$ Q4 $[0.00235,\,0.03797]$ & 180,000 & 0.226 & 0.595 & 0.63 \\
\midrule
$T=1$ & 360,000 & 0.257 & 0.708 & 3.08 \\
$T=\hat T$ & 360,000 & 0.270 & 0.776 & 4.64 \\
\midrule
$\rho=0.01$ & 120,000 & 0.318 & 0.820 & 5.29 \\
$\rho=0.03$ & 120,000 & 0.318 & 0.816 & 5.16 \\
$\rho=0.1$ & 120,000 & 0.318 & 0.825 & 5.36 \\
$\rho=0.3$ & 120,000 & 0.311 & 0.806 & 4.93 \\
$\rho=1.0$ & 120,000 & 0.221 & 0.608 & 1.86 \\
$\rho=3.0$ & 120,000 & 0.157 & 0.448 & 0.55 \\
\bottomrule
\end{tabular}
\end{table}

\subsection{A human-label proxy for the correctness propensity}\label{app:ext:human}

The propensity $\eta(x)=\Pr(C=1\mid X=x)$ is never observed, and the hierarchy of
Section~\ref{sec:theory} is informative only where it varies inside a level set.
CIFAR-10H \citep{peterson2019human} gives about fifty human labels for each of the
$10{,}000$ CIFAR-10 test images \citep{krizhevsky2009learning}, which allows a
direct if approximate look. For a frozen classifier with prediction $\hat y(x)$ we
use
\[
  \eta_H(x)=\frac{\#\{\text{annotations of }x\text{ equal to }\hat y(x)\}}
                 {\#\{\text{annotations of }x\}},
\]
the agreement of human annotators with the model. This is a proxy: it measures
correctness against the human label distribution rather than against the fixed
dataset label, and it is estimated from finitely many votes. We use three public
CIFAR-10 checkpoints fixed in the protocol (ResNet-56, VGG16-BN and MobileNetV2;
none trained by us) and the $514{,}200$ annotations of $2{,}571$ annotators that
remain after attention checks are removed. The annotators are split at random,
with a seed fixed in the protocol, into two disjoint groups, so each image's
labels divide into two halves given by different people; this yields proxies
$\eta_H^A$ and $\eta_H^B$ from a mean of about $26$ votes each (range
$14$--$39$). The protocol described the split per image; the implementation
instead partitions annotators globally, a stricter separation because no
annotator contributes to both halves. The data roles
follow the main pipeline with $J=10$ bins and $K=4$ groups, chosen in the protocol
so that cells keep about $100$ rate rows; the readout (logistic regression on
confidence, margin, entropy and a $16$-dimensional principal-component projection
of the penultimate representation) is fitted on $D_{\mathrm{fit}}$ CIFAR-10
correctness. The projection uses no labels but, unlike the ImageNet pipeline, is
computed from a random $5{,}000$ of all $10{,}000$ images, so evaluation-image
features (never their labels) enter it.
Table~\ref{tab:ext-human} and Figure~\ref{fig:ext-human} give the results in
three layers, which answer different questions.

\begin{table}[h]
\caption{CIFAR-10H case study, $J=10$ bins and $K=4$ groups in each of $5$ folds, $50$ fold--bin cells per model. $G_H^{\times}$ is the cross-half covariance of the two annotator-half proxies; $G_H^{\times}$, $G_R$ and $G_R^H$ are medians over cells, in units of $10^{-3}$. Columns 4--5: the pre-specified cross-target diagnostic, whose restricted term uses the fixed CIFAR-10 label and so a different propensity from $G_H^{\times}$; Proposition~\ref{prop:hierarchy} does not order the two. Columns 6--8: the post-hoc same-proxy diagnostic, in which both terms use $\eta_H$. Ratios are medians over cells with interquartile ranges in brackets. Top-1 accuracy on the CIFAR-10 test set: ResNet-56 94.37\%, VGG16-BN 94.15\%, MobileNetV2 94.05\%.}
\label{tab:ext-human}
\vspace{3pt}
\centering\footnotesize\setlength{\tabcolsep}{3.5pt}
\begin{tabular}{lccccccc}
\toprule
 & \multicolumn{2}{c}{human proxy} & \multicolumn{2}{c}{\shortstack{cross-target\\(pre-specified)}} & \multicolumn{3}{c}{\shortstack{same proxy\\(post-hoc)}} \\
\cmidrule(lr){2-3}\cmidrule(lr){4-5}\cmidrule(lr){6-8}
model & $G_H^{\times}$ & vote noise & $G_R$ & ratio & $G_R^H$ & ratio & $\le G_H^{\times}$ \\
\midrule
ResNet-56 & 9.29 & 9.3\% & 0.021 & 0.33\% & 0.521 & 6.3\% [3.8, 10.1] & 50/50 \\
VGG16-BN & 8.33 & 9.9\% & 0.012 & 0.17\% & 0.093 & 1.7\% [0.2, 3.5] & 50/50 \\
MobileNetV2 & 9.23 & 8.3\% & 0.035 & 0.29\% & 0.251 & 2.8\% [1.3, 4.3] & 50/50 \\
\midrule
pooled & & & & 0.25\% [0.0, 2.1] & & 3.3\% [1.4, 6.3] & 150/150 \\
\bottomrule
\end{tabular}
\end{table}

\paragraph{Human-proxy heterogeneity.} Within each bin,
$G_H(b)=\operatorname{Var}(\eta_H\mid B=b)$ on evaluation images. Computed separately
from the two annotator halves, the $150$ (model, fold, bin) values agree with
Pearson correlation $0.9995$ (Figure~\ref{fig:ext-human}a). Vote noise inflates a
variance computed from vote proportions, so the primary quantity is the
cross-half covariance $G_H^{\times}(b)=\operatorname{Cov}(\eta_H^A,\eta_H^B\mid B=b)$,
whose two factors carry independent vote noise; this correction was fixed in an
addendum written after the protocol but before any heterogeneity value was
inspected. $G_H^{\times}$ is positive in all $150$ cells, and vote noise accounts for
a median $8$--$10\%$ of the uncorrected variance. Like every binned quantity in
this paper, $G_H$ also contains the variation of the mean propensity across the
scores inside a bin, so it is an upper reference for, not an estimate of,
level-set heterogeneity in the human proxy.

\paragraph{Pre-specified cross-target diagnostic.} The protocol compared the
debiased restricted $G_R$ of Proposition~\ref{prop:gunbiased}, estimated from
CIFAR-10 correctness, with $G_H^{\times}$. The pooled median ratio is $0.25\%$ and
$G_R$ is below $G_H^{\times}$ in every cell. The two terms target different
propensities, however: with one fixed label per image the correctness target is
$\{0,1\}$-valued, whereas $G_H^{\times}$ is built on the human label distribution.
Proposition~\ref{prop:hierarchy} orders heterogeneity quantities that share a
propensity and does not order these two, so this ratio is not a capture ratio,
and we report it only as the diagnostic the protocol specified.

\paragraph{Post-hoc same-proxy diagnostic.} After reading that result we added a
comparison in which both terms use $\eta_H$. The readout, bins and cells are
unchanged; the $D_{\mathrm{rate}}$ cell rate is replaced by the mean of
$\eta_H^A$ over the cell, and the restricted heterogeneity $G_R^H(b)$ is debiased
with the continuous analogue of Proposition~\ref{prop:gunbiased}, subtracting
$\sum_kp_k(1-p_k)s_k^2/n_k$ with $s_k^2$ the within-cell sample variance.
$G_R^H$ is at most $G_H^{\times}$ in all $150$ cells
(Figure~\ref{fig:ext-human}b); $14$ of the $150$ debiased estimates are negative,
and they are retained, so for those cells the inequality holds trivially. The
pooled median of $G_R^H/G_H^{\times}$ is $3.3\%$ (interquartile range
$1.4$--$6.3\%$), between $1.7\%$ and $6.3\%$ by model. In the profile the fraction
is larger: at the smallest radius, $\rho=0.01$, the same-proxy median of
$\mathcal F^{H}_{B,R}/\mathcal F^{H}_{B}$ is $0.26$, $0.16$ and $0.18$ for the three
models, close to the square roots of their variance ratios, which is the relation
the variance-controlled regime of Theorem~\ref{thm:profile} gives when both
profiles are in that regime; it rises with $\rho$, to $0.53$, $0.36$ and $0.38$ at
$\rho=3$. Both profiles in these ratios are estimated from finite data --- the
human-proxy profile $\mathcal F^{H}_{B}$ has no simple vote-noise correction, and
the restricted profile uses plug-in cell means --- so the direction of bias of their
ratio is not established; we use it descriptively. This comparison was
defined after the pre-specified result had been read; it is labelled post-hoc and
does not replace that result.

Taken together: in this dataset a human-label proxy for $\eta$ has within-bin
heterogeneity that is reproducible across disjoint annotators, and a finite
readout fitted to dataset correctness exposes a few percent of it in variance. The case study involves three convolutional classifiers on one
dataset, and the proxy is not $\eta$; it grounds the premise of the theory and
does not estimate the oracle quantities of Section~\ref{sec:theory}. No
lower confidence bound is computed, because repeated labels of the same images
do not meet its assumptions.

\begin{figure}[h]
\centering
\includegraphics[width=\textwidth]{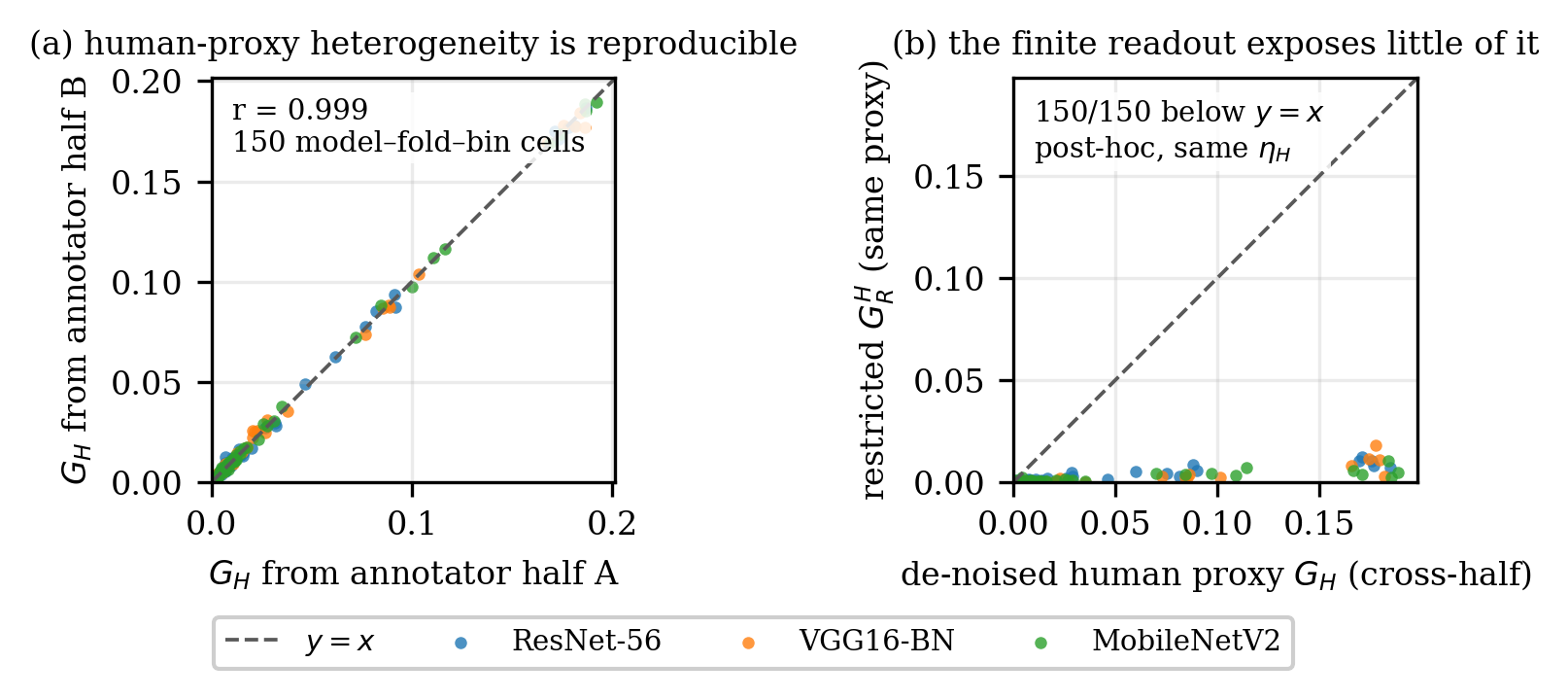}
\caption{\textbf{CIFAR-10H case study, post-audit.} One point per (model, fold, bin) cell, $150$ in total; dashed line $y=x$. (a) Human-proxy heterogeneity $G_H$ from
annotator half A against half B. (b) Post-hoc same-proxy diagnostic: the
vote-noise-corrected cross-half $G_H^{\times}$ against the debiased restricted
$G_R^H$ computed from the same proxy.}
\label{fig:ext-human}
\end{figure}

\subsection{Natural shift under the revised partition}\label{app:ext:natural}

Section~\ref{sec:natural} and Appendix~\ref{app:phase7} use the original frozen
partition. For ImageNetV2 and ImageNet-Sketch we forwarded every target image
through each checkpoint with that checkpoint's own preprocessing and cached the
penultimate representation, which allows the revised pipeline to be applied end
to end: bin edges, temperature, readout and group boundaries are fitted on
$D_{\mathrm{fit}}$ and applied unchanged to the target, source cell rates come
from $D_{\mathrm{rate}}$ and source proportions from $D_{\mathrm{eval}}$, and
target labels fit nothing. Retention rules, $J$, $K$ and the radius grid are
those of the main pipeline, and Proposition~\ref{prop:decomp} closes to within $3.9\times10^{-16}$ in every bin.

\begin{table}[h]
\caption{Natural shift under the original and the revised partition, with matched aggregation: per (checkpoint, source fold) unit, then the mean over the $30$ units. The primary comparison is between the first two rows of each block, both in the as-released state; the temperature-scaled row has no original-partition counterpart and is an additional result, not a partition sensitivity. Units share target images and source data and are not independent. Post-audit, not preregistered.}
\label{tab:ext-natural}
\vspace{3pt}
\centering\small
\begin{tabular}{llccccc}
\toprule
benchmark & partition, state & units & $L_{\mathrm{comp}}/L_{\mathrm{total}}$ & range & $L_{\mathrm{cond}}>L_{\mathrm{comp}}$ & median $U$ \\
\midrule
ImageNetV2 & original, $T=1$ & 30 & 12.2\% & [2.1, 29.7] & 100\% & 0.399 \\
 & revised, $T=1$ & 30 & 12.0\% & [2.4, 31.5] & 100\% & 0.357 \\
 & revised, $T=\hat T$ & 30 & 6.3\% & [3.0, 12.4] & 100\% & 0.331 \\
\midrule
ImageNet-Sketch & original, $T=1$ & 30 & 9.8\% & [3.7, 24.7] & 100\% & 0.891 \\
 & revised, $T=1$ & 30 & 10.0\% & [2.8, 30.0] & 100\% & 0.768 \\
 & revised, $T=\hat T$ & 30 & 7.8\% & [1.6, 22.1] & 100\% & 0.661 \\
\bottomrule
\end{tabular}
\end{table}

The comparison is made at matched aggregation. Each (checkpoint, source fold)
unit gives its own $L_{\mathrm{total}}$, $L_{\mathrm{comp}}$ and
$L_{\mathrm{cond}}$ as target-mass-weighted quadratic means over bins, its own
composition-to-total norm ratio and its own median utilisation; the benchmark
value is the mean over the $30$ units. The original partition, with its original
cell rates, is recomputed in exactly this way. Section~\ref{sec:natural} reports
$11.1\%$ and $9.6\%$ for the same benchmarks because Appendix~\ref{app:decomp}
pools the five folds at bin level before forming per-(model, domain) norms; the
matched recomputation of the original partition gives $12.2\%$ and $9.8\%$. The
difference is one of aggregation, not a conflict between analyses, and no
quantity below mixes the two partitions.

Table~\ref{tab:ext-natural} gives the result. The primary comparison is the
original against the revised partition in the as-released state, the only state
the original analysis used. The composition-to-total norm ratio is essentially
unchanged ($12.2\%\to12.0\%$ on ImageNetV2, $9.8\%\to10.0\%$ on
ImageNet-Sketch), and the conditional norm exceeds the composition norm in all
$30$ units of both benchmarks under both partitions. Utilisation is lower under
the revised partition ($0.399\to0.357$ and $0.891\to0.768$); the revision changes
the fitting data, the temperature, the readout and the rate estimator together,
so we attribute that fall to none of them. The temperature-scaled state, which
the original analysis did not include, gives lower ratios ($6.3\%$ and $7.8\%$)
with the conditional term again larger in every unit; it is an additional result
rather than a partition sensitivity. Agreement here is qualitative agreement on
two single-domain benchmarks, not a validation of Section~\ref{sec:natural}.

ImageNet-C was not rerun: its raw images were not available to us for this
analysis, and no subset of corruptions or severities was substituted. This
recomputation therefore says nothing about the per-corruption comparison, the
severity trend or the channel-specific association of
Appendix~\ref{app:phase7}, all of which rest on ImageNet-C.

\section{Architecture-breadth extension}\label{app:breadth}

\textbf{Post-audit, not preregistered.} The paper and every result above were
known when this extension was designed. Its question is whether the source-side
findings --- the controlled-shift rank agreement, the random-shift ordering, the
split-sample lower confidence bound and the identical as-released rankings of ECE,
$G_R$ and FAUC --- depend on the six checkpoints of Table~\ref{tab:models}. The answer is
descriptive: checkpoints are not i.i.d.\ draws from a population of models, so no
$p$-value and no interval is attached to any comparison below.

\paragraph{Cohort.} Twelve additional frozen ImageNet-1K checkpoints, one per
architecture family, were resolved from public metadata before any weight was
downloaded and before any quantity was computed, under rules fixed in the
released protocol: ImageNet-1K-only training, no distillation where the family
has a non-distilled checkpoint, evaluation at $224$, original-author weights where available, then the checkpoint closest to a
predeclared $22$~M parameter target. This yielded DenseNet-201, ResNeXt-50, RegNetX-4.0GF,
MobileNetV3-L, DeiT-S/16, XCiT-S12/16, PoolFormer-S24, MaxViT-T, ResMLP-12
(the all-MLP family; no ImageNet-1K-only MLP-Mixer checkpoint exists), MViTv2-T,
ConvMixer-768/32 and PVTv2-B2 (the predeclared backup after EfficientFormer,
distilled-only, and EdgeNeXt, evaluated at $256$). Exact identifiers, weight
checksums and preprocessing are released. Every checkpoint reproduces its public
top-1 within $0.04$~pp on the same $50{,}000$ validation images in the same
order as the primary cohort; no technical exclusion was needed. Each checkpoint
was then run through the revised pipeline of Appendix~\ref{app:estimators}
unchanged (five folds, seed $20270214$, $J=20$, $K=8$, the data roles
$D_{\mathrm{fit}}/D_{\mathrm{rate}}/D_{\mathrm{eval}}$, retention rules,
estimators and radius grid), in the four calibration states of
Section~\ref{sec:cannotsee}. Table~\ref{tab:breadth} gives the two primary states.

\begin{table}[h]
\caption{Post-audit architecture-breadth cohort: twelve additional frozen ImageNet-1K checkpoints, one per family, under the revised data roles of Appendix~\ref{app:estimators} (same estimators as Table~\ref{tab:models}). Acc.\ is the measured top-1 on the 50,000 validation images; every checkpoint reproduces its public reference within $0.04$~pp. Not preregistered; the six checkpoints of Table~\ref{tab:models} remain the frozen primary cohort.}
\label{tab:breadth}
\vspace{3pt}
\centering\footnotesize\setlength{\tabcolsep}{4pt}
\begin{tabular}{llccccccc}
\toprule
& & & \multicolumn{3}{c}{$T=1$} & \multicolumn{3}{c}{$T=\hat T$} \\
\cmidrule(lr){4-6}\cmidrule(lr){7-9}
Model & family & Acc. & ECE & $G_R$ & FAUC & ECE & $G_R$ & FAUC \\
\midrule
DenseNet-201 & DenseNet & 0.773 & 0.032 & 0.0006 & 0.043 & 0.021 & 0.0009 & 0.050 \\
ResNeXt-50 32x4d & ResNeXt & 0.776 & 0.065 & 0.0009 & 0.049 & 0.028 & 0.0015 & 0.057 \\
RegNetX-4.0GF & RegNet & 0.785 & 0.052 & 0.0008 & 0.047 & 0.025 & 0.0017 & 0.061 \\
MobileNetV3-L & MobileNetV3 & 0.758 & 0.067 & 0.0029 & 0.076 & 0.021 & 0.0007 & 0.047 \\
DeiT-S/16 & DeiT & 0.799 & 0.081 & 0.0010 & 0.049 & 0.041 & 0.0002 & 0.037 \\
XCiT-S12/16 & XCiT & 0.820 & 0.066 & 0.0018 & 0.055 & 0.050 & 0.0014 & 0.049 \\
PoolFormer-S24 & PoolFormer & 0.803 & 0.038 & 0.0042 & 0.072 & 0.033 & 0.0046 & 0.077 \\
MaxViT-T & MaxViT & 0.834 & 0.096 & 0.0027 & 0.072 & 0.037 & 0.0006 & 0.042 \\
ResMLP-12 & all-MLP & 0.766 & 0.121 & 0.0032 & 0.078 & 0.025 & 0.0002 & 0.038 \\
MViTv2-T & MViTv2 & 0.824 & 0.112 & 0.0032 & 0.077 & 0.032 & 0.0006 & 0.041 \\
ConvMixer-768/32 & ConvMixer & 0.802 & 0.175 & 0.0074 & 0.109 & 0.029 & 0.0006 & 0.044 \\
PVTv2-B2 & PVTv2 & 0.821 & 0.080 & 0.0020 & 0.060 & 0.040 & 0.0006 & 0.040 \\
\bottomrule
\end{tabular}
\end{table}

\begin{table}[h]
\caption{Cohort summaries. Top: Spearman rank correlations among the three source metrics across checkpoints (one point per checkpoint), as released / temperature-scaled. Bottom: held-out controlled-shift rank agreement over (checkpoint, state) combinations with the aggregation of Appendix~\ref{app:estimators}, and random score-preserving shifts (Appendix~\ref{app:ext:random}; ``corr.'' is the Spearman correlation over bin--radius cells between the two-sided plug-in profile and the median, or secondarily the maximum, of $100$ held-out movements; most of it is reproduced by the permutation null of Appendix~\ref{app:r7}). Primary-six values in this table use the revised estimator; Appendix~\ref{app:phase7} reports the original-protocol values. All rank correlations are computed from the unrounded source values. Descriptive; checkpoints and cells are not independent samples and no interval is attached.}
\label{tab:breadth-summary}
\vspace{3pt}
\centering\footnotesize\setlength{\tabcolsep}{5pt}
\begin{tabular}{lccc}
\toprule
cohort & (ECE, $G_R$) & (ECE, FAUC) & ($G_R$, FAUC) \\
\midrule
primary six & +1.00 / -0.37 & +1.00 / -0.83 & +1.00 / +0.60 \\
additional twelve & +0.61 / -0.28 & +0.81 / -0.36 & +0.92 / +0.99 \\
all eighteen & +0.75 / -0.26 & +0.85 / -0.36 & +0.95 / +0.95 \\
\bottomrule
\end{tabular}
\par\vspace{5pt}\setlength{\tabcolsep}{3pt}
\begin{tabular}{lcccccccc}
\toprule
& \multicolumn{4}{c}{held-out rank agreement} & \multicolumn{4}{c}{random shifts} \\
\cmidrule(lr){2-5}\cmidrule(lr){6-9}
cohort & comb. & median & IQR & min--max & cells & corr.\ (median) & corr.\ (max) & below $\widehat{\mathcal F}$ \\
\midrule
primary six & 24 & 0.82 & [0.57, 0.87] & 0.08--0.94 & 7,200 & 0.9010 & 0.9254 & 96.1\% \\
additional twelve & 48 & 0.80 & [0.63, 0.86] & 0.14--0.94 & 14,400 & 0.9005 & 0.9241 & 96.2\% \\
all eighteen & 72 & 0.81 & [0.61, 0.87] & 0.08--0.94 & 21,600 & 0.9007 & 0.9245 & 96.2\% \\
\bottomrule
\end{tabular}
\end{table}

\begin{table}[h]
\caption{Additional twelve checkpoints: split-sample lower confidence bound $\mathrm{LCB}^{+}_{\mathrm{model}}(0.3)$ of Proposition~\ref{prop:fscert}, one column per split arm (construction and split seed of Appendix~\ref{app:fscert} unchanged; each entry carries its own per-analysis $95\%$ statement, never joint), and the fold-wise monotone reparameterisation of Appendix~\ref{app:monotone} replicated post-audit (relative reduction of cross-fitted ECE; structural quantities unchanged exactly). The original 4-of-6 outcome and Outcome~A are historical and are not redefined.}
\label{tab:breadth-cert}
\vspace{3pt}
\centering\footnotesize\setlength{\tabcolsep}{3pt}
\begin{tabular}{lcccccc}
\toprule
& \multicolumn{2}{c}{as released $T{=}1$} & \multicolumn{2}{c}{split-fitted $T{=}\hat T$} & \multicolumn{2}{c}{monotone map} \\
\cmidrule(lr){2-3}\cmidrule(lr){4-5}\cmidrule(lr){6-7}
Model & A fits/B cert. & B fits/A cert. & A fits/B cert. & B fits/A cert. & ECE reduction & lower NLL \\
\midrule
DenseNet-201 & $0$ & $0$ & $0$ & $0$ & $57.4\%$ & yes \\
ResNeXt-50 32x4d & $0.00068$ & $0$ & $0.00018$ & $0.00058$ & not fitted & -- \\
RegNetX-4.0GF & $0$ & $0$ & $0.00053$ & $0.00037$ & not fitted & -- \\
MobileNetV3-L & $0.0011$ & $0.0011$ & $0$ & $0.00034$ & $85.2\%$ & yes \\
DeiT-S/16 & $0.00011$ & $0.00064$ & $0$ & $0$ & $70.0\%$ & yes \\
XCiT-S12/16 & $0.00066$ & $0.00059$ & $0$ & $0$ & $53.1\%$ & yes \\
PoolFormer-S24 & $0.0029$ & $0.0032$ & $0.0036$ & $0.0034$ & $50.9\%$ & yes \\
MaxViT-T & $0.0024$ & $0.0028$ & $0$ & $0$ & $81.4\%$ & yes \\
ResMLP-12 & $0.00066$ & $0.0019$ & $0$ & $0$ & $88.5\%$ & yes \\
MViTv2-T & $0.0020$ & $0.0042$ & $0$ & $0$ & $78.6\%$ & yes \\
ConvMixer-768/32 & $0.0110$ & $0.0092$ & $0$ & $0$ & $88.3\%$ & yes \\
PVTv2-B2 & $0.0015$ & $0.0015$ & $0$ & $0$ & $73.8\%$ & yes \\
\bottomrule
\end{tabular}
\end{table}

\paragraph{Rank association among the source metrics.} On the primary six, ECE,
$G_R$ and FAUC rank the checkpoints identically as released (all three Spearman
correlations $+1.00$; Appendix~\ref{app:phase7}). On the twelve new checkpoints
the correlations are \BrConfEceGRraw{} (ECE, $G_R$), \BrConfEceFaucRaw{} (ECE,
FAUC) and \BrConfGRFaucRaw{} ($G_R$, FAUC), and on all eighteen
\CombConfEceGRraw{}, \CombConfEceFaucRaw{} and \CombConfGRFaucRaw{}
(Table~\ref{tab:breadth-summary}, Figure~\ref{fig:breadth-ranks}). The identical
ranking therefore does not persist in the broader cohort: the three quantities
remain positively associated but are no longer rank-identical. Under temperature
scaling the ECE associations change sign (ECE--FAUC \SpEceFaucTsPrim{} on the six,
\SpEceFaucTsAll{} on all eighteen) while $G_R$ and FAUC stay positively associated
(\SpGrFaucTsPrim{} and \SpGrFaucTsAll{}); Table~\ref{tab:r7-cohort} gives the
family-deletion ranges for both states. None of this is evidence about
determination, for the reason given in Appendix~\ref{app:ladder}.

\paragraph{Controlled-shift rank agreement.} With the aggregation of
Appendix~\ref{app:estimators} (median over radii within fold, mean over folds,
one value per checkpoint--state), the breadth cohort gives a median of
\BrRankMedian{} over its $48$ combinations, range \BrRankLo{}--\BrRankHi{},
against $\SpearmanMedian{}$ (range $\SpearmanRangeLo{}$--$\SpearmanRangeHi{}$)
on the primary $24$ and \CombRankMedian{} on all $72$. The weakest values again
sit on the combinations with the smallest restricted heterogeneity
(Figure~\ref{fig:breadth-gr}); as before, this is consistent with a
finite-sample difficulty and identifies no cause.

\begin{figure}[h]
\centering
\includegraphics[width=\textwidth]{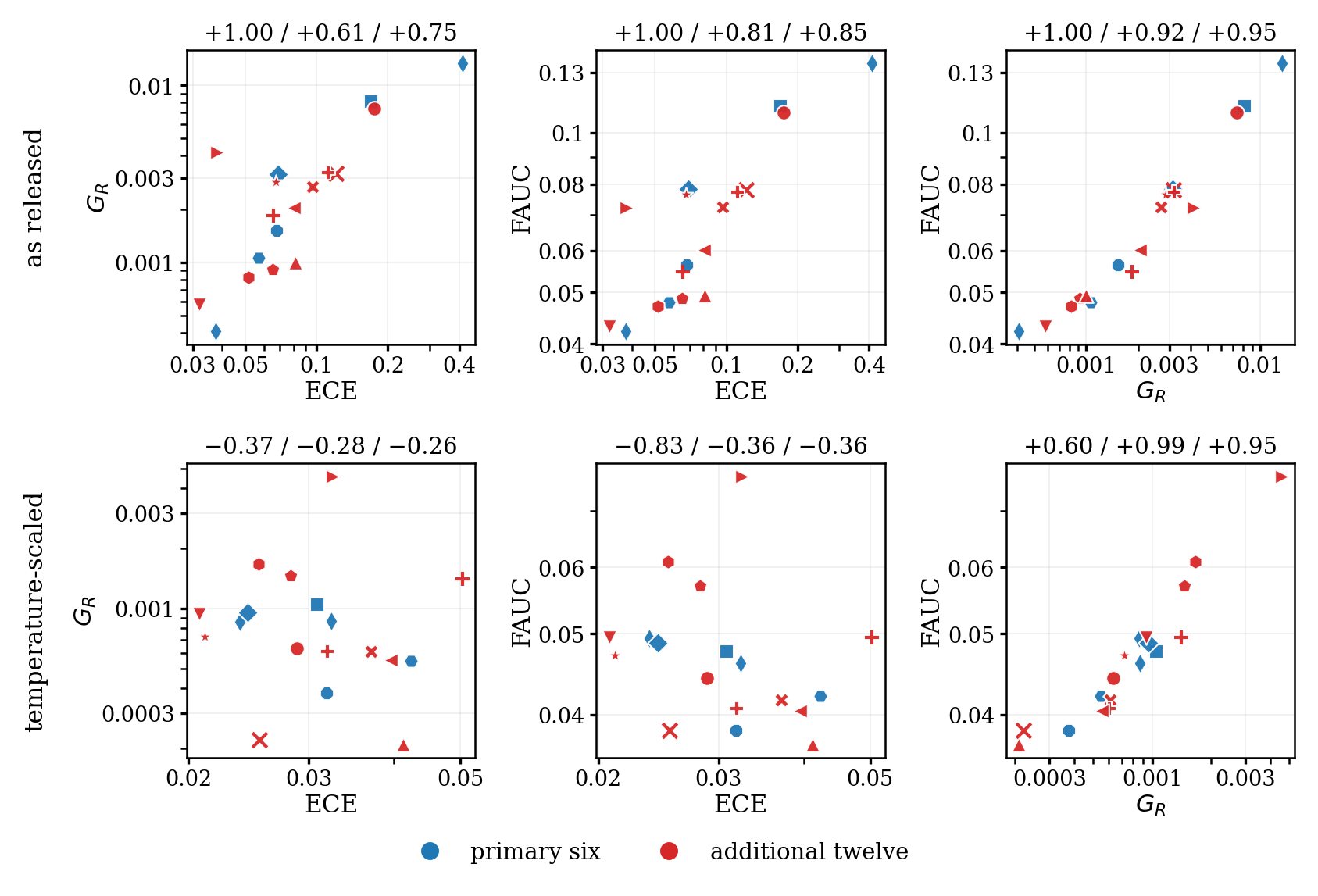}
\caption{\textbf{Source metrics across eighteen checkpoints, as released (top)
and temperature-scaled (bottom).} One point per checkpoint (blue: primary six;
red: additional twelve; marker = family), log axes. Panel titles give the Spearman
rank correlation on the six / the twelve / all eighteen, computed from unrounded
values. Descriptive; the checkpoints are not independent samples.}
\label{fig:breadth-ranks}
\end{figure}

\begin{figure}[h]
\centering
\includegraphics[width=0.6\textwidth]{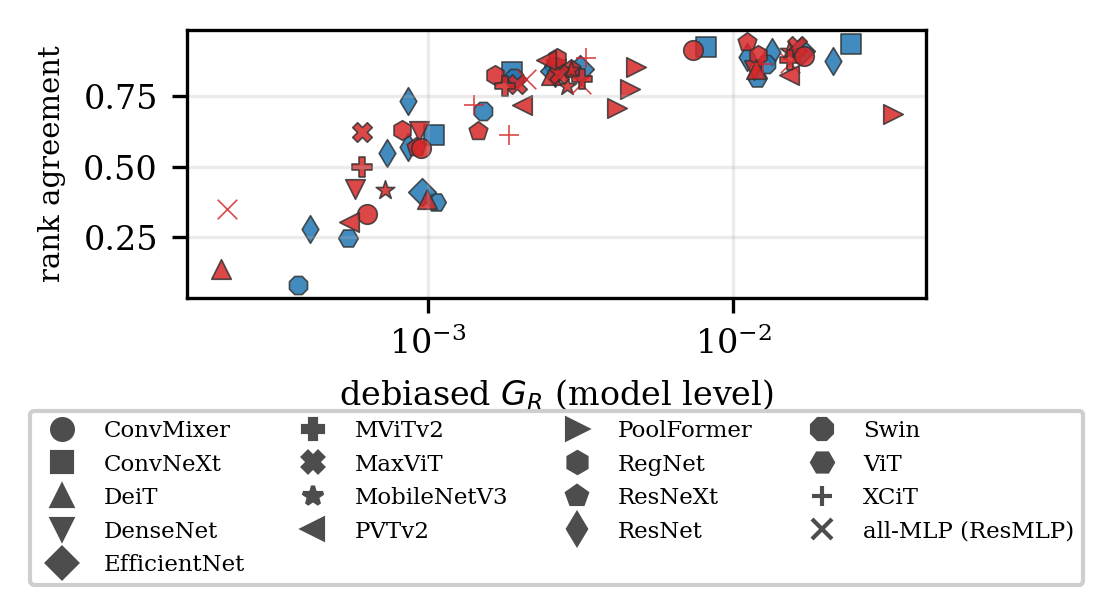}
\caption{\textbf{Held-out rank agreement against restricted heterogeneity.} All
eighteen checkpoints in four calibration states (blue: primary six; red:
breadth twelve; marker = family). The weakest agreement occurs where $G_R$ is
smallest.}
\label{fig:breadth-gr}
\end{figure}

\paragraph{Random score-preserving shifts.} The random-shift construction of
Appendix~\ref{app:ext:random} was repeated on the breadth cohort with a new seed
declared in the protocol ($20260924$): primary states, five folds, twenty bins,
six radii, $100$ label-free Gaussian directions per bin--radius cell, the same
feasibility assertions (no violation in $1{,}440{,}000$ shifts). Over the
breadth cohort's cells the two-sided plug-in profile has Spearman correlation
\RsPooledBr{} with the median held-out movement (\BrEOneSpMax{} with the largest),
against \RsPooledPrim{} ($0.925$) on the primary six; \BrEOneBelow{}\% of movements fall
below the branch of the profile in their own direction, with a median ratio of
\BrEOneMedRatio{}, and exceedances again concentrate in the lowest $G_R$
quartile (\BrEOneQone{}\% against \BrEOneQfour{}\% in the highest). Over all
eighteen checkpoints ($2{,}160{,}000$ shifts) the primary correlation is
\RsPooledAll{} (Figure~\ref{fig:breadth-random}). Cells and checkpoints are
dependent; the shifts remain non-natural and give the plug-in no coverage. The
near-equality of the three cohort values reflects the construction: most of the
statistic is reproduced by the permutation null of Appendix~\ref{app:r7}.

\begin{figure}[h]
\centering
\includegraphics[width=\textwidth]{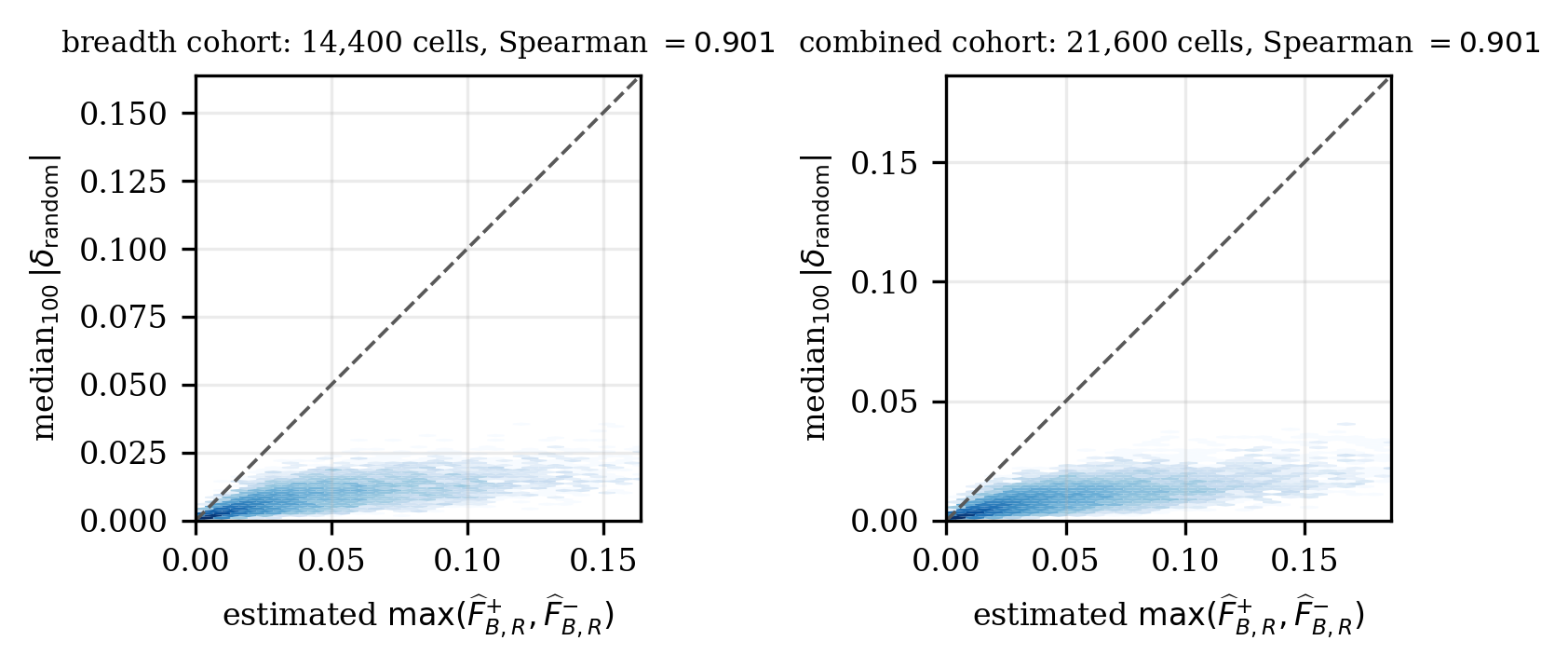}
\caption{\textbf{Random score-preserving shifts on the breadth and combined
cohorts.} One point per bin--radius cell: two-sided plug-in profile against the
median held-out movement over $100$ random directions; dashed line $y=x$, equal
axes. Descriptive; most of the association is reproduced by the permutation null of
Appendix~\ref{app:r7}.}
\label{fig:breadth-random}
\end{figure}

\paragraph{Split-sample lower confidence bound.} The construction of
Appendix~\ref{app:fscert} was run unchanged (split seed $20260902$, both arms,
$500$ Bonferroni-corrected one-sided statements per analysis, no retention
filter) on each new checkpoint. As released, $\mathrm{LCB}^{+}_{\mathrm{model}}(0.3)$
is positive in both arms for \BrCertBothRaw{} of the twelve and in at least one
arm for \BrCertOneRaw{}; in the split-fitted temperature-scaled state it is
positive in both arms for \BrCertBothTemp{} (Table~\ref{tab:breadth-cert}).
Positive values are small, at most $0.0042$ except ConvMixer-768/32
($0.0092$--$0.0110$), in the root-mean-square probability units of
Section~\ref{sec:certificate}. RegNetX-4.0GF is zero as released and positive after
scaling: the scaled state refits the temperature, bins and readout, so it has a
different target. Each entry carries its own $95\%$ statement; the counts describe
the table and assert no joint coverage, and the preregistered outcome categories
of Appendix~\ref{app:fscert}, defined for the six, are not re-scored on the
enlarged set.

\paragraph{Monotone reparameterisation, replicated.} The fold-wise map
$g_f(s)=\mathrm{sigmoid}(a_f\,\mathrm{logit}\,s+c_f)$ of Appendix~\ref{app:monotone}
was fitted and applied by the same code on the new caches. For the ten
checkpoints on which the calibrator fits, the cross-fitted ECE falls by a median
\BrMonoMedian{}\%, binary NLL improves for \BrMonoNLL{} of the ten, and every
structural quantity (bin assignment, frozen readout, cell table, $G_R$,
$\mathcal F^{\pm}$, FAUC) is unchanged to a maximum absolute difference of
\BrMonoMaxdF{}, as Lemma~\ref{lem:reparam} predicts. For ResNeXt-50 and
RegNetX-4.0GF the fit is undefined: three and six validation images receive a
top-label confidence of exactly $1.0$ in float32, where $\mathrm{logit}\,s$ is
infinite, so the optimiser terminates at its identity start; the frozen map
family was not modified and these two are reported as not fitted. A clipped
variant would merge distinct scores and so would not be a strictly increasing
bijection in the sense of Lemma~\ref{lem:reparam}; we did not run one. Outcome~A of
the preregistered protocol is not re-scored.

\paragraph{Family-level sensitivity.} Removing all checkpoints of one of the
$\BrFamilies{}$ families in turn from the combined cohort moves the as-released
Spearman(ECE, FAUC) over \LofoEceFaucLo{}--\LofoEceFaucHi{}, Spearman($G_R$,
FAUC) over \LofoGRFaucLo{}--\LofoGRFaucHi{}, the median held-out rank agreement
over \LofoRankLo{}--\LofoRankHi{} and the random-shift correlation over
\LofoEOneLo{}--\LofoEOneHi{}. Keeping one checkpoint per family (the first in
the frozen cohort order, so ResNet-50 V1 for ResNet) gives \OpfEceFauc{},
\OpfGRFauc{}, \OpfRank{} and \OpfEOne{}. In the temperature-scaled state the
deletions move Spearman(ECE, FAUC) over \LofoTsLo{} to \LofoTsHi{} and one per family
gives \OpfTsEceFauc{} (post-hoc; Table~\ref{tab:r7-cohort}). These are ranges over
deletions, not intervals.

\paragraph{What this changes and what it does not.} The controlled-shift rank
agreement and the positivity of the lower confidence bound as released recur on
twelve checkpoints from twelve further families with no change to the pipeline.
Two things differ from the primary six: the identical as-released rankings of ECE,
$G_R$ and FAUC do not persist, which removes a design-power caveat of
Appendix~\ref{app:phase7} without altering any preregistered outcome; and after
temperature scaling the bound is positive in both arms for \BrCertBothTemp{} of the
\BrN{} breadth checkpoints, whereas it is zero in at least one arm for all six
primary ones. Nothing here bears on natural shift or on the CIFAR-10H case study,
which were not extended.

\section{Post-hoc diagnostics added in the final revision}\label{app:r7}

\paragraph{Provenance.} Every analysis in this appendix was designed after all
earlier results were known and is post-hoc. A protocol fixing the statistics,
nulls, seeds, replicate counts, tolerances and reporting rules was written and
hashed before any of them was run. A first amendment, written before any result
below was computed, records how the frozen partitions were reconstructed. A second
amendment, the controlled-shift null, was defined after the random-shift null had
been inspected. No model is re-run and no preregistered outcome is re-scored. The
partitions ($D_{\mathrm{fit}}$ temperature, bins, readout and groups) were rebuilt
with the frozen code. They reproduce all $21{,}600$ frozen random-shift cells
exactly, and the frozen per-checkpoint $G_R$, FAUC and rank agreement exactly for
$17$ of $18$ checkpoints. For DenseNet-201 the frozen source-pipeline run used a
different number of BLAS threads from its frozen random-shift run, and the
reconstruction, which matches the latter, differs from the former by $3\times10^{-5}$
in FAUC. The randomised PCA in the readout makes the frozen pipeline reproducible
to about $10^{-5}$ in FAUC across thread settings, not bit for bit. The released
code records the thread counts.

\paragraph{Cohort statistics and matched pairs.} Table~\ref{tab:r7-cohort}
recomputes the rank correlations from unrounded source metrics in both states,
including family-deletion ranges for the temperature-scaled state, which the
original extension reported only as released. Table~\ref{tab:r7-matched} lists
every pair among the eighteen checkpoints whose ECE agrees within the original
tolerance; the full enumeration of all $2\times153$ pairs is released. No primary
pair matches in either state, so the original primary-cohort outcome is unchanged.
In every matched pair the FAUC difference has the same sign as the $G_R$
difference. The pairs therefore show that similar estimated ECE does not fix FAUC;
they are not a test of fragility at fixed $G_R$.

\paragraph{The within-bin label-permutation null.} For each (checkpoint, state,
fold, bin), correctness is permuted uniformly among the $D_{\mathrm{rate}}$ rows of
the bin and, independently, among its $D_{\mathrm{eval}}$ rows. Bin assignments,
readout groups, cell counts, retention, source and evaluation proportions and the
random directions are all held fixed. The permutation preserves each bin's accuracy
in each role and the unequal cell counts, and it removes any association between
correctness and the readout inside a bin, so the permuted data have population
$G_R=0$. Everything else is recomputed exactly as for the observed data. There are
$100$ replicates per variant (both roles, $D_{\mathrm{rate}}$ only,
$D_{\mathrm{eval}}$ only) with a fixed seed sequence. The Monte Carlo standard error
of the null mean of the pooled statistic is below $2\times10^{-4}$. Permutations are
conditional on fold: an image in the rate or evaluation set of several folds is
permuted independently in each, so the pooled null does not reproduce cross-fold
dependence. It is a reference distribution for these statistics, not a sampling
distribution for the data, and we report no $p$-value.

Table~\ref{tab:r7-null} gives the results. For the random-shift statistic, the
null with both roles permuted reproduces most of the observed correlation on all
eighteen checkpoints: pooled $0.877$ against $0.901$, and at fixed radius $0.75$
against $0.83$. The mechanism is estimation noise. The plug-in profile of a bin
with no heterogeneity is of order $\sqrt{\rho\,\mu(1-\mu)/n}$, and the unsigned
random movement on held-out data is of the same order with the evaluation counts,
so both rise and fall with bin accuracy whatever the heterogeneity. Permuting only
one role removes part of this common factor and lowers the null further. The
observed statistics exceed every replicate, by a few hundredths pooled and about
$0.08$ at fixed radius. Within single (checkpoint, state, fold, radius) units the
primary cohort's observed median lies inside the null range. The controlled-shift
statistic behaves differently. Its target is the \emph{signed} held-out drift of
the frozen maximising weights, which has mean zero under the null. Its null median
is centred at zero, with maxima of at most $0.07$ over the three cohorts, against
$0.59$--$0.63$ observed.

\paragraph{The profile against its variance.} Table~\ref{tab:r7-profile}
compares, on identical cells, rates and proportions, the plug-in profile with
$\sqrt{\rho\hat G}$ and with $\min\{\sqrt{\rho\hat G},\hat z_+\}$, where
$\hat G=\sum_kp_k(\hat\theta_k-\bar\theta)^2$ is the plug-in variance. It is not
the corrected $G_R$ and not a square-rooted Table~\ref{tab:models} entry. Both
baselines bound the plug-in profile from above, by Cauchy--Schwarz and the support
bound, and they coincide with it in the variance-controlled regime. The cap uses
the same estimated support endpoint $\hat z_+$ as the profile, so it removes the part
of the overstatement that the endpoint information alone accounts for. Model-level FAUC
recomputed with either baseline in place of the profile, with the frozen
aggregation, exceeds FAUC by a median factor of \FaucVoverF{} and \FaucVcapOverF{}
respectively, and ranks the $36$ checkpoint--state combinations with Spearman
\FaucVRank{} and \FaucVcapRank{} against FAUC. On the random-shift cells the three
predictors have pooled Spearman correlations of $0.901$, $0.896$ and $0.898$ with
the median movement. For the controlled shifts, the median rank agreement of each
with held-out drift is $0.62$. The differences between profile and variance
therefore show in magnitudes at $\rho\ge1$ and between matched-variance bins, but
they do not change rankings or held-out ordering on these readouts. The hierarchy
of Proposition~\ref{prop:hierarchy} does not order regime boundaries: refining the
cells $\{0.4,0.6\}$ (probabilities $\tfrac12,\tfrac12$) into rates $0.4,0.5,0.7$
(probabilities $\tfrac12,\tfrac14,\tfrac14$) raises $\rho_+$ from $1$ to $1.5$, and a
larger budget can move a unit from the tail-controlled regime into saturation. We did
not test how the granularity of the readout affects the regime boundaries or this
comparison.

\paragraph{Not done.} The protocol listed two optional analyses, a bin-count
sensitivity and a stabilised monotone map for ResNeXt-50 and RegNetX-4.0GF. Neither
was run, and no claim depends on them. ImageNet-C was not re-run under the revised
pipeline.

\begin{table}[h]
\caption{Rank correlations among source metrics across checkpoints (Spearman, unrounded inputs), as released / temperature-scaled. Leave-one-family-out gives the range over the $17$ deletions; one-per-family keeps the first checkpoint of each family in the frozen cohort order. Post-hoc and descriptive; checkpoints are not independent samples.}
\label{tab:r7-cohort}
\vspace{3pt}\centering\scriptsize\setlength{\tabcolsep}{3pt}
\begin{tabular}{lccc}
\toprule
& (ECE, FAUC) & ($G_R$, FAUC) & (ECE, $G_R$) \\
\midrule
primary six & +1.00 / $-$0.83 & +1.00 / +0.60 & +1.00 / $-$0.37 \\
additional twelve & +0.81 / $-$0.36 & +0.92 / +0.99 & +0.61 / $-$0.28 \\
all eighteen & +0.85 / $-$0.36 & +0.95 / +0.95 & +0.75 / $-$0.26 \\
\midrule
leave one family out & [+0.79, +0.90] / [$-$0.50, $-$0.28] & [+0.93, +0.98] / [+0.94, +0.96] & [+0.65, +0.90] / [$-$0.42, $-$0.16] \\
one per family & +0.82 / $-$0.35 & +0.94 / +0.96 & +0.71 / $-$0.28 \\
\bottomrule
\end{tabular}
\end{table}

\begin{table}[h]
\caption{All checkpoint pairs among the eighteen whose estimated ECE agrees within the original tolerance $5\times10^{-4}$ (all $153$ pairs per state enumerated; none within the primary six in either state). Post-hoc enumeration on an enlarged cohort; similar estimated ECE is not equal population calibration error, and no significance is claimed.}
\label{tab:r7-matched}
\vspace{3pt}\centering\scriptsize\setlength{\tabcolsep}{3pt}
\begin{tabular}{llcccc}
\toprule
state & pair & $|\Delta\mathrm{ECE}|$ & FAUC & FAUC ratio & $G_R$ \\
\midrule
$T=\hat T$ & ResNet-50 V2 / PoolFormer-S24 & $1.41\times10^{-4}$ & 0.0461 / 0.0770 & 1.67 & 0.00086 / 0.00459 \\
$T=\hat T$ & ResNet-50 V2 / MViTv2-T & $4.54\times10^{-4}$ & 0.0461 / 0.0407 & 1.13 & 0.00086 / 0.00061 \\
$T=\hat T$ & Swin-T / MViTv2-T & $0.77\times10^{-4}$ & 0.0384 / 0.0407 & 1.06 & 0.00038 / 0.00061 \\
$T=\hat T$ & DenseNet-201 / MobileNetV3-L & $3.72\times10^{-4}$ & 0.0495 / 0.0471 & 1.05 & 0.00094 / 0.00072 \\
$T=\hat T$ & RegNetX-4.0GF / ResMLP-12 & $1.09\times10^{-4}$ & 0.0610 / 0.0384 & 1.59 & 0.00166 / 0.00022 \\
$T=1$ & ResNet-50 V1 / PoolFormer-S24 & $4.98\times10^{-4}$ & 0.0423 / 0.0723 & 1.71 & 0.00041 / 0.00417 \\
$T=1$ & ResNeXt-50 / XCiT-S12/16 & $4.29\times10^{-4}$ & 0.0488 / 0.0547 & 1.12 & 0.00091 / 0.00184 \\
\bottomrule
\end{tabular}
\end{table}

\begin{table}[h]
\caption{Random-shift ordering statistic against the within-bin label-permutation null ($M=100$ replicates per variant; correctness permuted within each bin in $D_{\mathrm{rate}}$ and/or $D_{\mathrm{eval}}$, with partitions, cell counts, proportions and directions fixed). Entries: observed / null mean [null maximum]. ``Fixed radius'': for the observed data and for each replicate, the mean over the six radii of the Spearman correlation over cells at that radius; the null mean and maximum are taken over replicates of that mean. (The r7 version of this table printed, as the bracket, the maximum over radii and replicates of the radius-specific correlations, a looser quantity; corrected here.) ``Within unit'': median over (checkpoint, state, fold, radius) units of the Spearman correlation across the $20$ bins. Permutations are fold-conditional, so the null does not reproduce cross-fold dependence; no $p$-value is reported. Post-hoc.}
\label{tab:r7-null}
\vspace{3pt}\centering\scriptsize\setlength{\tabcolsep}{3pt}
\begin{tabular}{llccc}
\toprule
cohort & null variant & pooled & fixed radius & within unit \\
\midrule
primary six & both roles & 0.901 / 0.876 [0.886] & 0.828 / 0.747 [0.771] & 0.776 / 0.748 [0.777] \\
 & $D_{\mathrm{rate}}$ only & 0.901 / 0.835 [0.843] & 0.828 / 0.692 [0.714] & 0.776 / 0.755 [0.776] \\
 & $D_{\mathrm{eval}}$ only & 0.901 / 0.846 [0.854] & 0.828 / 0.659 [0.681] & 0.776 / 0.743 [0.765] \\
\midrule
additional twelve & both roles & 0.901 / 0.877 [0.881] & 0.838 / 0.755 [0.768] & 0.803 / 0.755 [0.774] \\
 & $D_{\mathrm{rate}}$ only & 0.901 / 0.851 [0.856] & 0.838 / 0.726 [0.741] & 0.803 / 0.759 [0.775] \\
 & $D_{\mathrm{eval}}$ only & 0.901 / 0.866 [0.872] & 0.838 / 0.711 [0.722] & 0.803 / 0.764 [0.776] \\
\midrule
all eighteen & both roles & 0.901 / 0.877 [0.881] & 0.835 / 0.753 [0.765] & 0.795 / 0.752 [0.776] \\
 & $D_{\mathrm{rate}}$ only & 0.901 / 0.846 [0.849] & 0.835 / 0.715 [0.725] & 0.795 / 0.758 [0.768] \\
 & $D_{\mathrm{eval}}$ only & 0.901 / 0.859 [0.864] & 0.835 / 0.693 [0.703] & 0.795 / 0.757 [0.770] \\

\bottomrule
\end{tabular}
\par\vspace{5pt}
\begin{tabular}{lccc}
\toprule
controlled shifts (as released and $T=\hat T$) & observed median & null mean & null max \\
\midrule
primary six & $0.592$ & $-0.005$ & $0.059$ \\
additional twelve & $0.627$ & $-0.004$ & $0.067$ \\
all eighteen & $0.624$ & $-0.005$ & $0.046$ \\
\bottomrule
\end{tabular}
\end{table}

\begin{table}[h]
\caption{The plug-in profile against its variance summaries on identical cells (all eighteen checkpoints, as released and $T=\hat T$, $3{,}600$ bin units). Regime: share of units outside the variance-controlled regime (upward profile). Excess: median relative overstatement of $\widehat{\mathcal F}^{+}$ by $\sqrt{\rho\hat G}$ and by $\min\{\sqrt{\rho\hat G},\hat z_+\}$. Matched variance: median relative profile difference within disjoint consecutive pairs of units whose $\hat G$ agree within $1\%$ ($1{,}742$ pairs), and the share of pairs differing by more than $10\%$. Post-hoc; plug-in quantities, no coverage.}
\label{tab:r7-profile}
\vspace{3pt}\centering\footnotesize\setlength{\tabcolsep}{4pt}
\begin{tabular}{lcccccc}
\toprule
$\rho$ & $0.01$ & $0.03$ & $0.1$ & $0.3$ & $1$ & $3$ \\
\midrule
outside variance regime & 0\% & 0\% & 0\% & 45\% & 99\% & 100\% \\
excess of $\sqrt{\rho\hat G}$ & 0\% & 0\% & 0\% & 0\% & 10\% & 44\% \\
excess of capped variance & 0\% & 0\% & 0\% & 0\% & 9\% & 10\% \\
matched-variance difference & 0\% & 0\% & 0\% & 1\% & 10\% & 19\% \\
pairs differing by $>10\%$ & 0\% & 0\% & 0\% & 6\% & 48\% & 72\% \\
\bottomrule
\end{tabular}
\end{table}

\clearpage
\section{Known-propensity magnitude and resolution study}\label{app:r9}

This study was designed after every other result in the paper was known and is
post-hoc. Its specification (laws, sampling, seeds, replicates, metrics, reading rule,
failure rules) was written and hashed before any of its random data were generated;
one technical amendment, a corrected solver (Section~\ref{app:r9:audit}), was made
after the first run and changed no classification. The known-propensity parts are
synthetic: they show when the full profile's magnitude is estimable from rate labels,
not that it matters on any benchmark. Code, seeds and unrounded outputs are in the
supplementary material.

\subsection{Finite-sample magnitude with a known propensity}\label{app:r9:mag}

\paragraph{Design.} One bin with $S\equiv0.5$; the input space is a finite set of
cells with known masses $p_j$ and propensities $\eta_j$, and the readout is the cell,
so the target is the exact profile of the law $(\eta_j,p_j)$. Law A puts
$(\tfrac14,\tfrac12,\tfrac14)$ on $\eta\in\{0.1,0.5,0.9\}$ and law B puts
$(\tfrac1{102},\tfrac{25}{51},\tfrac{25}{51},\tfrac1{102})$ on
$\{0.1,0.22,0.78,0.9\}$: both have mean $\tfrac12$, variance $\tfrac2{25}$ and
endpoints $0.1$ and $0.9$, and at $\rho=3$ the capped variance is $0.4$ for both, while
$\mathcal F^+_{A}(3)=\tfrac25$ (saturation, $\rho_{\mathrm{sat}}=3$) and
$\mathcal F^+_{B}(3)=(24+\sqrt2)/85\approx0.29899$ (tail regime; exact, and equal to a
50-digit independent dual). The family $\lambda P_A+(1-\lambda)P_B$,
$\lambda\in\{0,\tfrac14,\tfrac12,\tfrac34,1\}$, varies the tails at fixed mean,
variance and endpoints. Controls: the same laws with variance divided by $16$
($\eta'=0.5+0.25(\eta-0.5)$); a constant propensity on four cells (target $0$); a
two-point law on four cells, $\eta=(0.3,0.3,0.7,0.7)$, for which full and capped
profiles coincide in population; and one asymmetric law fixed in advance,
$\eta=(0.2,0.7,0.8,0.9)$ with masses $(0.05,0.35,0.40,0.20)$, in both directions. Budgets
$\rho\in\{0,0.01,0.03,0.1,0.3,1,3\}$; $n\in\{250,500,1000,2000,5000\}$ rate labels per
bin; $2000$ independent replicates per condition. Primary sampling is stratified with
proportional allocation and at least one label per cell (law B's endpoint cells get
\MagMinLabelB{} labels at $n=250$); secondary arms use equal allocation, iid sampling
with empty cells dropped, and iid sampling with cells of fewer than $30$ labels dropped,
which is the pipeline's rule; dropped cells stay in the target. From the same
estimated rates and known masses we compute the full plug-in profile,
$\sqrt{\rho\hat G}$ with the uncorrected plug-in variance, the capped variance
$\min\{\sqrt{\rho\hat G},\hat z\}$ with the estimated endpoint, and, as a labelled extra
with more information, the cap at the true endpoint. Errors are in probability units
against the exact profile. The paired unit is the replicate; a condition counts as
favouring one estimator if the mean paired difference in absolute error exceeds
$0.005$ and its $95\%$ Monte-Carlo interval excludes zero.

\begin{table}[h]
\caption{Known-propensity check, primary arm: one bin, upward, $\rho=3$. Exact profile, population slack of the capped variance, and mean absolute
error over $2000$ replicates of the full plug-in and the capped variance from the
same rates.}
\label{tab:r9-magnitude}
\vspace{2pt}
\centering\small
\RnineMainTable
\end{table}

\paragraph{Results.} Table~\ref{tab:r9-magnitude} summarises four laws at $\rho=3$
and Table~\ref{tab:r9-e1} lists the primary arm. Every plug-in
profile was at most both $\sqrt{\rho\hat G}$ and the capped variance (largest excess
$\MagMaxFullOverCap$ over all replicates), as the theory requires, so the plug-in
can be less accurate than the cap only by falling below the target. Of the
\MagCondTotal{} conditions with $\rho>0$, \MagCondFull{} meet the specified
practical-advantage criterion for the full plug-in and none for the cap. In the other
\MagCondTie{} the difference stays below the criterion, which does not mean that the two
estimates are equal. They include every condition in which the cap is exact in
population (the variance-controlled budgets, law A at $\rho=3$, and the two controls)
and conditions whose cap slack is small (at most $0.012$). Population equality need not
persist in the estimates: with four separately estimated cells, the estimated law of a
control is generally neither constant nor two-point. Where the cap has population slack, its error does not
shrink with $n$ (law B at $\rho=3$: \MagCapBsmall{} at $n=250$, \MagCapBlarge{} at
$n=5000$, slack \MagSlackB{}), whereas the plug-in error does (\MagFullBsmall{} to
\MagFullBlarge{}). The plug-in's upward bias, from maximising over noisy rates, is
small in the main laws but reaches \MagWeakSelBias{} for law B with variance
divided by $16$ at $n=250$, where both estimators' relative errors exceed
\MagWeakRelFull{}: weak heterogeneity is where rate noise swamps the difference.
Table~\ref{tab:r9-arms} shows the sampling arms. Equal allocation and iid sampling
with empty cells dropped behave like the primary arm. The $30$-label rule removes at
least one of law B's endpoint cells in every replicate through $n=\MagIIDDropMaxN{}$, and
both endpoints in every replicate through $n=1000$; the estimators then describe a
truncated law and underestimate by about $0.02$. Through $n=1000$ the two coincide; at
$n=2000$ one endpoint survives in some draws and the estimates can differ, but neither
meets the practical-advantage criterion. For the
asymmetric law's downward profile the rule drops the rare hard cell in every
replicate at $n=250$ and in most at $n=500$; the uncapped variance, being larger, is
then closer to the target than both the full plug-in and the capped variance (four
conditions). Full per-condition results, including RMSE, bias
decompositions and the other arms, are in the supplementary material.

\begin{table}[h]
\caption{Known-propensity magnitude, primary arm (proportional stratified labels,
$2000$ replicates). $\mathcal F$: exact profile. MAE f/c: mean absolute error of the
full plug-in and of the capped variance; $\Delta$: mean paired difference
$|\text{full}-\mathcal F|-|\text{capped}-\mathcal F|$ (negative favours the full
plug-in; Monte-Carlo standard errors are at most $0.002$). Upward unless stated.}
\label{tab:r9-e1}
\vspace{2pt}
\centering\scriptsize
\RnineEoneTable
\end{table}

\begin{table}[h]
\caption{Sampling arms at two conditions: bias and MAE of the full plug-in (f) and the
capped variance (c).}
\label{tab:r9-arms}
\vspace{2pt}
\centering\scriptsize
\RnineArmsTable
\end{table}

\subsection{Nested readouts on known atoms}\label{app:r9:nested}

\paragraph{Design.} Three laws on $32$ atoms in one bin, generated by fixed rules and
seeds: \emph{smooth} ($\eta$ increasing in a feature $x$ with noise, masses
Dirichlet$(2)$), \emph{rare} (same $\eta$, masses Dirichlet$(0.3)$, so many atoms are
rare) and \emph{wave} ($\eta$ non-monotone in $x$). A readout of $K$ cells takes
contiguous blocks of $32/K$ atoms in a fixed ordering, $K\in\{1,2,4,8,16\}$, so each
level refines the previous one; the atoms define the binned oracle. The feature
ordering uses no propensity or label. The ordering by $\eta$ is an oracle,
structural diagnostic, not a readout one could learn. We verified that every cell has
exactly one parent and that masses aggregate exactly. For finite samples, $n$ rate
labels are stratified over the $16$ finest cells, atoms are drawn within a cell in
proportion to mass, and the rates of coarser cells are mass-weighted averages of the
finest ones, so all levels use the same labels; $1000$ replicates.

\paragraph{Results.} Table~\ref{tab:r9-e2pop} gives exact quantities. The oracle gap
falls under refinement, and the middle bound of Proposition~\ref{prop:approx} is
tight in the variance-controlled regime and close beyond it: over the refined levels
its excess over the restricted profile is a median \ResBratioMedOne{} times the true
gap at $\rho=1$ and \ResBratioMedThree{} times at $\rho=3$ (largest
\ResBratioMaxThree{}), against \ResAratioMedOne{} and \ResAratioMedThree{} times for
the simple bound. The wave law shows why a coarse readout can miss almost
everything: its two-cell feature split has $G_R=0.0016$ against $G_P=0.050$. At a fixed
label budget (Table~\ref{tab:r9-e2fin}), the estimation error of the restricted
profile tends to grow with $K$, though not in every condition, while the approximation
gap shrinks; in that table (feature ordering, upward, $\rho=3$) the total error against
the oracle was smallest at $K=\ResBestKsmall{}$ for all three laws with $250$ labels
and at $K=\ResBestKlarge{}$ with $5000$. No output was smoothed or constrained to be
monotone. These laws are synthetic, and the bound's input $H_R$ is known here only
because $\eta$ is.

\begin{table}[h]
\caption{Nested readouts, exact population values (upward). Gap: $\mathcal F_B-\mathcal
F_{B,R}$; B and A: the middle and right-hand bounds of Proposition~\ref{prop:approx},
reported as excess over $\mathcal F_{B,R}$.}
\label{tab:r9-e2pop}
\vspace{2pt}
\centering\scriptsize
\RnineEtwoPopTable
\end{table}

\begin{table}[h]
\caption{Nested readouts at a fixed label budget, feature ordering, upward, $\rho=3$:
RMSE of the plug-in restricted profile against its own target / against the binned
oracle ($1000$ replicates).}
\label{tab:r9-e2fin}
\vspace{2pt}
\centering\scriptsize
\RnineEtwoFinTable
\end{table}

\subsection{Nested merges of the ImageNet readout}\label{app:r9:real}

The $K=8$ groups of each bin are quantile groups of one fitted readout score, so
merging adjacent groups gives the quartile, median and trivial partitions of the same
score: a genuinely nested sequence, unlike the readouts of
Appendix~\ref{app:ladder}. Bins, readout, rates ($D_{\mathrm{rate}}$) and masses
($D_{\mathrm{eval}}$) are the frozen ones; the support is fixed once by the pipeline's
rule at $K=8$ and inherited by the merges. At $K=8$ this reproduces every checkpoint's
FAUC and $G_R$ (largest difference $\RealReproMax$). Table~\ref{tab:r9-e3} summarises
the eighteen checkpoints. Four groups retain \RealKfourLo{}--\RealKfourHi{} of the
eight-group FAUC and two groups \RealKtwoLo{}--\RealKtwoHi{}. From four to eight
groups the plug-in $G_R$ keeps rising, but a median of only \RealDebShareRaw{} (as
released) and \RealDebShareTemp{} (temperature-scaled) of that rise survives the
binomial correction of Proposition~\ref{prop:gunbiased}: much of the finest
resolution's increase is rate noise, most clearly after scaling. This is descriptive.
Without $\eta$ there is no oracle gap or $H_R$ to report; the assumption-free envelope
$\mathbb E_b[\theta(1-\theta)]$ of Appendix~\ref{app:approx} is a median \EnvRatioMed{}
times the corrected $G_R$ over the \EnvRecords{} bin records where the latter is
positive (tenth percentile \EnvRatioQten{} times), so it gives no useful bound here.

\begin{table}[h]
\caption{Nested merges on ImageNet, eighteen checkpoints, median [range]. The last
column is the share of the plug-in increase in $G_R$ from four to eight groups that
remains after the binomial-noise correction.}
\label{tab:r9-e3}
\vspace{2pt}
\centering\scriptsize
\RnineEthreeTable
\end{table}

\subsection{A defect in a closed-form solver used by two permutation nulls}\label{app:r9:audit}

The permutation nulls of Appendix~\ref{app:r7} computed profiles with a closed-form
tail-regime solver. When the largest cell rates of a bin are tied, that solver could
accept a spurious root at the top value and return a value above $\sqrt{\rho\hat G}$.
We found this while running the study above, corrected it (ties merged, roots
re-verified elementwise, the frozen bisection as fallback), and recomputed both nulls
with only the solver changed. The corrected solver agrees with the frozen bisection
to $10^{-13}$ on $480{,}000$ random tie-heavy cases. In the random-shift null
\AudChanged{} of \AudTotal{} profile values change by more than $10^{-6}$ (largest
change \AudMaxEntry{}), the null summaries move by at most $\AudMaxSummary$, and no
number or table entry in this paper changes; the controlled-shift null, observed and
permuted, is unchanged. All pre-existing observed statistics, and all profiles outside
these two nulls, use the frozen bisection solver; the study in this appendix uses the
corrected solver, with bisection as fallback. The frozen arrays are kept as released; the corrected ones and the audit are
in the supplementary material.

\end{document}